\documentclass{article}
\pdfoutput=1

     \PassOptionsToPackage{numbers, compress}{natbib}

    \usepackage[preprint]{neurips_2023}

\usepackage[utf8]{inputenc} 
\usepackage[T1]{fontenc}    
\usepackage{hyperref}       
\usepackage{url}            
\usepackage{booktabs}       
\usepackage{amsfonts}       
\usepackage{nicefrac}       
\usepackage{microtype}      
\usepackage{xcolor}         
\usepackage{caption}
\usepackage{graphicx}
 \graphicspath{{figures/}} 
\usepackage{float}
\newfloat{figtab}{htb}{fgtb}
\makeatletter
  \newcommand\figcaption{\def\@captype{figure}\caption}
  \newcommand\tabcaption{\def\@captype{table}\caption}

\usepackage{amsmath}
\usepackage{amssymb}
\usepackage{amsthm}
\usepackage{subfigure}
\usepackage{wrapfig}        
\usepackage{algorithm}
\usepackage{algorithmic}

\newtheorem{theorem}{Theorem}
\newtheorem{lemma}{Lemma}

\newtheorem{corollary}{Corollary}
\newtheorem{assumption}{Assumption}

\hypersetup{
	colorlinks=true
}

\title{Off-Policy RL Algorithms Can be Sample-Efficient for Continuous Control via Sample Multiple Reuse}

\author{%
  Jiafei Lyu$^{1}$\thanks{Work done while working as an intern at Tencent IEG. $^\dagger$ Corresponding Authors.}, Le Wan$^{2}$ Zongqing Lu$^{3 \dagger}$, Xiu Li$^{1 \dagger}$ \\
  $^{1}$Tsinghua Shenzhen International Graduate School, Tsinghua University\\
  $^{2}$IEG, Tencent \\
  $^{3}$School of Computer Science, Peking University \\
  \texttt{lvjf20@mails.tsinghua.edu.cn, li.xiu@sz.tsinghua.edu.cn}
}

\begin{document}

\maketitle

\begin{abstract}
  Sample efficiency is one of the most critical issues for online reinforcement learning (RL). Existing methods achieve higher sample efficiency by adopting model-based methods, Q-ensemble, or better exploration mechanisms. We, instead, propose to train an off-policy RL agent via updating on a fixed sampled batch multiple times, thus reusing these samples and better exploiting them within a single optimization loop. We name our method \textit{sample multiple reuse} (SMR). We theoretically show the properties of Q-learning with SMR, e.g., convergence. Furthermore, we incorporate SMR with off-the-shelf off-policy RL algorithms and conduct experiments on a variety of continuous control benchmarks. Empirical results show that SMR significantly boosts the sample efficiency of the base methods across most of the evaluated tasks without any hyperparameter tuning or additional tricks.
\end{abstract}

\section{Introduction}

In recent years, the success of reinforcement learning (RL) has been witnessed in fields like games \cite{Mnih2015HumanlevelCT, Silver2016MasteringTG, Vinyals2019GrandmasterLI}, neuroscience \cite{Dabney2020ADC}, fast matrix multiplication \cite{Fawzi2022DiscoveringFM}, and nuclear fusion control \cite{Degrave2022MagneticCO}. 

Online RL, different from batch RL \cite{Lange2012BatchRL}, defines the task of learning an optimal policy via continual interactions with the environment. The agent can generally explore (discover unseen regions) and exploit (use what it already knows) \cite{Sutton2005ReinforcementLA} the data due to the accessibility to the environment. Prior work explores many exploration methods for both discrete \cite{Ecoffet2020FirstRT, Burda2018ExplorationBR} and continuous control \cite{lillicrap2015continuous, Colas2017GEPPGDE} domains. With respect to the exploitation, off-policy deep RL algorithms are known to be more sample-efficient than on-policy methods, as they usually store past experiences and reuse them during training. Unfortunately, most of the off-policy deep RL algorithms, especially on continuous control domains, still need a vast number of interactions to learn meaningful policies. Such a phenomenon undoubtedly barriers the wide application of RL algorithms in real-world problems, e.g., robotics.

\begin{figure}[!htb]
    \centering
    \includegraphics[width=0.4\linewidth]{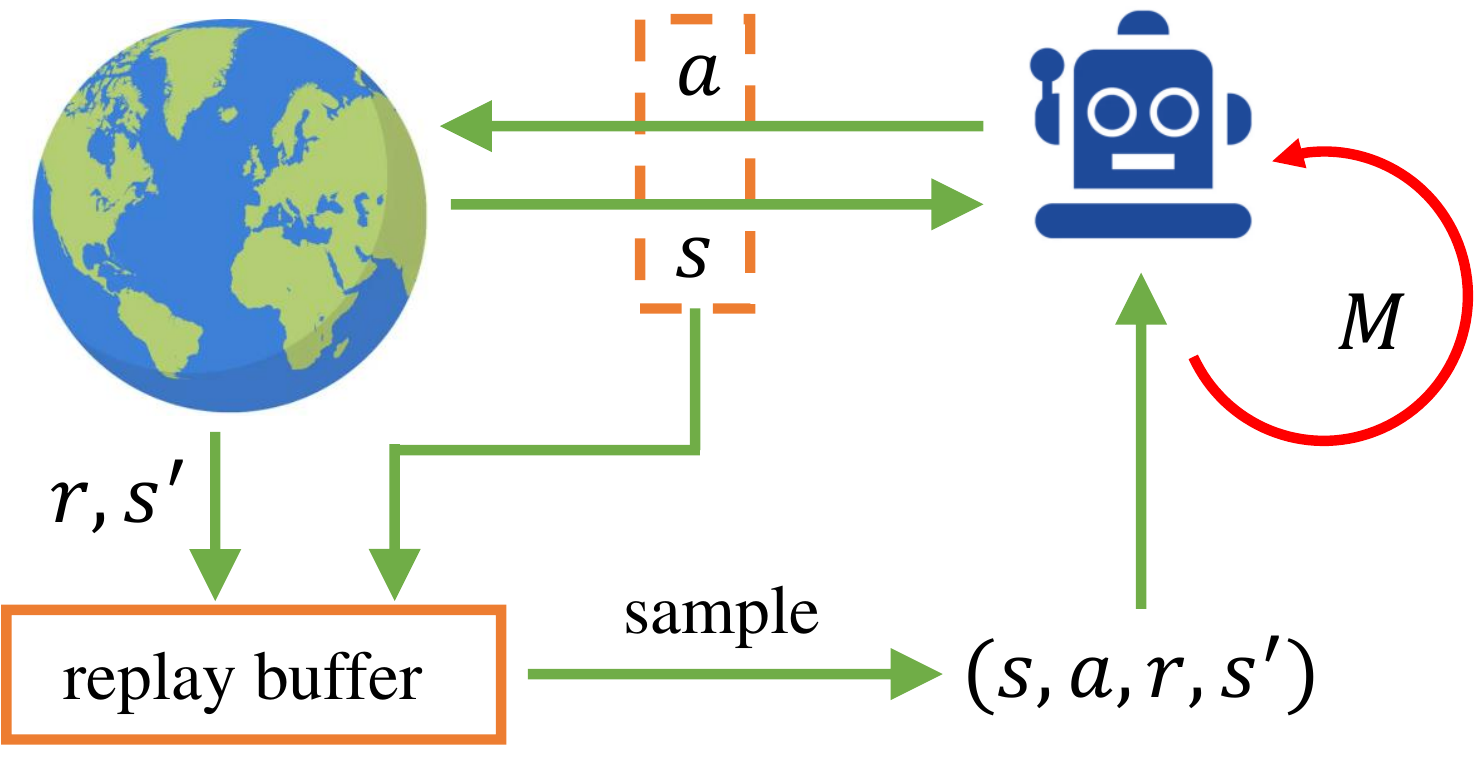}
    \includegraphics[width=0.55\linewidth]{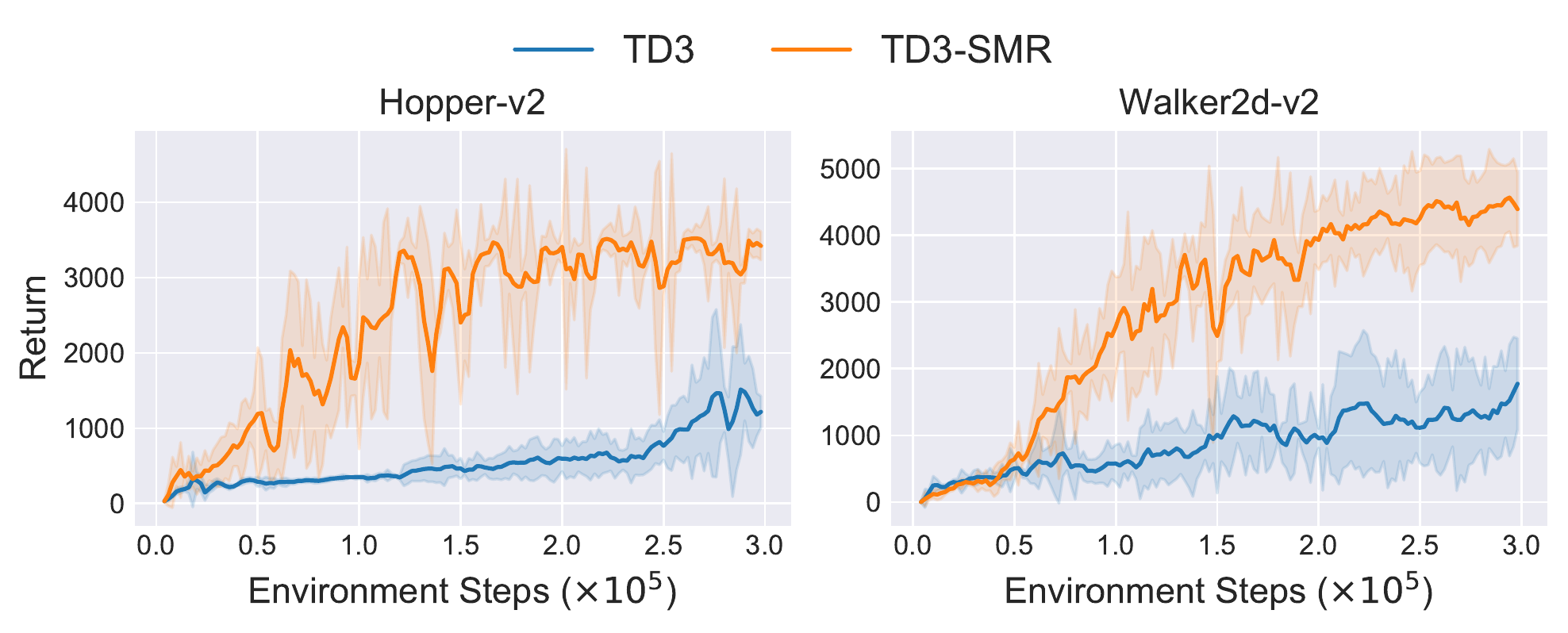}
    \caption{\textbf{Left:} the key idea behind Sample Multiple Reuse (SMR) lies in the {\color{red} red arrow} where we update the agent on the fixed samples for $M$ times. \textbf{Right:} SMR significantly boosts the sample efficiency of TD3 \cite{Fujimoto2018AddressingFA}.}
    \label{fig:example}
\end{figure}

In this paper, we set our focus on continuous control domains. There are many efforts in enhancing the exploration capability of the off-policy RL algorithms by adding extra bonus reward \cite{Tang2016ExplorationAS, Fu2017EX2EW, Houthooft2016CuriositydrivenEI, Achiam2017SurpriseBasedIM}, leveraging maximum entropy framework \cite{ziebart2010modeling, Haarnoja2018SoftAO, haarnoja2018softactorcritic}, etc. Another line of research focuses on better exploiting the data. They achieve this by alleviating the overestimation bias in value estimate \cite{Fujimoto2018AddressingFA, Lee2020SUNRISEAS, Kuznetsov2020ControllingOB, Efficient2022Lyu}, using high update-to-data (UTD) ratio \cite{Chen2021RandomizedED, Hiraoka2021DropoutQF}, adopting model-based methods \cite{Janner2019WhenTT, Lai2020BidirectionalMP, Pan2020TrustTM, wuplan2022}, etc. Nevertheless, these advances often involve complex components like ensemble. We wonder: \textit{is it possible to design a simple method that can universally better exploit data and improve sample efficiency?}


To this end, we propose \textit{sample multiple reuse} (SMR), where we update the actor and the critic network multiple times on the fixed sampled batch data, as shown in Figure \ref{fig:example}. By doing so, the networks can better fit and exploit the batch data (as depicted in Figure \ref{fig:smriteration}). We deem that every collected sample from online interaction is valuable and is worth being utilized more times during training. SMR is general and can be combined with \textit{any} off-the-shelf off-policy continuous control RL algorithms by modifying only a few lines of code.

To illustrate the rationality and benefits of SMR, we combine it with Q-learning and propose Q-SMR algorithm. We theoretically analyze the convergence property of Q-SMR in the tabular case. We empirically show that Q-SMR exhibits stronger sample efficiency than vanilla Q-learning. We then combine SMR with five typical continuous control RL algorithms and run experiments on four tasks from OpenAI Gym \cite{Brockman2016OpenAIG}. We combine SMR with SAC  \cite{Haarnoja2018SoftAO} and extensively evaluate SAC-SMR on two additional continuous control benchmarks, yielding a total of 30 tasks. Across most of the evaluated tasks, we observe improvement in sample efficiency over the base algorithms, often by a large margin (as shown in Figure \ref{fig:example}). The empirical results reveal that SMR is very general and can improve the sample efficiency of different algorithms in a variety of tasks.


To ensure that our proposed method is reproducible \cite{Islam2017ReproducibilityOB, henderson2018deep}, we include the anonymous code in \href{https://anonymous.4open.science/r/SMR-F3F2}{https://anonymous.4open.science/r/SMR-F3F2}, and evaluate the experimental results across fair evaluation metrics. 

\begin{figure}
    \centering
    \includegraphics[width=0.95\linewidth]{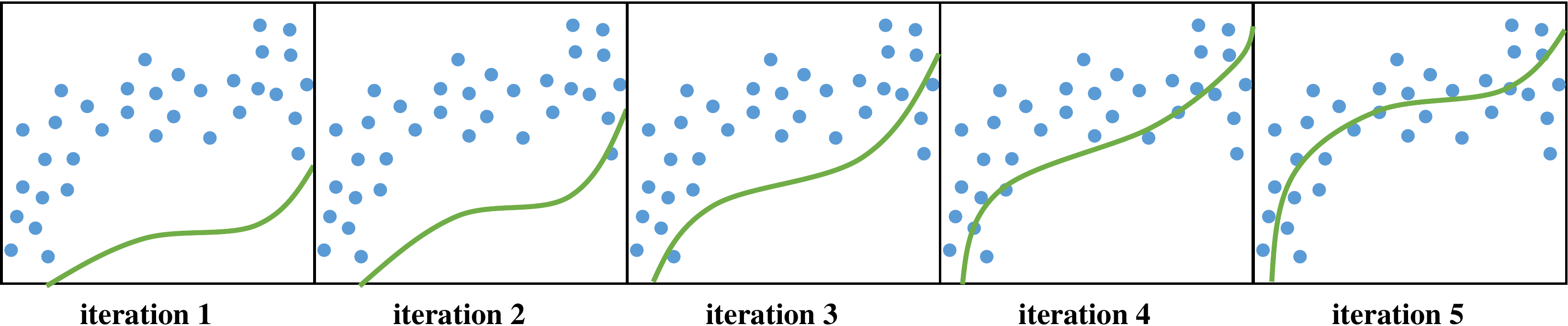}
    \caption{The key idea illustration of sample multiple reuse. The {\color{blue} blue} dots represent the samples in a batch. With only one iteration, it is hard for the approximator (e.g., neural network) to well-fit the data. Whereas, with more updates on the identical batch, the approximator can better fit the samples.}
    \label{fig:smriteration}
    \vspace{-0.4cm}
\end{figure}

\section{Preliminaries}
Reinforcement learning (RL) aims at dealing with sequential decision-making tasks. It can be formulated as a Markov decision process (MDP) defined by a tuple $\langle \mathcal{S}, \mathcal{A}, r, p, \gamma \rangle$. $\mathcal{S}$ is the state space, $\mathcal{A}$ is the action space, $r:\mathcal{S}\times\mathcal{A}\mapsto\mathbb{R}$ is the scalar reward signal, $p(\cdot|s,a)$ is the dynamics transition probability, and $\gamma\in[0,1)$ is the discount factor. In online RL, the agent can continually interact with the environment by following a policy $\pi:\mathcal{S}\mapsto\mathcal{A}$. The goal of the agent is to maximize the expected discounted long-term rewards, i.e.,
\begin{equation}
    \max J(\phi) = \mathbb{E} \left[ \sum_{t=0}^\infty \gamma^t r(s_t,a_t) \bigg| s_0, a_0; \pi \right].
\end{equation}
A policy is said to be \textit{stationary} if it is time-invariant. The state-action value function (also $Q$-function) $Q^\pi:\mathcal{S}\times\mathcal{A}\mapsto\mathbb{R}$ given a policy $\pi$ is defined by 
\begin{equation}
    Q^\pi(s,a) = \mathbb{E}_\pi \left[ \sum_{t=0}^\infty \gamma^t r(s_t,a_t) \bigg| s_0=s, a_0=a\right].
\end{equation}
The optimal $Q$-function $Q^*$ is the unique fixed point of the Bellman operator $\mathcal{T}Q$, which is given by:
\begin{equation}
    \mathcal{T}Q(s,a) := r(s,a) + \gamma\mathbb{E}_{s^\prime\sim p(\cdot|s,a)} [ \max_{a^\prime\in\mathcal{A}}Q(s^\prime,a^\prime)].
\end{equation}

A typical off-policy RL algorithm is Q-learning \cite{Watkins1992Qlearning}. It aims at learning the optimal $Q$-function and updates its entry via the following rule:
\begin{equation}
\label{eq:qlearning}
    Q_{t+1}(s,a) = (1-\alpha_t)Q_t(s,a) + \alpha_t (r_t + \gamma\max_{a^\prime\in\mathcal{A}}Q_t(s^\prime,a^\prime)),
\end{equation}
where $\alpha_t$ is the learning rate at timestep $t$.


\section{Why Not Reuse Your Data More?}

In online deep RL, it is a common practice that we sample a mini-batch in a bootstrapping way from the replay buffer, where the past experience is stored, for training the RL agent. However, existing off-policy RL methods only evaluate \textit{once} upon the sampled transitions, which is a waste since they fail to better exploit the collected valuable samples.

We remedy existing off-policy RL algorithms by reusing the sampled batch data more times. Our key intuition and motivation lie in the fact that it is hard for the neural network to well-fit and well-evaluate the sampled batch with just one glance (check Figure \ref{fig:smriteration}). With more updates on the sampled batch, the network can better adapt to the sample distribution, in conjunction with a more reliable evaluation upon them. We name our method \textit{sample multiple reuse} (SMR), which can be combined with \textit{any} off-policy RL algorithms. We first combine our method with vanilla Q-learning \cite{Watkins1992Qlearning}, yielding the Q-SMR algorithm as depicted in Algorithm \ref{alg:q-smr}. We further define the number of iterations $M$ as the SMR ratio, which measures the fixed batch reusing frequency of the agent. Empirically, Figure \ref{fig:q-smrcliff} illustrates the superior sample efficiency of our proposed Q-SMR algorithm against vanilla Q-learning in the tabular case, where a fixed $M=10$ is utilized for the Q-SMR. In both the classical cliff-walking environment and a maze environment, Q-SMR is able to learn faster and converge faster.

\begin{minipage}{0.48\textwidth}
\begin{algorithm}[H]
    \centering
    \caption{Q-SMR}\label{alg:q-smr}
    \begin{algorithmic}[1]
        \STATE Set learning rate sequence $\{\alpha_t\}$, number of iterations $T$.
        \STATE Initialize $Q(s,a)$ table with 0.
        \FOR{$t$ = 1 to $T$}
        \STATE Choose action $a$ derived from $Q$, e.g., $\epsilon$-greedy, and observe reward $r$ and next state $s^\prime$.
        \color{red}
        \FOR{$m$ = 1 to $M$}
        \STATE Update $Q_t$ according to Equation \ref{eq:qlearning}.
        \ENDFOR
        \ENDFOR
    \end{algorithmic}
\end{algorithm}
\vspace{-2mm}
\includegraphics[width=\linewidth]{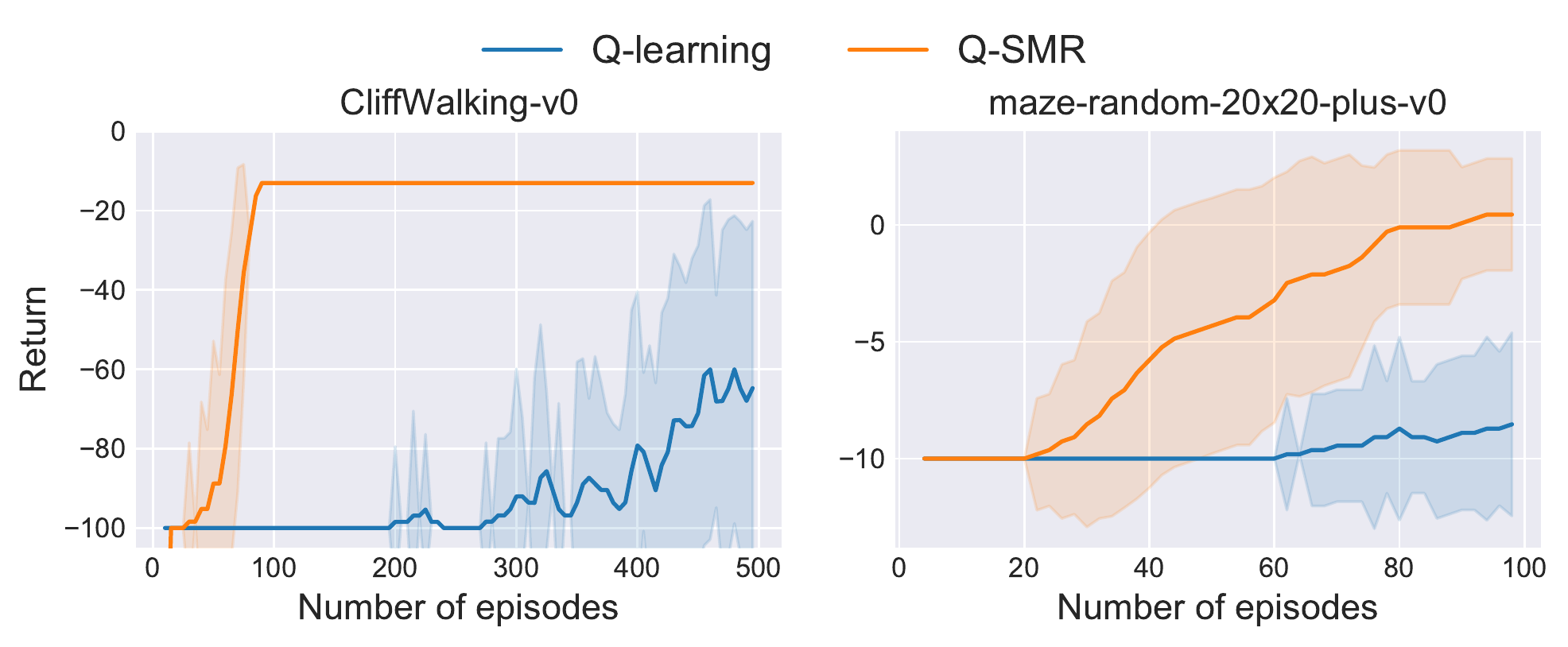}
\captionof{figure}{Comparison of Q-SMR and Q-learning on CliffWalking-v0 and maze-random-20x20-plus-v0 tasks from Gym \cite{Brockman2016OpenAIG}. The results are averaged over 20 independent runs, and the shaded region is the standard deviation.}
\label{fig:q-smrcliff}
\end{minipage}
\hfill
\begin{minipage}{0.47\textwidth}
\begin{algorithm}[H]
    \centering
    \caption{Off-policy actor-critic with SMR}\label{alg:ac-smr}
    \begin{algorithmic}[1]
        \STATE Initialize critic network parameter $\theta$, actor network parameter $\phi$ with random parameters.\\
        \STATE Initialize target critic network parameter $\theta^\prime\leftarrow \theta$.
        \STATE Initialize empty replay buffer $\mathcal{B}=\{\}$. \\
        \STATE (Optional) Initialize target actor network parameter $\phi^\prime\leftarrow \phi$.
        \FOR{$t$ = 1 to $T$}
        \STATE Choose action $a$ and observe reward $r$, next state $s^\prime$.
        \STATE Store the transition in the replay buffer, i.e., $\mathcal{B}\leftarrow \mathcal{B}\cup \{(s,a,r,s^\prime)\}$.
        \STATE Sample $N$ transitions $\{(s_j,a_j,r_j,s^\prime_j)\}_{j=1}^{N}$ from $\mathcal{B}$.
        \color{red}
        \FOR{$m$ = 1 to $M$}
        \STATE Update critic by minimizing Bellman error.
        \STATE Update actor with policy gradient.
        \STATE Update target network.
        \ENDFOR
        \ENDFOR
    \end{algorithmic}
\end{algorithm}
\end{minipage}




Moreover, our method can also be incorporated with any off-policy (deep) RL algorithms, and the experimental results in Figure \ref{fig:q-smrcliff} shed light on doing so. We detail the (abstracted) off-policy actor-critic with SMR in Algorithm \ref{alg:ac-smr}. Compared to typical actor-critic methods, our revised algorithm only enforces the agent to train on identical batch data multiple times. This requires a minimal change to the base algorithm, which can be completed by modifying a few lines of code. We defer the detailed pseudo-code of various off-policy algorithms with SMR in Appendix \ref{sec:pseudocodes}.



\section{Theoretical Analysis}
In this section, we aim at showing the theoretical properties of the Q-SMR algorithm in the tabular case. The theoretical guarantee of Q-SMR can pave the way for applying SMR in complex continuous control tasks. All missing proofs can be found in Appendix \ref{sec:missingproof}.

We consider asynchronous Q-learning \cite{EvenDar2004LearningRF, Li2020SampleCO} which follows the update rule:
\begin{equation}
    \begin{aligned}
    &Q_{t+1}(s_t,a_t) = (1-\alpha_t)Q_t(s_t,a_t)+\alpha_t \mathcal{T}_{t+1}Q_{t}(s_t,a_t), \\
    &Q_{t+1}(s,a) = Q_t(s,a) \quad \forall\, (s,a)\neq (s_t,a_t),
    \end{aligned}
\end{equation}
where $\mathcal{T}_{t+1}Q_{t}(s_t,a_t)=r_t+\gamma\max_{a^\prime\in\mathcal{A}}Q_t(s_{t+1},a^\prime)$. We have access to a sample trajectory $\{s_t,a_t,r_t\}_{t=0}^\infty$ from a behavior policy $\pi_b$, and we only update one $(s,a)$-entry each step here. 

Given the SMR ratio $M$, we define $Q_t^{(i)}(s,a), i\in[1,M]$ as the intermediate $Q$-function at timestep $t$ and iteration $i$. The resulting $Q$-function after SMR iteration is $Q_t(s,a)$ where we omit superscript $(M)$ for $Q_t(s,a)$. We define $Q_{t+1}^{(0)}(s,a)=Q_t^{(M)}(s,a)$. We first give the update rule for Q-SMR that is equivalent to the loop (line 4-6 in Algorithm \ref{alg:q-smr}) in Theorem \ref{theo:updaterule}.

\begin{theorem}
\label{theo:updaterule}
The update rule of Q-SMR is equivalent to:
\begin{equation}
    \begin{aligned}
    &Q_{t+1}(s_t,a_t) = (1-\alpha_t)^M Q_{t}(s_t,a_t)  + \sum_{i=0}^{M-1} \alpha_t (1-\alpha_t)^i \mathcal{T}_{t+1}Q_{t+1}^{(M-1-i)}(s_t,a_t), \\
    &Q_{t+1}(s,a) = Q_t(s,a) \quad \forall\, (s,a)\neq (s_t,a_t),
    \end{aligned}
\end{equation}
where $\mathcal{T}_{t+1}Q_{t+1}(s_t,a_t) = r_t + \gamma\max_{a^\prime\in\mathcal{A}}Q_{t+1}(s_{t+1},a^\prime)$ denotes the empirical Bellman operator w.r.t. timestep $t+1$.
\end{theorem}

\noindent\textbf{Remark:} The update rule of Q-SMR relies on the intermediate value during the SMR iteration. The influence of the current Q-value in Q-SMR is reduced (as $(1-\alpha_t)^M\le (1-\alpha_t)$), and hence the value estimate can change faster by querying the maximal value. This we believe can partly explain the superior sample efficiency of Q-SMR against vanilla Q-learning depicted in Figure \ref{fig:q-smrcliff}.


\begin{assumption}
\label{ass:mdp}
Assume that $\forall t$, the reward signal is bounded, $|r_t|\le r_{\rm max}$.
\end{assumption}

We note that this is a widely used assumption, which can also be easily satisfied in practice as many reward functions are hand-crafted. We then show in Theorem \ref{theo:stability} that Q-SMR outputs a bounded value estimate throughout its iteration. 

\begin{theorem}[Stability]
\label{theo:stability}
Let Assumption \ref{ass:mdp} hold and assume the initial $Q$-function is set to be 0, then for any iteration $t$, the value estimate induced by Q-SMR, $\hat{Q}_t$, is bounded, i.e., $|\hat{Q}_t|\le \dfrac{r_{\rm max}}{1-\gamma},\forall t$.
\end{theorem}

We further show that the Q-SMR algorithm is guaranteed to converge to the optimal Q-value, which reveals the rationality of utilizing the Q-SMR algorithm in practice and paves the way for extending the Q-SMR algorithm into deep RL scenarios.
\begin{theorem}[Convergence]
\label{theo:convergence}
Under some mild assumptions that are similar to \cite{Fujimoto2018AddressingFA,melo2001convergence}, the Q-SMR algorithm converges to the optimal $Q$-function.
\end{theorem}

Interestingly, we can establish a connection between modified learning rate and SMR update rule by assuming that the underlying MDP is \textit{nonreturnable}, i.e., $s_{t+1}\neq s_t$. Then, the rule can be simplified.

\begin{corollary}
\label{coro:simpleupdaterule}
If the MDP is nonreturnable, the update rule of Q-SMR gives:
\begin{equation}
    \label{eq:simpleupdaterule}
    \begin{aligned}
    &Q_{t+1}(s_t,a_t) = (1-\alpha_t)^M Q_{t}(s_t,a_t) + \left[ 1-(1-\alpha_t)^M \right] \mathcal{T}_{t+1}Q_t(s_t,a_t), \\
    &Q_{t+1}(s,a) = Q_t(s,a) \quad \forall\, (s,a)\neq (s_t,a_t).
    \end{aligned}
\end{equation}
\end{corollary}
\noindent\textbf{Remark:} Compared to vanilla Q-learning, this rule actually \textit{modifies} the learning rate from $\alpha_t$ to $1-(1-\alpha_t)^M$. Since $\alpha_t\in[0,1]$, it is easy to see $1 - (1-\alpha_t)^M\in[0,1],\forall\,t$.

Furthermore, we can derive the finite time error bound of the Q-SMR algorithm based on the above corollary, which improves over the prior results \cite{Szepesvari1997TheAC, EvenDar2004LearningRF,Qu2020FiniteTimeAO}. Please refer to Appendix \ref{sec:additionaltheory} for more details and discussions.


\section{Experiments}
\begin{figure*}[!htb]
    \centering
    \includegraphics[width=0.95\linewidth]{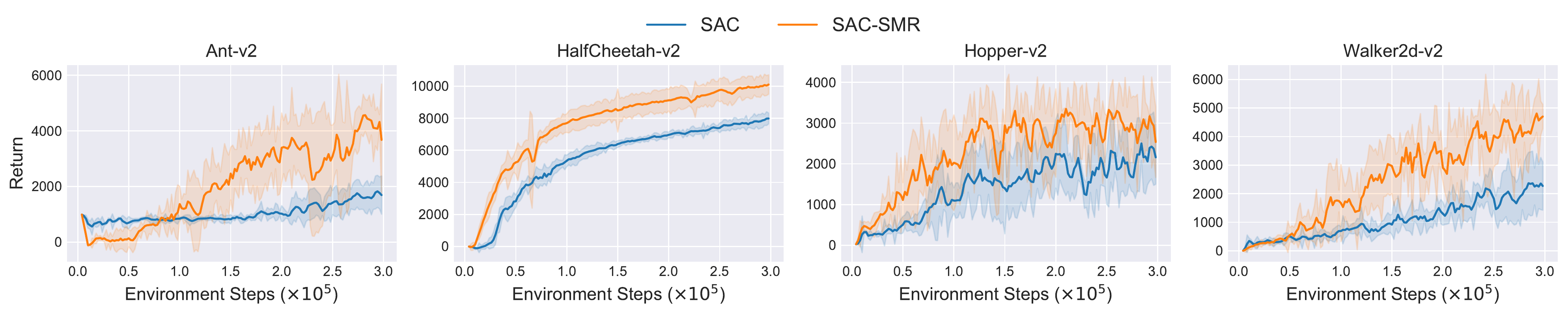}
    \includegraphics[width=0.95\linewidth]{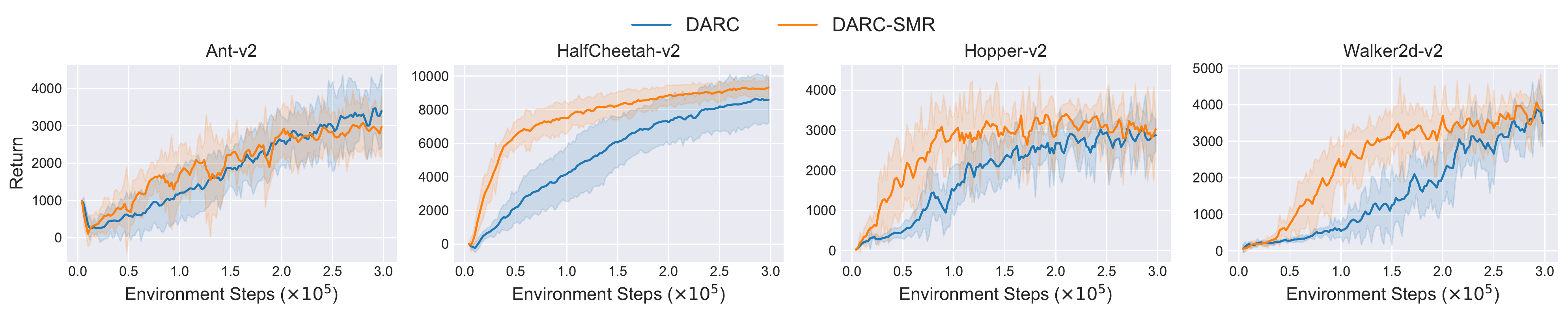}
    \includegraphics[width=0.95\linewidth]{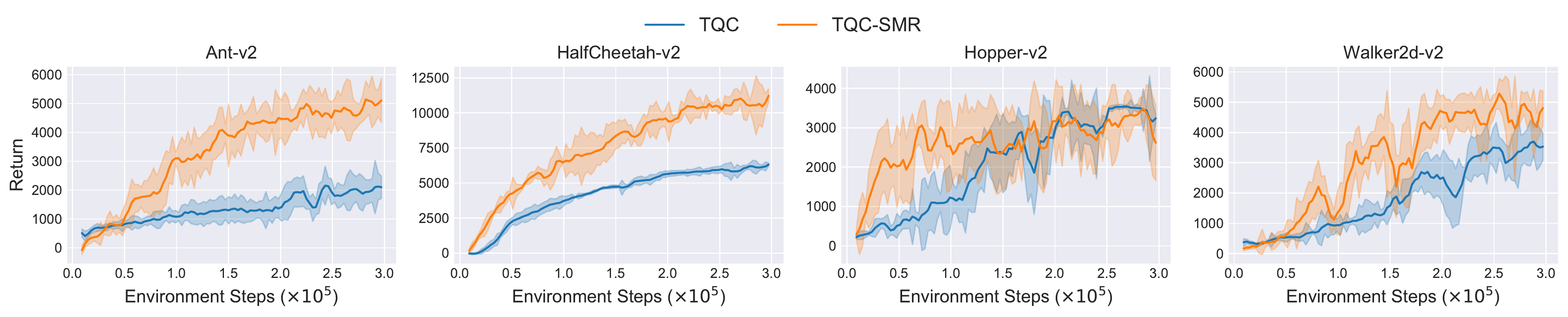}
    \includegraphics[width=0.95\linewidth]{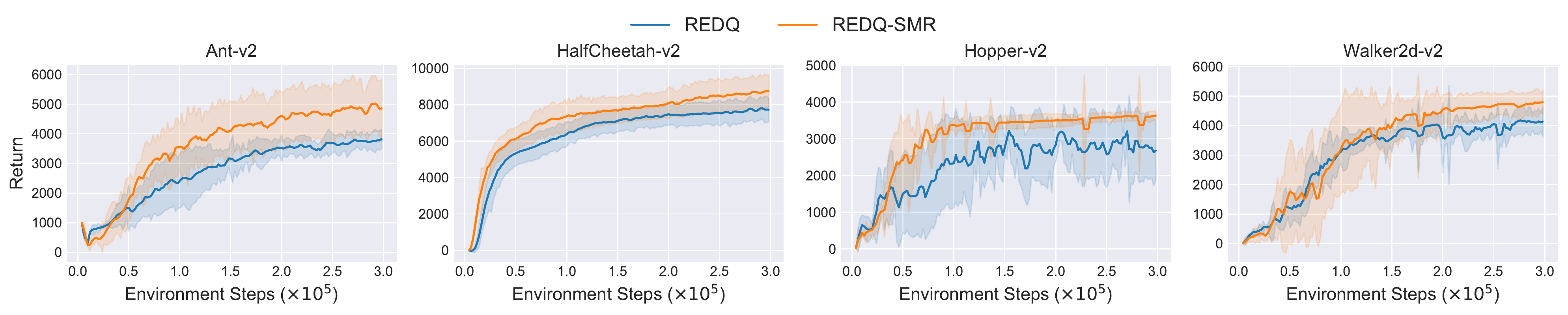}
    \caption{Experimental results of four typical continuous control algorithms with and w/o SMR on four OpenAI Gym \cite{Brockman2016OpenAIG} environments. The results are averaged over 6 independent runs. The shaded region denotes the standard deviation.}
    \label{fig:gymresult}
\end{figure*}

In this section, we investigate the benefits of SMR upon off-the-shelf off-policy continuous control RL algorithms. We aim at answering the following questions: (1) is the SMR general enough to benefit wide off-policy RL algorithms? (2) how much performance gain can off-policy RL algorithms acquire by using SMR? 

In order to show the strong data exploitation ability and the generality of SMR, we combine SMR with TD3 \cite{Fujimoto2018AddressingFA}, SAC \cite{haarnoja2018softactorcritic}, DARC \cite{Efficient2022Lyu}, TQC \cite{Kuznetsov2020ControllingOB}, and REDQ \cite{Chen2021RandomizedED}. We choose these methods as they typically represent different categories of continuous control algorithms, i.e., TD3 leverages clipped double Q-learning, SAC is based on the maximum entropy RL, DARC enhances the agent's exploration capability by using double actors, TQC addresses overestimation by incorporating distributional RL into the continuous setting, and REDQ is the state-of-the-art model-free RL method which trains critic ensemble and uses a high update-to-data (UTD) ratio.

Besides the loop of reusing samples (line 5-9 in Algorithm \ref{alg:ac-smr}), we do not make any additional modifications (e.g., parameter tuning) to the base algorithm. We run experiments on four continuous control tasks from OpenAI Gym \cite{Brockman2016OpenAIG} simulated by MuJoCo \cite{Todorov2012MuJoCoAP}. All methods are run for 300K online interactions where we adopt the SMR ratio $M=10$ by default except REDQ where we set $M=5$ (as REDQ already uses a large UTD ratio). We note that 300K is a typical interaction step adopted widely in prior work \cite{Chen2021RandomizedED,Janner2019WhenTT,Hansen2022TemporalDL} for examining sample efficiency.

Each algorithm is repeated with 6 random seeds and evaluated over 10 trials every 1000 timesteps. We find that SMR significantly improves the sample efficiency of the base algorithms on almost every task, often outperforming them by a large margin (see Figure \ref{fig:gymresult}). SAC-SMR achieves 4x and TQC-SMR has 3x sample efficiency than the base algorithm as shown in Table \ref{tab:sampleefficiency}. Notably, SAC-SMR takes only 93K online interactions to reach 3000 in Hopper-v2, and TQC-SMR takes merely 34K online interactions. The results even match the performance of MBPO \cite{Janner2019WhenTT} (around 73K). We show in Appendix \ref{sec:missingexperiments} that other off-policy RL algorithms like DDPG, DrQ-v2 \cite{yarats2022mastering} can benefit from SMR as well. These altogether reveal that {\color{red} the advantage of SMR is algorithm-agnostic.}

\begin{figure*}[!htb]
    \centering
    \vspace{-0.2cm}
    \includegraphics[width=0.95\linewidth]{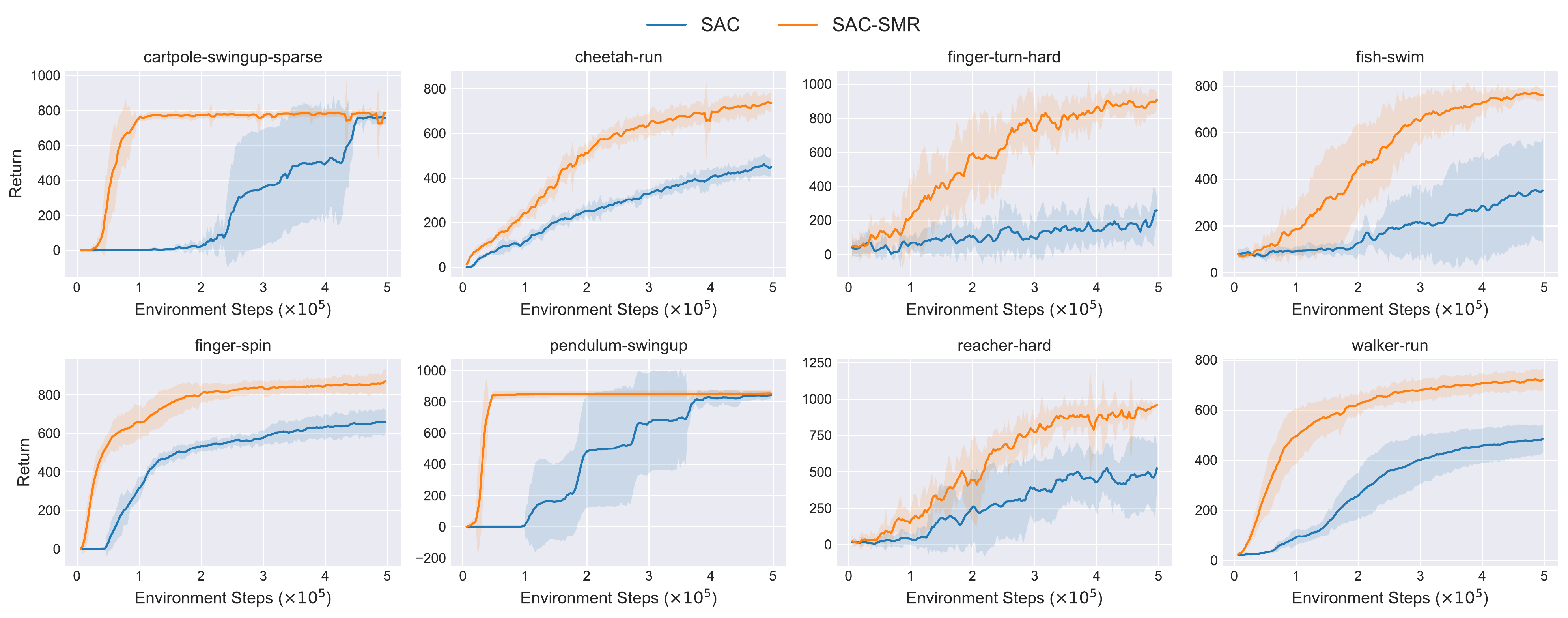}
    \includegraphics[width=0.95\linewidth]{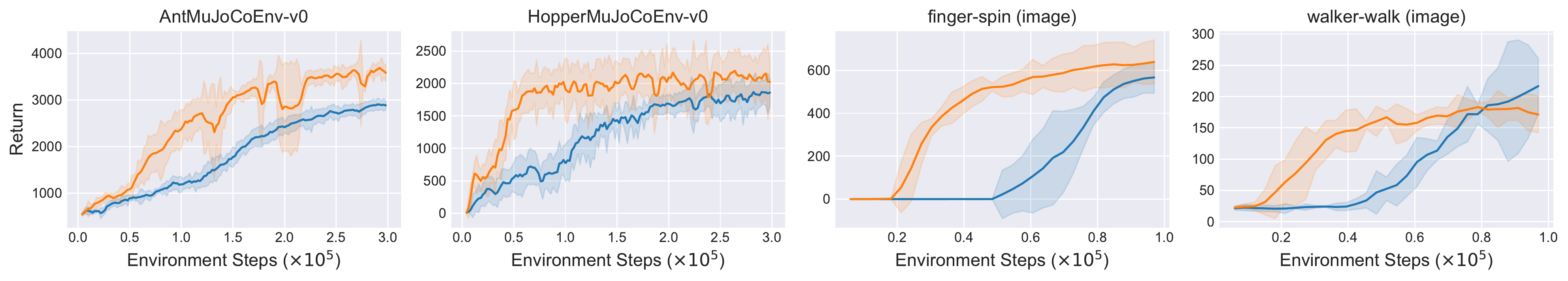}
    \caption{Experimental results of SAC-SMR against vanilla SAC on 8 state-based, 2 image-based DMC suite \cite{Tassa2018DeepMindCS} tasks and 2 PyBullet-Gym \cite{benelot2018} tasks. The results are averaged over 6 seeds. The shaded region captures the standard deviation.}
    \label{fig:sacdmc8}
\end{figure*}

\begin{table}[!htb]
  \vspace{-0.7cm}
  \caption{Sample efficiency comparison. We choose SAC, TQC and DARC as examples. The numbers indicate the number of online interactions when the specified performance level is reached.}
  \renewcommand\arraystretch{1.05}
  \label{tab:sampleefficiency}
  \setlength{\tabcolsep}{4pt}
  \centering
  \small
  \begin{tabular}{l|ll|ll|ll}
    \toprule
    Score  & SAC & SAC-SMR & TQC & TQC-SMR & DARC & DARC-SMR \\
    \midrule
    Hopper@3000 & 373K & \textbf{93K} & 160K & \textbf{34K} & 205K & \textbf{67K} \\
    Ant@4000 & 982K & \textbf{211K} & 469K & \textbf{135K} & 324K & \textbf{305K} \\
    HalfCheetah@10000 & 860K & \textbf{282K} & 576K & \textbf{185K} & 407K & \textbf{324K} \\  
    Walker2d@4000 & 656K & \textbf{164K} & 281K & \textbf{133K} & 292K & \textbf{264K} \\
    \bottomrule
  \end{tabular}
\end{table}

We further combine SMR with SAC and run SAC-SMR extensively on two additional continuous control benchmarks, DMC suite \cite{Tassa2018DeepMindCS} and PyBullet-Gym \cite{benelot2018}. We conduct experiments on 20 DMC suite tasks, 4 PyBullet-Gym tasks, and 6 image-based tasks from DMC suite, yielding a total of \textbf{30} tasks. For state-based tasks, we use $M=10$ and run for 500K interactions. For image-based tasks, as it is very time-consuming with $M=10$, we use $M=5$, which we find is sufficient to show the advantage of SMR. Both SAC and SAC-SMR are evaluated over 10 trials every 1000 timesteps. It can be seen in Figure \ref{fig:sacdmc8} that SMR significantly boosts the sample efficiency of SAC on the evaluated tasks. This can also be validated from Table \ref{tab:performance} where SAC-SMR achieves 2.5x the performance of SAC at 250K and 2.0x the performance of SAC at 500K when averaging the numbers.

Due to the space limit, we defer some results to Appendix \ref{sec:missingexperiments} and only report a small proportion of tasks here. These experimental results show that {\color{red} the advantage of SMR is task-agnostic.} In summary, we believe the above evidence is enough to verify the generality and effectiveness of SMR.

\begin{table}
  \caption{Performance comparison of SAC, SAC-UTD (UTD ratio $G$=10) and SAC-SMR. We choose \texttt{cheetah-run} and \texttt{fish-swim} as examples. The numbers indicate the performance achieved when the specific number of data is collected. $\pm$ captures the standard deviation.}
  \vspace{0.2cm}
  \renewcommand\arraystretch{1.05}
  \label{tab:performance}
  \centering
  \small
  \begin{tabular}{l|lll}
    \toprule
    Amount of data  & SAC & SAC-UTD & SAC-SMR \\
    \midrule
    cheetah-run@250K & 284.6$\pm$20.5 & 434.1$\pm$72.6 & \textbf{600.1}$\pm$49.2 \\
    fish-swim@250K & 178.5$\pm$113.9 & 382.5$\pm$70.7 & \textbf{544.3}$\pm$184.4 \\  
    cheetah-run@500K & 452.1$\pm$47.7 & 633.9$\pm$99.1 & \textbf{725.4}$\pm$48.7 \\
    fish-swim@500K & 324.8$\pm$213.9 & 712.0$\pm$41.9 & \textbf{756.3}$\pm$38.7 \\
    \bottomrule
  \end{tabular}
\end{table}

\noindent\textbf{Parameter Study.} The most critical hyperparameter in our method is the SMR ratio. It controls the frequency we reuse a fixed batch. Intuitively, we ought not to use too large $M$ to prevent potential overfitting in neural networks.  For state-based tasks, we find that setting $M=10$ can incur very satisfying performance. In order to see the influence of the SMR ratio $M$, we conduct experiments on Ant-v2 and HalfCheetah-v2 from OpenAI Gym \cite{Brockman2016OpenAIG}. We sweep $M$ across $\{1,2,5,10,20\}$ and demonstrate in Figure \ref{fig:smr-ratio} that SMR can improve the sample efficiency of the base algorithm even with a small $M=2$, and the sample efficiency generally increases with larger $M$. We do not bother tuning $M$ and keep it fixed across our experiments.

\noindent\textbf{Computation Budget.} SMR consumes more computation budget than its base algorithm due to multiple updates on the fixed batch. Intuitively, our method will require more training time with a larger SMR ratio $M$. Typically, SMR ($M=10$) will take about 3-5 times of more training time, e.g., SAC-SMR takes around 6 hours for 300K interactions on Walker2d-v2, while SAC takes around 1.5 hours. Such cost is tolerable for \textit{state-based} tasks considering the superior sample efficiency improvement with SMR.

\noindent\textbf{Clarification on the Asymptotic Performance.} As we focus on improving the sample efficiency, the asymptotic performance of SMR upon different base methods lies out of the scope of this work. Nevertheless, readers of interest can find that the asymptotic performance of SMR is quite good (please refer to Appendix \ref{sec:smrlonger} where we run SMR upon different algorithms for longer interactions).


\begin{figure}[!htb]
    \centering
    \includegraphics[width=0.85\linewidth]{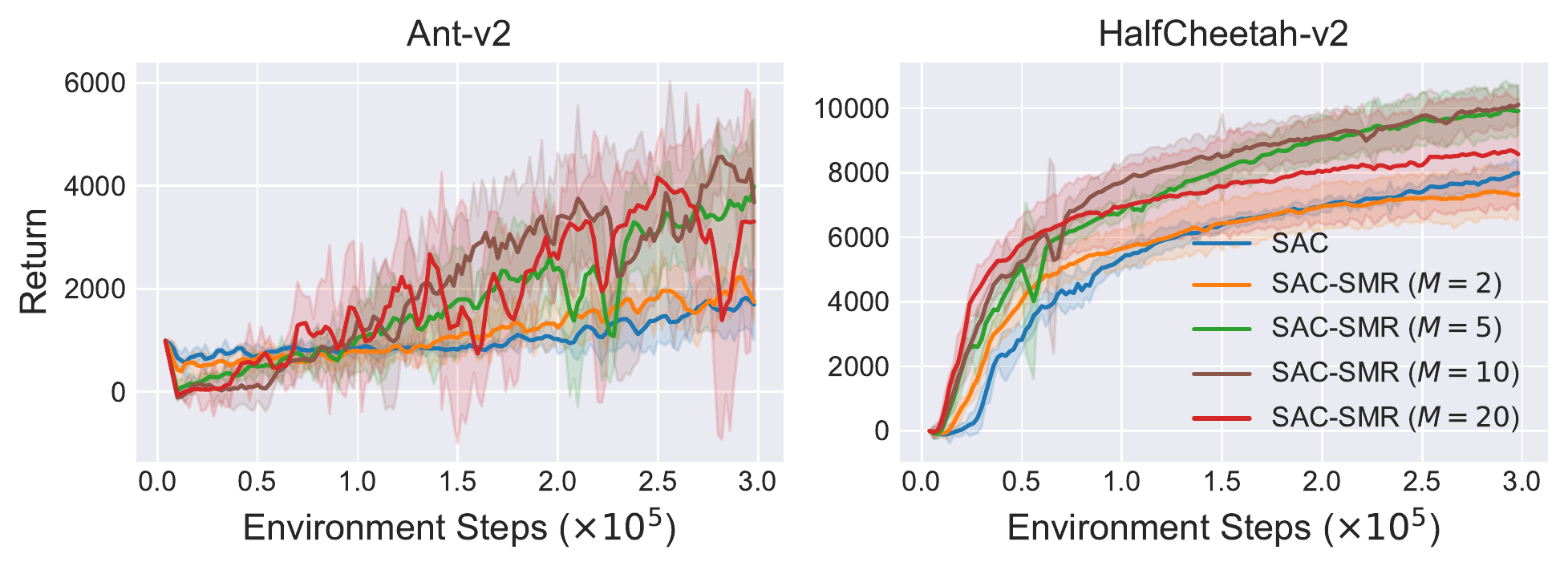}
    \caption{The performance of SAC-SMR under different SMR ratios on two selected environments. The results are averaged over 6 runs and the shaded area captures the standard deviation.}
    \label{fig:smr-ratio}
    \vspace{-0.5cm}
\end{figure}

\section{Discussions}

\subsection{Is SMR equivalent to enlarging learning rate?}
One may think that SMR is equivalent to amplifying learning rate $M$ times at first sight, i.e., $\alpha_t \rightarrow M\alpha_t$. Whereas, we argue that they are \textit{quite different}. In the tabular case, we show in Theorem \ref{theo:updaterule} the update rule for Q-SMR, which is obviously not the rule that enlarges the original learning rate sequence $M$ times. In deep RL, suppose the (single) critic and actor are parameterized by $\theta$ and $\phi$, respectively. The objective function of the critic gives:
\begin{equation}
    \label{eq:criticobjective}
    \mathcal{L}_\theta = \mathbb{E}_{s,a,s^\prime\sim\rho}\left[ (Q_\theta(s,a) - r - \gamma Q_{\theta^\prime}(s^\prime, a^\prime))^2 \right],
\end{equation}
where $a^\prime\sim\pi_\phi$, $\rho$ is the sample distribution in the replay buffer, $\theta^\prime$ is the parameter of the target network. Deep neural networks are typically trained with stochastic gradient descent (SGD) \cite{LeCun2015DeepL, loshchilov2017sgdr, Bottou2010LargeScaleML}. The critic is optimized using the gradient information $\nabla \mathcal{L}_{\theta_t}$ obtained on the $t$-th batch, i.e., $\theta_{t+1} = \theta_t - \alpha_t \nabla \mathcal{L}_{\theta_t}$. We then show that, in deep RL, SMR is also not equivalent to enlarging learning rate.
\begin{theorem}
\label{theo:difference}
    Denote $\theta_t^{(i)}$ as the intermediate parameter in the SMR loop at timestep $t$ and iteration $i$, then in deep RL, the parameter update using SMR satisfies:
    \begin{equation}
    \label{eq:smrlearningrate}
    \theta_{t+1} = \theta_t - {\color{red} \alpha_t \sum_{i=0}^{M-1} \nabla\mathcal{L}_{\theta_{t+1}^{(i)}}} \neq \theta_t - {\color{red} M\alpha_t \nabla \mathcal{L}_{\theta_t}}.
    \end{equation}
\end{theorem}

The inequality in the above theorem is due to the fact that $\theta_{t+1}^{(i+1)}\neq \theta_{t+1}^{(i)}$. A natural question is then raised: how does SMR compete against magnifying the learning rate? 

\begin{minipage}{0.5\textwidth}
    \includegraphics[width=0.95\linewidth]{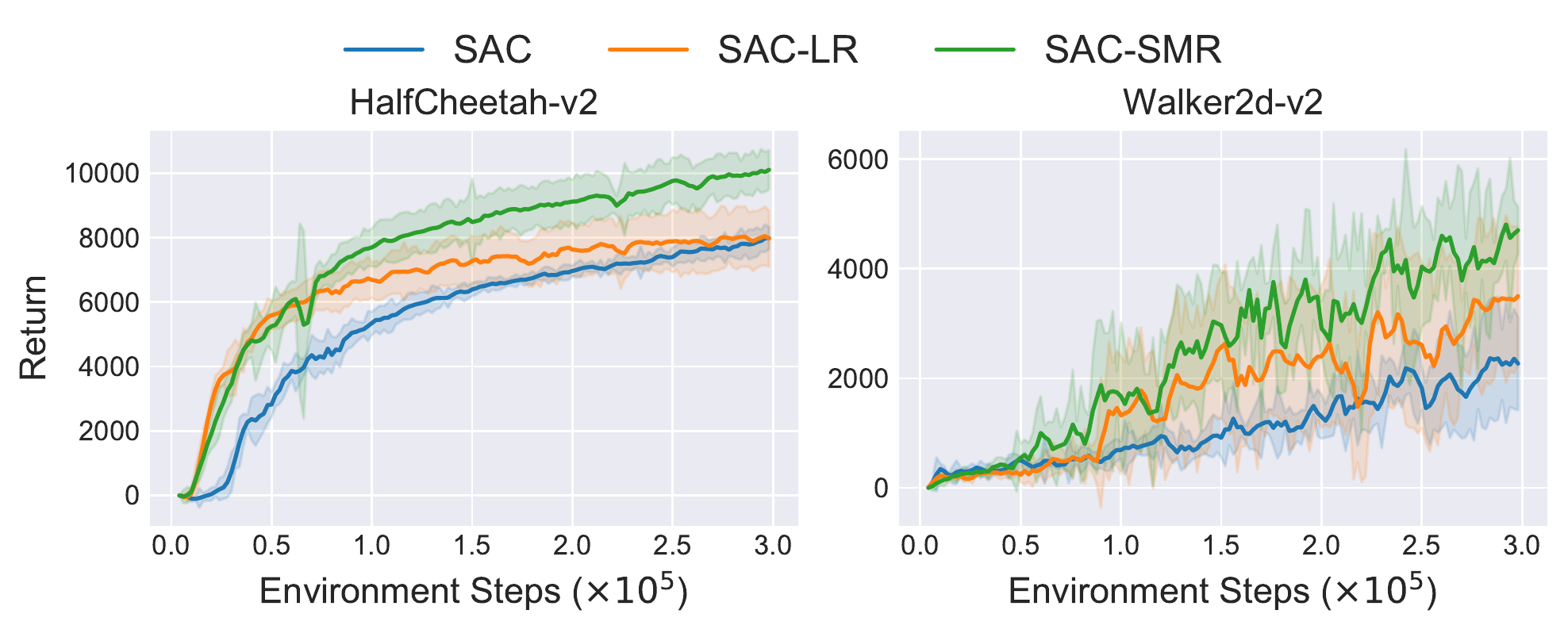}
    \captionof{figure}{Comparison of SAC-SMR ($M=10$) against SAC-LR (i.e., amplify the learning rate 10 times). Each algorithm is repeated with 6 seeds and evaluated over 10 trials every 1000 timesteps. We report the mean performance and the standard deviation.}
    \label{fig:enlargelr}
\end{minipage}
\hspace{0.03\linewidth}
\begin{minipage}{0.5\textwidth}
    \includegraphics[width=0.95\linewidth]{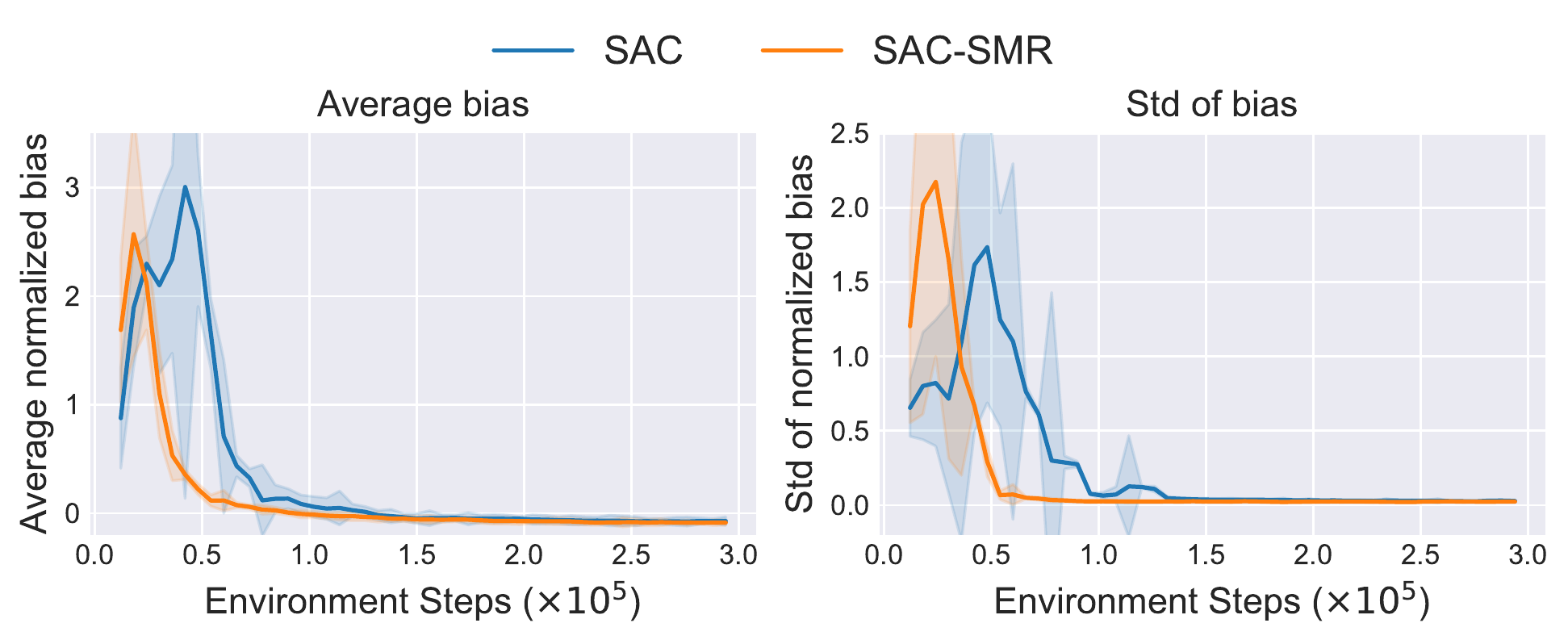}
    \captionof{figure}{Normalized bias comparison of SAC and SAC-SMR on HalfCheetah-v2. SAC-SMR exhibits overfitting at first (with both larger average bias and std of bias) while can incur smaller estimation bias very quickly.}
    \label{fig:sac-bias}
\end{minipage}

We answer this by conducting experiments on two selected environments from OpenAI Gym \cite{Brockman2016OpenAIG}. Empirical results in Figure \ref{fig:enlargelr} show that enlarging the learning rate does aid performance gain, yet it still underperforms SMR in sample efficiency. It is trivial to find the best learning rate. SMR, instead, can benefit the base algorithm with the default parameter (see more evidence in Appendix \ref{sec:missingexperiments}).

\subsection{Concerns on overfitting}
\label{sec:overfitting}
One may wonder whether the phenomenon of overfitting \cite{Dietterich1995OverfittingAU, Srivastava2014DropoutAS} will occur in SMR since we optimize the networks on fixed samples for multiple times. The networks may overfit the collected samples at first, but they can get rid of this dilemma and end up with better data exploitation later on with \textit{reasonable} $M$. We verify this by measuring the accuracy of $Q^\pi(s,a)$ over the state-action distribution of the current policy $\pi$ against its true value $Q(s,a)$ (i.e., discounted Monte Carlo return). Since the Monte Carlo return can change drastically during training, we adopt normalized estimation bias $\frac{Q^\pi(s,a)-Q(s,a)}{|\mathbb{E}_{\hat{s},\hat{a}\sim\pi}[Q(\hat{s},\hat{a})]|}$ for more meaningful comparison. We conduct experiments on HalfCheetah-v2. We run each algorithm with 6 seeds for 300K online interactions and evaluate them over 10 trials every 1000 timesteps. We adopt the same way of calculating the normalized estimation bias as REDQ \cite{Chen2021RandomizedED}. As illustrated in Figure \ref{fig:sac-bias}, SMR incurs slight overfitting at the beginning of training, while it can quickly escape from it and result in a smaller estimation bias afterwards.

This may be because the networks can well-fit new transitions from continual online interactions with multiple updates. Since SMR uses much fewer gradient steps per interaction with the environment compared with REDQ (with UTD $G=20$), we believe the concerns on overfitting can be mitigated to some extent. As a piece of evidence, we do not find any performance degradation with $M=10$ across a wide range of algorithms and tasks. The key for not overfitting is the appropriate choice of SMR ratio $M$. No wonder that it will be hard for the agent to get rid of overfitting with too large $M$ (e.g., $M=10^5$, also referred to as \emph{heavy priming} phenomenon in \cite{Nikishin2022ThePB}). For those who still worry about overfitting, we can remedy this by: (1) using a small $M$, e.g., $M=5$; (2) resetting the agent periodically \cite{Nikishin2022ThePB} such that it forgets past learned policy; (3) leveraging a larger batch size; etc. Note that one does not have to stick to adopting a high SMR ratio throughout the training process, and can use SMR as a \emph{starting point}, or a warm-up phase, e.g., one can use $M=10$ for 300K interactions and then resume vanilla training process (i.e., $M=1$), which can also relieve potential overfitting.


\subsection{Comparison with UTD (update-to-data)}

SMR focuses on boosting the sample efficiency of model-free algorithms by better exploiting collected samples. This is similar in spirit to model-based methods (e.g., MBPO \cite{Janner2019WhenTT}) and REDQ \cite{Chen2021RandomizedED} as they usually employ a large update-to-data (UTD) ratio, i.e., update the critic multiple times by sampling with bootstrapping (the sampled batch is different each time). However, SMR updates both actor and critic on the \textit{fixed} sampled batch multiple times to better fit the data (as Figure \ref{fig:smriteration} shows). 

It is interesting to examine which way of reusing data can benefit the agent more. To answer this question, we compare SAC-SMR ($M=10$) against SAC-UTD (UTD $G=10$) and vanilla SAC on four DMC suite tasks. We run each algorithm for 1M online interactions. One can see that with the identical gradient steps per interaction with the environment, SMR achieves much better final performance and sample efficiency than UTD, as shown in Figure \ref{fig:replaycomp} and Table \ref{tab:performance}, indicating that SMR may be a better choice in practice. We remark here that the success of UTD in REDQ is attributed to a much higher UTD ratio and randomized critic ensemble. However, SMR does not rely on any specific component and can consistently improve the performance of the base algorithm. Meanwhile, we do not view SMR and UTD as contradictory methods, but rather orthogonal methods, e.g., one can find in Figure \ref{fig:gymresult} that SMR improves the sample efficiency of REDQ.

\begin{figure}
    \centering
    \includegraphics[width=\linewidth]{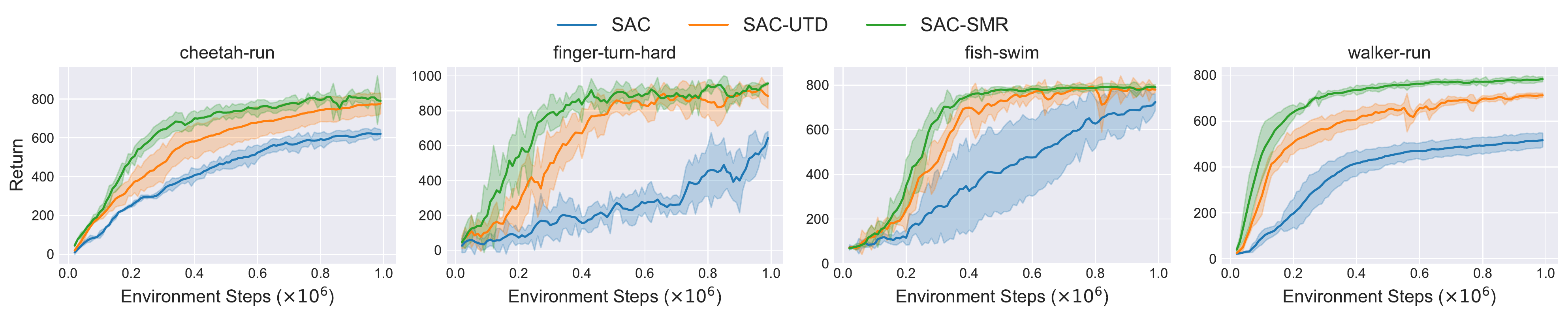}
    \caption{Comparison of SAC-SMR ($M=10$) against SAC-UTD ($G=10$) under identical update frequency. The results are averaged over 6 runs, and the standard deviation is also reported.}
    \label{fig:replaycomp}
    \vspace{-0.6cm}
\end{figure}

\section{Related Work}
\label{sec:relatedwork}
\noindent\textbf{Off-policy RL algorithms.} Recently, we have witnessed the great success of off-policy algorithms in discrete settings since DQN \cite{Mnih2015HumanlevelCT}. There are many improvements upon it, including double Q-learning \cite{Hasselt2010DoubleQ, Hasselt2015DeepRL}, dueling structure \cite{Wang2015DuelingNA}, distributional perspective \cite{Bellemare2017ADP, Nam2021GMACAD, Dabney2017DistributionalRL}, experience replay techniques \cite{Schaul2015PrioritizedER, Hessel2017RainbowCI}, self-supervised learning \cite{Laskin2020ReinforcementLW, schwarzer2021dataefficient, Srinivas2020CURLCU}, model-based methods \cite{Kaiser2019ModelBasedRL, Ye2021MasteringAG, Schrittwieser2019MasteringAG, hamrick2021on, Hessel2021MuesliCI}, etc.

In the continuous control domain, off-policy RL algorithms are widely adopted, such as DDPG \cite{lillicrap2015continuous} and TD3 \cite{Fujimoto2018AddressingFA}. These methods are usually built upon the actor-critic framework \cite{prokhorov1997adaptive, konda2000actor}, accompanied with a replay buffer for storing past experiences. There are also many efforts in exploring off-policy training with image input \cite{Finn2015LearningVF, Dwibedi2018LearningAR, Sermanet2017TimeContrastiveNS, Nair2018VisualRL, Lee2019StochasticLA, yarats2021image, Hafner2020Dream, Hafner2018LearningLD, Yuan2022PreTrainedIE}.

\noindent \textbf{Sample-efficient continuous control algorithms.} How to improve the sample efficiency is one of the most critical issues to deploying the RL algorithms widely in the real world. Existing work realizes it via adding exploration noise \cite{lillicrap2015continuous, Plappert2017ParameterSN}, extra bonus reward \cite{Tang2016ExplorationAS, Fu2017EX2EW, Houthooft2016CuriositydrivenEI, Achiam2017SurpriseBasedIM}, multiple actors \cite{Zhang2018ACEAA, Efficient2022Lyu}, value estimate correction \cite{Pan2020SoftmaxDD, Wu2020ReducingEB, Kuznetsov2020ControllingOB, Kuznetsov2021AutomatingCO, Lyu2021ValueAF}, or by leveraging maximum entropy framework \cite{ziebart2010modeling, Haarnoja2018SoftAO, haarnoja2018softactorcritic}, incorporating uncertainty measurement \cite{Lee2020SUNRISEAS}, etc. \textit{SMR is orthogonal to all these advances} and can be easily combined with them.

Another line of research aiming at improving the sample efficiency in continuous control tasks sets their focus on learning a dynamics model of the environment \cite{Sutton1991dyna, Buckman2018SampleEfficientRL, Chua2018DeepRL, Janner2019WhenTT, DOro2020HowTL, li2022gradient, Hansen2022TemporalDL, voelcker2022value}. However, training an accurate model can be difficult \cite{Asadi2018TowardsAS, Asadi2018LipschitzCI, Lai2020BidirectionalMP} due to compounding errors \cite{Deisenroth2011PILCOAM, Venkatraman2015ImprovingMP, Talvitie2016SelfCorrectingMF}, and it is very time-consuming to run model-based RL codebase. 

\noindent\textbf{Data replay methods.} There are many ways of utilizing data in deep RL scenario, e.g., replaying good transitions \cite{Schaul2015PrioritizedER,Liu2021RegretME,Zhang2017ADL,Kapturowski2018RecurrentER}, balancing synthetic data and real data in model-based RL \cite{Hasselt2019WhenTU,Pan2020TrustTM}, etc. Some studies \cite{Fedus2020RevisitingFO,oro2022sampleefficient,li2023efficient} explore and uncover the advantages of update frequency on the collected transitions in a bootstrapping way for sample efficiency. SMR, instead, reuses the fixed batch data for multiple times to aid sample efficiency, and is orthogonal to previous methods.

\section{Conclusion}
In this paper, we propose sample multiple reuse (SMR), a novel method for enhancing the exploitation ability of off-policy continuous control RL algorithms by optimizing the agent on the fixed sampled batch multiple times. We show the convergence property of Q-learning with SMR in the tabular case. SMR can be incorporated with \textit{any} off-policy RL algorithms to boost their sample efficiency. We empirically show that the benefits of SMR are both algorithm-agnostic and task-agnostic. We further show that SMR is different from amplifying learning rate and discuss the potential overfitting phenomenon when using SMR. We hope this work can provide some insights to the community and aid the design of more advanced off-policy RL algorithms.

The main limitation of our work lies in the fact that injecting a sample reuse loop for training neural networks takes extra time. Such cost is negligible for state-based tasks but not for image-based tasks (check Appendix \ref{sec:missingexperiments}). A promising solution may be dropout \cite{Srivastava2014DropoutAS}, which has been previously adopted to reduce the computation cost of REDQ in \cite{Hiraoka2021DropoutQF}. We leave it as future work.


\small
\bibliographystyle{abbrv}
\bibliography{neurips_2023.bib}

\newpage
\appendix
\onecolumn

\section{Additional Theoretical Results}
\label{sec:additionaltheory}
In this section, we present additional theoretical results concerning on the sample complexity and finite time error bound for Q-SMR in nonreturnable MDPs. We first need to impose the following assumptions, which extends the assumption we made in the main text. All of the missing proofs can be found in Appendix \ref{sec:missingproof}.

\begin{assumption}[MDP Regularity]
\label{ass:mdpreg}
(1) $\forall t$, the reward signal is bounded, $|r_t|\le r_{\rm max}$; (2) The Markov chain induced by the stationary behavior policy $\pi_b$ is uniformly ergodic, and has a mixing time $t_{\rm mix}$,
\begin{equation*}
    t_{\rm mix} := \min\left\{ t\bigg| \max_{(s_0,a_0)\in\mathcal{S}\times\mathcal{A}}D_{\rm TV}\left( P^t(\cdot|s_0,a_0)\| \mu_{\pi_b} \right) \le \dfrac{1}{4} \right\}.
\end{equation*}
\end{assumption}

$P^t(\cdot|s_0,a_0)$ is the distribution of $(s_t,a_t)$ conditioned on the initial state-action pair $(s_0,a_0)$ and $D_{\rm TV}(p\|q)$ denotes the total variation distance between two distributions $p,q$. Denote $\mu_{\pi_b}$ as the stationary distribution of the aforementioned Markov chain. We define $\mu_{\rm min}:=\inf_{(s,a)\in\mathcal{S}\times\mathcal{A}}\mu_{\pi_b}$.

We are now interested in the sample complexity of the Q-SMR algorithm, which is built upon Corollary \ref{coro:simpleupdaterule}. Note that a general analysis on the sample complexity of Q-SMR is very hard as the target value keeps changing during SMR iteration. We thus resort to nonreturnable MDP and present the sample complexity results in the appendix. We introduce an important lemma on the learning rate sequence in Lemma \ref{lemma:learningrate}, which plays a critical role in proving Theorem \ref{theo:convergencerate}.

\begin{lemma}
\label{lemma:learningrate}
Denote $\hat{\alpha}_t = 1 - (1-\alpha_t)^M$, then we have
\begin{equation}
    \alpha_t \le \hat{\alpha}_t \le \min\{1, M\alpha_t\}.
\end{equation}
\end{lemma}
We then formally present the sample complexity of Q-SMR.

\begin{theorem}[Finite time error bound]
\label{theo:convergencerate}
Assume that Assumption \ref{ass:mdp} holds and the SMR ratio is set to be $M$. Suppose the learning rate is taken to be $\alpha_t = \frac{h}{M(t+t_0)}$ with $t_0\ge\max(4h, \lceil \log_2 \frac{2}{\mu_{\rm min}} \rceil t_{\rm mix})$ and $h\ge \frac{4}{\mu_{\rm min}(1-\gamma)}$, then with probability at least $1-\delta$,
\begin{equation}
    \begin{aligned}
    \|\hat{Q}_T - Q^*\|_\infty \le \tilde{\mathcal{O}}\left( \dfrac{r_{\rm max}\sqrt{t_{\rm mix}}}{(1-\gamma)^{2.5}\mu_{\rm min}}\dfrac{1}{\sqrt{T}} + \dfrac{ r_{\rm max}t_{\rm mix}}{(1-\gamma)^3\mu_{\rm min}^2}\dfrac{1}{T} \right).
    \end{aligned}
\end{equation}
\end{theorem}

As an immediate corollary, we have:
\begin{corollary}[Sample complexity]
\label{coro:samplecomplexity}
For any $0<\delta<1$ and $0<\epsilon<1$, with the Q-SMR algorithm we have:
\begin{equation}
    \forall (s,a)\in\mathcal{S}\times\mathcal{A}: \|\hat{Q}_T - Q^*\|_\infty \le \epsilon,
\end{equation}
holds with probability at least $1-\delta$, provided the iteration number $T$ obeys:
\begin{equation}
    T\stackrel{>}{\sim}\dfrac{r_{\rm max}^2 t_{\rm mix}}{(1-\gamma)^5\mu_{\rm min}^2}\dfrac{1}{\epsilon^2}.
\end{equation}
\end{corollary}
\noindent\textbf{Remark:} The above conclusion says that the sample complexity of Q-SMR gives $\tilde{\mathcal{O}}\left(\dfrac{ t_{\rm mix}}{(1-\gamma)^5\mu_{\rm min}^2}\dfrac{1}{\epsilon^2}\right)$. This result matches the recent theoretical analysis on the sample complexity of asynchronous Q-learning \cite{Qu2020FiniteTimeAO}, which improves over the previous bound \cite{Szepesvari1997TheAC, EvenDar2004LearningRF}. For a detailed comparison, we notice that the above sample complexity becomes $\tilde{\mathcal{O}}\left(\dfrac{ t_{\rm mix} (|\mathcal{S}||\mathcal{A}|)^2}{(1-\gamma)^5}\dfrac{1}{\epsilon^2}\right)$ by using that $\dfrac{1}{\mu_{\rm min}}$ scales with $(|\mathcal{S}||\mathcal{A}|)$. The prior bound in \cite{Szepesvari1997TheAC} gives a sample complexity of $\tilde{\mathcal{O}}\left(\dfrac{(|\mathcal{S}||\mathcal{A}|)^5}{(1-\gamma)^5\epsilon^{2.5}}\right)$ ($\omega=0.8$) and $\tilde{\mathcal{O}}\left(\dfrac{(|\mathcal{S}||\mathcal{A}|)^{3.3}}{(1-\gamma)^{5.2}\epsilon^{2.6}}\right)$ ($\omega=0.77$), where $\omega$ is the step size. Our results are sharper in terms of the dependence of $\dfrac{1}{1-\gamma}, \dfrac{1}{\epsilon},(|\mathcal{S}||\mathcal{A}|)$. The result can be extended to a constant learning rate (i.e., $\alpha_t \equiv \alpha,\forall\, t$) by following a similar analysis as \cite{Li2020SampleCO, Li2021IsQM}.

\clearpage

\section{Missing Proofs}
\label{sec:missingproof}

\subsection{Proof of Theorem \ref{theo:updaterule}}
\begin{theorem}
\label{apptheo:updaterule}
The update rule of Q-SMR is equivalent to:
\begin{equation}
    \begin{aligned}
    &Q_{t+1}(s_t,a_t) = (1-\alpha_t)^M Q_{t}(s_t,a_t) + \sum_{i=0}^{M-1} \alpha_t (1-\alpha_t)^i \mathcal{T}_{t+1}Q_{t+1}^{(M-1-i)}(s_t,a_t), \\
    &Q_{t+1}(s,a) = Q_t(s,a) \quad \forall\, (s,a)\neq (s_t,a_t),
    \end{aligned}
\end{equation}
where $\mathcal{T}_{t+1}Q_{t+1}(s_t,a_t) = r_t + \gamma\max_{a^\prime\in\mathcal{A}}Q_{t+1}(s_{t+1},a^\prime)$ denotes the empirical Bellman operator w.r.t. timestep $t+1$.
\end{theorem}

\begin{proof}
    Note that we omit the superscript $^{(M)}$ for both the right $Q_t(s_t,a_t)$ and the left $Q_{t+1}(s_t,a_t)$ for clarity. We do $M$ iterations in SMR with intermediate $Q$ value labeled as $Q_t^{(i)}$ at timestep $t$ and iteration $i, i\in\{1,2,\ldots,M\}$. Set the current $Q$-function at timestep $t$ as $Q_t^{(0)}$, then with the SMR iteration, we have $Q_{t}^{(M)}(s_t,a_t)$ which is set to be the new $Q$-function at timestep $t+1$, $Q_{t+1}^{(0)}(s_t,a_t)=Q_{t}^{(M)}(s_t,a_t)$. Note that in SMR iteration, the timestep is fixed, only the superscript changes with iteration, using the rule that $Q_{t+1}^{(i)}(s_t,a_t) = (1-\alpha_t)Q_{t+1}^{(i-1)}(s_t,a_t) + \alpha_t \mathcal{T}_{t+1}Q_{t+1}^{(i-1)}(s_t,a_t),i\in\{1,2,\ldots,M\}$. Then run the loop till convergence. We will use induction to show the above conclusion.

    If $M=1$, then the update rule becomes the vanilla Q-learning style (notice that $Q_t^{(0)}(\cdot,\cdot) = Q_t(\cdot,\cdot)$).

    Now for $\forall\, M\ge 1$, let us assume the update rule holds, if $(s,a) = (s_t,a_t)$, then,
    \begin{equation}
        Q_{t+1}^{(M)}(s_t,a_t) = (1-\alpha_t)^M Q^{(0)}_{t+1}(s_t,a_t) + \sum_{i=0}^{M-1} \alpha_t (1-\alpha_t)^i \mathcal{T}_{t+1}Q_{t+1}^{(M-1-i)}(s_t,a_t).
    \end{equation}
    Thus, 
    \begin{equation*}
    \begin{aligned}
        Q_{t+1}^{(M+1)}(s_t,a_t) &= (1-\alpha_t) Q^{(M)}_{t+1}(s_t,a_t) + \alpha_t \mathcal{T}_{t+1}Q_{t+1}^{(M)}(s_t,a_t). \quad \rm{(By}\,\rm{doing}\,\rm{one}\,\rm{iteration.)} \\
        &= (1-\alpha_t)\left[(1-\alpha_t)^M Q^{(0)}_{t+1}(s_t,a_t) + \sum_{i=0}^{M-1} \alpha_t (1-\alpha_t)^i \mathcal{T}_{t+1}Q_{t+1}^{(M-1-i)}(s_t,a_t)\right] + \alpha_t \mathcal{T}_{t+1}Q_{t+1}^{(M)}(s_t,a_t). \\
        &= (1-\alpha_t)^{M+1} Q^{(0)}_{t+1}(s_t,a_t) + \sum_{i=1}^{M} \alpha_t (1-\alpha_t)^i \mathcal{T}_{t+1}Q_{t+1}^{(M-i)}(s_t,a_t) + \alpha_t \mathcal{T}_{t+1}Q_{t+1}^{(M)}(s_t,a_t). \\
        &= (1-\alpha_t)^{M+1} Q^{(0)}_{t+1}(s_t,a_t) + \sum_{i=0}^{M} \alpha_t (1-\alpha_t)^i \mathcal{T}_{t+1}Q_{t+1}^{(M-i)}(s_t,a_t) \\
        &= (1-\alpha_t)^{M+1} Q^{(M+1)}_{t}(s_t,a_t) + \sum_{i=0}^{M} \alpha_t (1-\alpha_t)^i \mathcal{T}_{t+1}Q_{t+1}^{(M-i)}(s_t,a_t) \\
    \end{aligned}
    \end{equation*}
    Then by induction and omitting the superscript $^{(M+1)}$, we deduce that the update rule holds for $\forall\, M\ge 1$.
\end{proof}
\noindent\textbf{Remark:} It is quite hard to trace back the intermediate Bellman backup $\mathcal{T}_{t+1}Q_{t+1}^{(i)}(s_t,a_t)$ since it is taken over $r_t + \gamma\max_{a^\prime\in\mathcal{A}}Q_{t+1}^{(i)}(s_{t+1},a^\prime)$. Though $s_{t+1}$ is known, the maximal $Q$ value may change position with the iteration.

\subsection{Proof of Theorem \ref{theo:stability}}

\begin{theorem}[Stability]
Let Assumption \ref{ass:mdpreg} holds and assume that the initial $Q$-function is set to be 0, then for any iteration $t$, the value estimate induced by the Q-SMR, $\hat{Q}_t$, is bounded, i.e., $|\hat{Q}_t|\le \dfrac{r_{\rm max}}{1-\gamma},\forall t$.
\end{theorem}
\begin{proof}
    We also show this by induction. Obviously, $|\hat{Q}_0|=0\le\dfrac{r_{\rm max}}{1-\gamma}$. Now let us suppose for $\forall \, t\ge0, |\hat{Q}_t|\le \dfrac{r_{\rm max}}{1-\gamma}$, then by using the update rule from Theorem \ref{theo:updaterule}, we have
    \begin{align*}
        |\hat{Q}_{t+1}(s_t,a_t)| &= \left|(1-\alpha_t)^M \hat{Q}_{t}(s_t,a_t) + \sum_{i=0}^{M-1} \alpha_t (1-\alpha_t)^i \mathcal{T}_{t+1}\hat{Q}_{t+1}^{(M-1-i)}(s_t,a_t)\right| \\
        &= \left|(1-\alpha_t)^M \hat{Q}_{t}(s_t,a_t) + \sum_{i=0}^{M-1} \alpha_t (1-\alpha_t)^i \left[r_t + \gamma\max_{a^\prime\in\mathcal{A}}\hat{Q}_{t+1}^{(M-1-i)}(s_{t+1},a^\prime)\right]\right| \\
        &\le (1-\alpha_t)^M \left| \hat{Q}_{t}(s_t,a_t) \right| + \sum_{i=0}^{M-1} \alpha_t (1-\alpha_t)^i \left|r_t + \gamma\max_{a^\prime\in\mathcal{A}}\hat{Q}_{t+1}^{(M-1-i)}(s_{t+1},a^\prime)\right| \\
        &\le (1-\alpha_t)^M \dfrac{r_{\rm max}}{1-\gamma} + \sum_{i=0}^{M-1} \alpha_t (1-\alpha_t)^i \left[|r_t| + \gamma\left|\max_{a^\prime\in\mathcal{A}}\hat{Q}_{t+1}^{(M-1-i)}(s_{t+1},a^\prime)\right|\right] \\
        &\le (1-\alpha_t)^M \dfrac{r_{\rm max}}{1-\gamma} + \sum_{i=0}^{M-1} \alpha_t (1-\alpha_t)^i \left[r_{\rm max} + \gamma\dfrac{r_{\rm max}}{1-\gamma}\right] \\
        &= (1-\alpha_t)^M \dfrac{r_{\rm max}}{1-\gamma} + \left[\sum_{i=0}^{M-1} \alpha_t (1-\alpha_t)^i\right] \dfrac{r_{\rm max}}{1-\gamma} \\
        &= (1-\alpha_t)^M \dfrac{r_{\rm max}}{1-\gamma} + \left[1-(1-\alpha_t)^M\right] \dfrac{r_{\rm max}}{1-\gamma} \\
        &= \dfrac{r_{\rm max}}{1-\gamma}.
    \end{align*}
    By using induction, we deduce that $\forall\, t\ge0$ the Q-SMR outputs stable $Q$ value, which satisfies $|\hat{Q}_t|\le\dfrac{r_{\rm max}}{1-\gamma}$.
\end{proof}

\subsection{Proof of Corollary \ref{coro:simpleupdaterule}}
\begin{corollary}
\label{appcoro:simpleupdaterule}
If the MDP is nonreturnable, i.e., $s_{t+1}\neq s_t$, the update rule of Q-SMR gives:
\begin{equation}
    \begin{aligned}
    &Q_{t+1}(s_t,a_t) = (1-\alpha_t)^M Q_{t}(s_t,a_t) + \left[ 1-(1-\alpha_t)^M \right] \mathcal{T}_{t+1}Q_t(s_t,a_t), \\
    &Q_{t+1}(s,a) = Q_t(s,a) \quad \forall\, (s,a)\neq (s_t,a_t),
    \end{aligned}
\end{equation}
\end{corollary}
\begin{proof}
    If the MDP is nonreturnable, then it is easy to address the empirical Bellman backup $\mathcal{T}_{t+1}Q_{t+1}^{(i)}$. We have that $\mathcal{T}_{t+1}Q_{t+1}^{(i)}(s_t,a_t) = r_t + \gamma \max_{a^\prime\in\mathcal{A}}Q_{t+1}^{(i)}(s_{t+1},a^\prime)$. Since it is asynchronous Q-learning, only entry $(s_t,a_t)$ will be updated inside the SMR loop. That is to say, $\mathcal{T}_{t+1}Q_{t+1}^{(i)}(s_t,a_t) = r_t + \gamma \max_{a^\prime\in\mathcal{A}}Q_{t+1}^{(i)}(s_{t+1},a^\prime)$ is unchanged throughout the SMR iteration. Therefore, based on Theorem \ref{apptheo:updaterule}, we have that the update rule gives
    \begin{align*}
        Q_{t+1}(s_t,a_t) &= (1-\alpha_t)^M Q_{t}(s_t,a_t) + \sum_{i=0}^{M-1} \alpha_t (1-\alpha_t)^i \mathcal{T}_{t+1}Q_{t+1}^{(M-1-i)}(s_t,a_t), \\
        &= (1-\alpha_t)^M Q_{t}(s_t,a_t) + \sum_{i=0}^{M-1} \alpha_t (1-\alpha_t)^i \mathcal{T}_{t+1}Q_t(s_t,a_t), \\
        &= (1-\alpha_t)^M Q_{t}(s_t,a_t) + \left[ 1-(1-\alpha_t)^M \right]\mathcal{T}_{t+1}Q_t(s_t,a_t).
    \end{align*}
\end{proof}
\noindent\textbf{Remark:} If we also let $s_{t+1}=s_t$ follow the above update rule, then the analysis below (e.g., sample complexity) can be extended naturally. This, however, triggers a gap between the original SMR loop and this practical update rule. We thus enforce $s_{t+1}\neq s_t$. Our analysis is restricted to nonreturnable MDPs, while our empirical results remedy this and validate the effectiveness of our proposed method.

\subsection{Proof of Theorem \ref{theo:convergence}}
In order to show Theorem \ref{theo:convergence}, we first present a well-known result from \cite{Singh2000ConvergenceRF}, which is built upon a proposition from \cite{bertsekas2012dynamic}.
\begin{lemma}
    \label{applemma:randomprocess}
    Consider a stochastic process $(\zeta_t, \Delta_t, F_t),t\ge0$ where $\zeta_t, \Delta_t, F_t:X\mapsto \mathbb{R}$ satisfy the equation:
    \begin{equation}
        \Delta_{t+1}(x_t) = (1 - \zeta_t(x_t))\Delta_t(x_t) + \zeta_t(x_t)F_t(x_t),
    \end{equation}
    where $x_t\in X$ and $t=0,1,2,\ldots$. Let $P_t$ be a sequence of increasing $\sigma$-fields such that $\zeta_0$ and $\Delta_0$ are $P_0$-measurable and $\zeta_t, \Delta_t$ and $F_{t-1}$ are $P_t$-measurable, $t=1,2,\ldots$. Assume the following conditions hold: (1) The set $X$ is finite; (2) $\zeta_t(x_t)\in[0,1]$, $\sum_t\zeta_t(x_t) = \infty, \sum_t(\zeta_t(x_t))^2 < \infty$ with probability 1 and $\forall x\neq x_t:\zeta_t(x_t)=0$; (3) $\| \mathbb{E}[F_t|P_t] \| \le \kappa \|\Delta_t\| + c_t$, where $\|\cdot\|$ denotes maximum norm, $\kappa\in[0,1)$ and $c_t$ converges to 0 with probability 1; (4) Var$[F_t(x_t)|P_t]\le C(1 + \|\Delta_t\|)^2$, where $C$ is some constant. Then $\Delta_t$ converges to 0 with probability 1.
\end{lemma}

We also need the following lemma, which will be of great help.
\begin{lemma}
    \label{applemma:learningratelemma}
    If the learning rates satisfy $\alpha_t(s,a)\in[0,1], \sum_t\alpha_t(s,a) = \infty, \sum_t(\alpha_t(s,a))^2<\infty$ with probability 1, then the following holds with probability 1:
    \begin{equation}
        \sum_{i=0}^{M-1} \alpha_t (1-\alpha_t)^i \in[0,1], \quad \sum_t \sum_{i=0}^{M-1} \alpha_t (1-\alpha_t)^i = \infty, \quad \sum_t\left( \sum_{i=0}^{M-1} \alpha_t (1-\alpha_t)^i \right)^2 < \infty.
    \end{equation}
\end{lemma}
\begin{proof}
    It is easy to find that $\sum_{i=0}^{M-1} \alpha_t (1-\alpha_t)^i = 1 - (1-\alpha_t)^M$. Since $\alpha_t\in[0,1]$, we have $1-\alpha_t\in[0,1], (1-\alpha_t)^M\in[0,1]$ and therefore $\sum_{i=0}^{M-1} \alpha_t (1-\alpha_t)^i = 1 - (1-\alpha_t)^M\in[0,1]$.
    
    Meanwhile, $1 - \alpha_t \ge (1-\alpha_t)^M$, then $\sum_t \sum_{i=0}^{M-1} \alpha_t (1-\alpha_t)^i \ge \sum_t \alpha_t = \infty$, thus $\sum_t \sum_{i=0}^{M-1} \alpha_t (1-\alpha_t)^i = \infty$.
    
    Finally, $\sum_t\left( \sum_{i=0}^{M-1} \alpha_t (1-\alpha_t)^i \right)^2 \le \sum_t\left( \sum_{i=0}^{M-1} \alpha_t \right)^2 = M^2 \sum_t\left(\alpha_t\right)^2 < \infty$.
\end{proof}

Then we formally give the convergence property of Q-SMR below.

\begin{theorem}[Formal, Convergence of Q-SMR]
\label{apptheo:convergence}
Given the following conditions: (1) each state-action pair is sampled an infinite number of times; (2) the MDP is finite; (3) $\gamma\in[0,1)$; (4) $Q$ values are stored in a look-up table; (5) the learning rates satisfy $\alpha_t(s,a)\in[0,1], \sum_t\alpha_t(s,a) = \infty, \sum_t(\alpha_t(s,a))^2<\infty$ with probability 1 and $\alpha_t(s,a)=0,\forall (s,a)\neq (s_t,a_t)$; (6) Var$[r(s,a)]<\infty, \forall s,a$, then the Q-SMR algorithm converges to the optimal $Q$-function.
\end{theorem}
\begin{proof}
    To show the convergence of Q-SMR, we first show the convergence of the following update rule, which is exactly the rule of the simplified Q-SMR algorithm presented in the Corollary \ref{coro:simpleupdaterule}.
    \begin{equation}
    \label{eq:simplifiedrule}
        \begin{aligned}
        Q_{t+1}(s_t,a_t) = (1-\alpha_t)^M Q_{t}(s_t,a_t) + \left[ 1-(1-\alpha_t)^M \right] \mathcal{T}_{t+1}Q_t(s_t,a_t).
        \end{aligned}
    \end{equation}
    Subtracting from both sides the quantity $Q^*(s_t,a_t)$, and letting $\Delta_t(s_t,a_t) = Q_t(s_t,a_t) - Q^*(s_t,a_t)$ yields:
    \begin{equation*}
        \Delta_{t+1}(s_t,a_t) = (1-\alpha_t)^M \Delta_t(s_t,a_t) + \left[ 1-(1-\alpha_t)^M \right] \left(r_t + \gamma\max_{a^\prime\in\mathcal{A}}Q_t(s_{t+1},a^\prime) - Q^*(s_t,a_t)\right).
    \end{equation*}
    Denote $\beta_t = 1 - (1-\alpha_t)^M$, and write $F_t(s_t,a_t) = r_t + \gamma\max_{a^\prime\in\mathcal{A}}Q_t(s_{t+1},a^\prime) - Q^*(s_t,a_t)$, we have
    \begin{equation*}
        \Delta_{t+1}(s_t,a_t) = (1-\beta_t) \Delta_t(s_t,a_t) + \beta_t F_t.
    \end{equation*}
    From Lemma \ref{applemma:learningratelemma}, we conclude that the new learning rate sequence obeys $\beta_t\in[0,1], \sum_t\beta_t = \infty$ and $\sum_t(\beta_t)^2<\infty$. Meanwhile, $\mathbb{E}[F_t(s_t,a_t)|P_t] = \mathcal{T}Q_t(s_t,a_t) - Q^*(s_t,a_t)$. Since the optimal $Q$-function is a fixed point of the Bellman operator, we have $\mathbb{E}[F_t(s_t,a_t)|P_t] = \mathcal{T}Q_t(s_t,a_t) - \mathcal{T}Q^*(s_t,a_t)$. Since the Bellman operator is a contraction, we have $\|\mathbb{E}[F_t(s_t,a_t)|P_t]\| = \|\mathcal{T}Q_t(s_t,a_t) - \mathcal{T}Q^*(s_t,a_t)\|\le \gamma\|\Delta_t\|$.

    Finally, we check the variance of $F_t(s_t,a_t)$, it is easy to find:
    \begin{align*}
        {\rm Var}[F_t(s_t,a_t)|P_t] &= \mathbb{E}\left[ \left( r_t+\gamma\max_{a^\prime\in\mathcal{A}}Q_t(s_{t+1},a^\prime) - Q^*(s_t,a_t) - \left( \mathcal{T}Q_t(s_t,a_t) - Q^*(s_t,a_t) \right) \right)^2 \bigg|P_t \right] \\
        &= \mathbb{E}\left[ \left( r_t+\gamma\max_{a^\prime\in\mathcal{A}}Q_t(s_{t+1},a^\prime) - \mathcal{T}Q_t(s_t,a_t) \right)^2 \bigg|P_t \right] \\
        &= {\rm Var}\left[ r_t+\gamma\max_{a^\prime\in\mathcal{A}}Q_t(s_{t+1},a^\prime)\bigg|P_t \right].
    \end{align*}
    Due to the fact that $r_t$ is bounded, it clearly verifies that ${\rm Var}[F_t(s_t,a_t)|P_t]\le C(1 + \|\Delta_t\|)^2$ for some constant $C$. Combining these together, and by using Lemma \ref{applemma:randomprocess}, we conclude that $\Delta_t$ converges to 0 with probability 1. That is to say, the simplified Q-SMR algorithm with update rule in Equation \ref{eq:simplifiedrule} converges to the optimal $Q$-function. Then, for the formal Q-SMR update rule, we have
    \begin{align}
        Q_{t+1}(s_t,a_t) &= (1-\alpha_t)^M Q_{t}(s_t,a_t) + \sum_{i=0}^{M-1} \alpha_t (1-\alpha_t)^i \mathcal{T}_{t+1}Q_{t+1}^{(M-1-i)}(s_t,a_t) \\
        &\le (1-\alpha_t)^M Q_{t}(s_t,a_t) + \left[ 1-(1-\alpha_t)^M \right] \max_{i\in[0,M-1]}\mathcal{T}_{t+1}Q_t^{(i)}(s_t,a_t).
    \end{align}
    It is easy to check that the right side converges to the optimal $Q$-function by following the same analysis above. Furthermore, we have
    \begin{align}
        Q_{t+1}(s_t,a_t) &= (1-\alpha_t)^M Q_{t}(s_t,a_t) + \sum_{i=0}^{M-1} \alpha_t (1-\alpha_t)^i \mathcal{T}_{t+1}Q_{t+1}^{(M-1-i)}(s_t,a_t) \\
        &\ge (1-\alpha_t)^M Q_{t}(s_t,a_t) + \left[ 1-(1-\alpha_t)^M \right] \min_{i\in[0,M-1]}\mathcal{T}_{t+1}Q_t^{(i)}(s_t,a_t).
    \end{align}
    Similarly, the lower bound side converges to the optimal $Q$-function. Then by combing the results above, we naturally conclude that Q-SMR converges to the optimal $Q$-function.
\end{proof}

\subsection{Proof of Theorem \ref{theo:difference}}

\begin{proof}
If we amplify $\alpha_t$, then we have
\begin{equation}
    \label{eq:largerlearningrate}
    \theta_{t+1} = \theta_t - {\color{red} M\alpha_t \nabla \mathcal{L}_{\theta_t}}.
\end{equation}
This is the parameter update rule for the case of enlarging learning rate $M$ times.

Now we investigate SMR with SGD. Denote $\theta_t^{(i)}$ as the intermediate parameter in the SMR loop at timestep $t$ and iteration $i$, then it is easy to find $\theta_{t+1}^{(1)} = \theta_{t+1}^{(0)} - \alpha_t \nabla\mathcal{L}_{\theta_{t+1}^{(0)}}$, and $\theta_{t+1}^{(2)} = \theta_{t+1}^{(1)} - \alpha_t \nabla\mathcal{L}_{\theta_{t+1}^{(1)}} = \theta_{t+1}^{(0)} - \alpha_t \nabla\mathcal{L}_{\theta_{t+1}^{(0)}} - \alpha_t \nabla\mathcal{L}_{\theta_{t+1}^{(1)}}$. Finally, by doing iteration till $M$, using $\theta_{t+1}^{(0)}=\theta_t^{(M)}$ and omitting the superscript $(M)$, we have
\begin{equation}
    \theta_{t+1} = \theta_t - {\color{red} \alpha_t \sum_{i=0}^{M-1} \nabla\mathcal{L}_{\theta_{t+1}^{(i)}}}.
\end{equation}
\end{proof}

\subsection{Proof of Lemma \ref{lemma:learningrate}}
\begin{lemma}
\label{applemma:learningrate}
Denote $\hat{\alpha}_t = 1 - (1-\alpha_t)^M$, then we have
\begin{equation}
    \alpha_t \le \hat{\alpha}_t \le \min\{1, M\alpha_t\}.
\end{equation}
\end{lemma}
\begin{proof}
    We write $f(x) = 1 - (1-x)^M - x$ and $g(x) = 1 - (1-x)^M - Mx, x\in[0,1], M\ge 1, M\in\mathbb{Z}$, then for $f(x)$, we have $f(x) = 1 - (1-x)^M - x = (1-x) - (1-x)^M = (1-x)\left[ 1 - (1-x)^{M-1} \right]\ge 0$. Therefore, $1 - (1-x)^M\ge x$.
    
    For $g(x)$, we have 
    \begin{align*}
        g^\prime(x) = M(1-x)^{M-1} - M = M\left[ (1-x)^{M-1} - 1 \right] \le 0.
    \end{align*}
    It indicates that $g(x)$ decreases in the region $[0,1]$, thus $g(x) \le g(0) = 0$ which incurs $1 - (1-x)^M\le Mx,\forall\,x\in[0.1]$. Meanwhile, as $1 - (1-x)^M \le 1$, we thus have $1 - (1-x)^M \le \min\{1, Mx\}$.

    By setting $x = \alpha_t$, we have the desired conclusions immediately.
\end{proof}

\subsection{Proof of Theorem \ref{theo:convergencerate}}
\begin{theorem}[Finite time error bound]
\label{apptheo:convergencerate}
Assume that Assumption \ref{ass:mdp} holds and the SMR ratio is set to be $M$. Suppose the learning rate is taken to be $\alpha_t = \frac{h}{M(t+t_0)}$ with $t_0\ge\max(4h, \lceil \log_2 \frac{2}{\mu_{\rm min}} \rceil t_{\rm mix})$ and $h\ge \frac{4}{\mu_{\rm min}(1-\gamma)}$, then with probability at least $1-\delta$,
\begin{equation}
    \begin{aligned}
    \|\hat{Q}_T - Q^*\|_\infty \le \tilde{\mathcal{O}}\left( \dfrac{r_{\rm max}\sqrt{t_{\rm mix}}}{(1-\gamma)^{2.5}\mu_{\rm min}}\dfrac{1}{\sqrt{T}} + \dfrac{r_{\rm max}t_{\rm mix}}{(1-\gamma)^3\mu_{\rm min}^2}\dfrac{1}{T} \right)
    \end{aligned}
\end{equation}
\end{theorem}
The proof of this theorem is borrowed heavily from \cite{Qu2020FiniteTimeAO}. Throughout the proof, we denote $\|\cdot\|$ as the infinity norm. We also assume there exist some constant $C>0$ s.t. $\|F(x)\|\le \gamma\|x\| + C,\forall\,x\in\mathbb{R}^n$, where $F(\cdot)$ denotes the bellman operator. This assumption can be generally satisfied with $C = (1+\gamma)\|x^*\|$ since $\|F(x)\|\le \|F(x) - F(x^*)\| + \|F(x^*)\|\le \gamma\|x - x^*\| + \|x^*\|\le\gamma \|x\| + (1+\gamma)\|x^*\|$. 
\begin{proof}
    The proof is generally divided into three steps. First, we decompose the error in a recursive form. Second, we bound the contribution of the noise sequence to the error decomposition. Third, we use the error decomposition and the bounds to prove the result. We let $\hat{\alpha}_t = 1 - (1-\alpha_t)^M$ and rewrite the update rule for Q-SMR below.
    \begin{align*}
        &x_i(t+1) = x_i(t) + \hat{\alpha}_t (F_i(x(t)) -x_i(t) + \omega(t)), i=i_t, \\
        &x_i(t+1) = x_i(t), i\neq i_t,
    \end{align*}
    where we write $Q_t(s_t,a_t)$ as $x_i(t)$, $i_t\in\{1,2,\ldots,n\}$ is a stochastic process adapted to a filtration $P_t$, $F_i(x(t)) = r(s_t,a_t) + \gamma \mathbb{E}_{s^\prime\sim P(\cdot|s_t,a_t)}\max_{a^\prime\in\mathcal{A}}Q_t(s^\prime,a^\prime)$, $\omega(t) = r_t + \gamma\max_{a^\prime\in\mathcal{A}}Q_{t}(s_{t+1},a^\prime) - r(s_t,a_t) - \gamma \mathbb{E}_{s^\prime\sim P(\cdot|s_t,a_t)}\max_{a^\prime\in\mathcal{A}}Q_t(s^\prime,a^\prime)$.

    Following the same way in \cite{Qu2020FiniteTimeAO} (Equation 7), we rewrite the update formula as follows:
    \begin{align*}
        x(t+1) = (I-\hat{\alpha}_tD_t)x(t) + \hat{\alpha}_t D_t F(x(t)) + \hat{\alpha}_t(\epsilon(t) + \phi(t)),
    \end{align*}
    where $\epsilon(t) = \left[ (e_{i_t}e_{i_t}^T - D_t)(F(x(t-\tau)) - x(t-\tau)) + \omega(t)e_{i_t} \right]$, and $e_i$ is the unit vector with its $i$-th entry 1 and others 0. Clearly, $x(t)$ is $P_t$ measurable, and as $\epsilon(t)$ depends on $\omega(t)$ which is $P_{t+1}$ measurable, $\epsilon(t)$ is $P_{t+1}$ measurable. Moreover, we have
    \begin{equation}
        \label{appeq:epsilon}
        \mathbb{E}\epsilon(t)|P_{t-\tau} = \mathbb{E}[(e_{i_t}e_{i_t}^T-D_t)|P_{t-\tau}][F(x(t-\tau)) - x(t-\tau)]+\mathbb{E}[\mathbb{E}[\omega(t)|P_t]e_{i_t}|P_{t-\tau}] = 0,
    \end{equation}
    where $D_t = \mathbb{E}e_{i_t}e_{i_t}^T|P_{t-\tau}$, $\tau$ is a positive integer. Assume there exist $\tau$ and a $\sigma^\prime\in(0,1)$ such that for any $i\in\mathcal{N}, \mathcal{N}=\{1,2,\ldots,n\}$ and $t\ge\tau$, $P(i_t=i|P_{t-\tau})\ge \sigma^\prime= M\sigma$, i.e., exploration is sufficient. Such requirement can be satisfied if we take $\sigma = \frac{1}{2}\mu_{\rm min}$ and $\tau = \lceil \log_2 (\frac{2}{\mu_{\rm min}}) \rceil t_{\rm mix}$ where $\lceil \cdot \rceil$ denotes taking ceiling of the integer, e.g., $\lceil 2.7 \rceil = 3, \lceil 5.1\rceil = 6$. $\phi(t) = \left[ (e_{i_t}e_{i_t}^T - D_t)(F(x(t)) - F(x(t-\tau)) - (x(t) - x(t-\tau))) \right]$.

    We expand it recursively and have:
    \begin{align*}
        x(t+1) = \tilde{B}_{\tau-1,t}x(\tau) + \sum_{k=\tau}^t B_{k,t}F(x(k)) + \sum_{k=\tau}^t \hat{\alpha}_t \tilde{B}_{k,t}(\epsilon(k) + \phi(k)),
    \end{align*}
    where $B_{k,t} = \hat{\alpha}_k D_k\prod_{l=k+1}^t(I-\hat{\alpha}_l D_l), \tilde{B}_{k,t} = \prod_{l=k+1}^t (I-\hat{\alpha}_l D_l)$. It is easy to notice that $B_{k,t}$ and $\tilde{B}_{k,t}$ are $n$-by-$n$ diagonal random metrics, with their $i$-th diagonal entry given by $b_{k,t,i} = \hat{\alpha}_t d_{k,i}\prod_{l=k+1}^t (1-\hat{\alpha}_l d_{l,i})$ and $\tilde{b}_{k,t,i} = \prod_{l=k+1}^t (1 - \hat{\alpha}_l d_{l,i})$. For any $i$, the following holds almost surely,
    \begin{align*}
        b_{k,t,i}\le \beta_{k,t} := \hat{\alpha}_k \prod_{l=k+1}^t (1-\hat{\alpha}_l M\sigma), \quad \tilde{b}_{k,t,i}\le\tilde{\beta}_{k,t} := \prod_{l=k+1}^t (1-\hat{\alpha}_l M\sigma).
    \end{align*}
    Based on Lemma 8 in \cite{Qu2020FiniteTimeAO}, denote $a_t = \|x(t) - x^*\|$, then we have almost surely,
    \begin{align}
        \label{appeq:recursiveform}
        a_{t+1} \le \tilde{\beta}_{\tau-1,t}a_\tau + \gamma \sup_{i\in\mathcal{N}}\sum_{k=\tau}^t b_{k,t,i}a_k + \left\| \sum_{k=\tau}^t\hat{\alpha}_k \tilde{B}_{k,t}\epsilon(k) \right\| + \left\| \sum_{k=\tau}^t\hat{\alpha}_k \tilde{B}_{k,t}\phi(k) \right\|.
    \end{align}
    \begin{lemma}
    \label{applemma:boundonepsilonphi}
        The following bounds hold almost surely: (a) $\|\epsilon(t)\|\le \bar{\epsilon}:= \dfrac{4r_{\rm max}}{1-\gamma} + C$; (b) $\|\phi(t)\|\le \sum_{k = t-\tau+1}^t 2\bar{\epsilon}\hat{\alpha}_{k-1}$.
    \end{lemma}
    \begin{proof}
        Replacing $\bar{x}$ with $\dfrac{r_{\rm max}}{1-\gamma}$, $\bar{\omega}$ with $\dfrac{2r_{\rm max}}{1-\gamma}$ and using $\underline{v}$ as 1 (since we use infinity norm), and replacing $\alpha_{k-1}$ with $\hat{\alpha}_{k-1}$ in Lemma 9 of \cite{Qu2020FiniteTimeAO} will induce the conclusion immediately.
    \end{proof}
    These are still not enough to bound $\|\sum_{k=\tau}^t\hat{\alpha}_k \tilde{B}_{k,t}\epsilon(k)\|$ and $\|\sum_{k=\tau}^t \hat{\alpha}_k\tilde{B}_{k,t}\phi(k)\|$. We provide in the following lemma some useful results of $\beta_{k,t}$ and $\tilde{\beta}_{k,t}$.
    \begin{lemma}
        \label{applemma:beta}
        If $\alpha_t = \frac{h}{M(t+t_0)}$, where $h>\frac{2}{\sigma}$ and $t_0\ge \max(4h,\tau)$, then $\beta_{k,t}$ and $\tilde{\beta}_{k,t}$ satisfy the following relationships:

        (a) $\beta_{k,t}\le \dfrac{h}{k+t_0}\left( \dfrac{k+1+t_0}{t+1+t_0} \right)^{\sigma h}$, $\tilde{\beta}_{k,t}\le \left(\dfrac{k+1+t_0}{t+1+t_0}\right)^{\sigma h}$; (b) $\sum_{k=1}^t \beta_{k,t}^2 \le \dfrac{2h}{\sigma}\dfrac{1}{t+1+t_0}$; 
        
        (c) $\sum_{k=\tau}^t \beta_{k,t}\sum_{l=k-\tau+1}^k\hat{\alpha}_{l-1}< \dfrac{8h\tau}{\sigma}\dfrac{1}{t+1+t_0}$.
    \end{lemma}
    \begin{proof}
        For part (a), notice that $\log (1-x) \le -x, \forall \, x <1$, then
        \begin{equation*}
            (1-M\sigma \hat{\alpha}_t) \le (1-M\sigma \alpha_t) = e^{\log (1-\frac{\sigma h}{t+t_0})} \le e^{-\frac{\sigma h}{t+t_0}},
        \end{equation*}
        where we use $\hat{\alpha}_t \ge \alpha_t$ according to Lemma \ref{applemma:learningrate}. Therefore, we have
        \begin{align*}
            \prod_{l=k+1}^t (1-M\sigma \hat{\alpha}_l) &\le e^{-\sum_{l=k+1}^t \frac{\sigma h}{l+t_0} } \le e^{-\int_{k+1}^{t+1}\frac{\sigma h}{y+t_0}dy} = e^{-\sigma h \log(\frac{t+1+t_0}{k+1+t_0})} = \left( \dfrac{k+1+t_0}{t+1+t_0} \right)^{\sigma h}.
        \end{align*}
        This directly leads to the bound on $\tilde{\beta}_{k,t}$. We have $\beta_{k,t} = \hat{\alpha}_k \prod_{l=k+1}^t(1-\hat{\alpha}_l \sigma) \le M \dfrac{h}{M(k+t_0)} \left(\dfrac{k+1+t_0}{t+1+t_0}\right)^{\sigma h} = \dfrac{h}{k+t_0}\left( \dfrac{k+1+t_0}{t+1+t_0} \right)^{\sigma h}$, where we use the fact that $\hat{\alpha}_k \le M\alpha_k$ based on Lemma \ref{applemma:learningrate}.

        For part (b), we have 
        \begin{align*}
            \beta^2_{k,t} \le \dfrac{h^2}{(k+t_0)^2}\left( \dfrac{k+1+t_0}{t+1+t_0} \right)^{2\sigma h} = \dfrac{h^2}{(t+1+t_0)^{2\sigma h}} \dfrac{(k+1+t_0)^{2\sigma h}}{(k+t_0)^2} \le \dfrac{2h^2}{(t+1+t_0)^{2\sigma h}}(k+t_0)^{2\sigma h-2},
        \end{align*}
        where we have used $(k+1+t_0)^{2\sigma h}\le 2(k+t_0)^{2\sigma h}$, which is true when $t_0\ge 4h$. Then, we have
        \begin{align*}
            \sum_{k=1}^t \beta^2_{k,t} &\le \dfrac{2h^2}{(t+1+t_0)^{2\sigma h}}\sum_{k=1}^t (k+t_0)^{2\sigma h -2} \le \dfrac{2h^2}{(t+1+t_0)^{2\sigma h}} \int_{1}^{t+1} (y+t_0)^{2\sigma h -2}dy \\
            & < \dfrac{2 h^2}{(t+1+t_0)^{2\sigma h}}\dfrac{1}{2\sigma h-1}(t+1+t_0)^{2\sigma h-1} < \dfrac{2h}{\sigma} \dfrac{1}{t+1+t_0},
        \end{align*}
        where we have used the fact that $2\sigma h-1> \sigma h$ (as $h> \frac{2}{\sigma}$).

        For part (c), notice that for $k-\tau\le l\le k-1$ where $k\ge \tau$, we have $\alpha_l \le \dfrac{h}{M(k-\tau+t_0)}$. Since $k\ge \tau$ and $t_0 > \max(4h,\tau)$ (the assumption), then we have $k+t_0 > 2\tau$ which indicates that $kh - 2h\tau + ht_0 >0$, and thus $kh + ht_0 < 2kh - 2h\tau + 2ht_0$, which is to say $\dfrac{h}{k-\tau+t_0}< \dfrac{2h}{k+t_0}$. Therefore, we have $\alpha_l < \dfrac{2h}{M(k+t_0)}$. By using Lemma \ref{applemma:learningrate}, we have $\hat{\alpha}_l \le M\alpha_l < \dfrac{2h}{k+t_0}$.

        Then, we have
        \begin{align*}
            \sum_{k=\tau}^t \beta_{k,t}\sum_{l=k-\tau+1}^k\hat{\alpha}_{l-1} &< \sum_{k=\tau}^t\beta_{k,t} \dfrac{2h\tau}{k+t_0} \le \sum_{k=\tau}^t\dfrac{h}{k+t_0}\left( \dfrac{k+1+t_0}{t+1+t_0} \right)^{\sigma h}\dfrac{2h\tau}{k+t_0} \le \sum_{k=\tau}^t\dfrac{4h^2\tau}{(t+1+t_0)^{\sigma h}}(k+t_0)^{\sigma h-2} \\
            &\le \dfrac{4h^2\tau}{(t+1+t_0)^{\sigma h}}\int_{\tau}^{t+1}(y+t_0)^{\sigma h-2}dy \le \dfrac{4h^2\tau}{(t+1+t_0)^{\sigma h}}\dfrac{(t+1+t_0)^{\sigma h-1}}{\sigma h - 1} \\
            &\le \dfrac{8h\tau}{\sigma} \dfrac{1}{t+1+t_0},
        \end{align*}
        where we have used $(k+1+t_0)^{\sigma h}\le 2(k+t_0)^{\sigma h}$ and $\sigma h - 1 > \frac{1}{2}\sigma h$.
    \end{proof}
    Now we are ready to bound $\|\sum_{k=\tau}^t \hat{\alpha}_k\tilde{B}_{k,t}\phi(k)\|$. It is easy to find that
    \begin{align*}
        \left\|\sum_{k=\tau}^t \hat{\alpha}_k\tilde{B}_{k,t}\phi(k)\right\| \le \sum_{k=\tau}^t \hat{\alpha}_k \|\tilde{B}_{k,t}\|\|\phi(k)\|\le \sum_{k=\tau}^t \beta_{k,t}\sum_{l=k-\tau+1}^k 2\bar{\epsilon} \hat{\alpha}_{l-1} < \dfrac{16\bar{\epsilon}h\tau}{\sigma(t+1+t_0)}:= C_\phi \dfrac{1}{t+1+t_0},
    \end{align*}
    where we have used the fact that each entry of $\tilde{B}_{k,t}$ is upper bounded by $\tilde{\beta}_{k,t}$, i.e., $\|\tilde{B}_{k,t}\|\le \tilde{\beta}_{k,t}$ and $\beta_{k,t} = \hat{\alpha}_k \tilde{\beta}_{k,t}$ by definition. We now move on to bound $\|\sum_{k=\tau}^t\hat{\alpha}_k \tilde{B}_{k,t}\epsilon(k)\|$. It is straightforward that we use Azuma Hoeffding inequality to show this, which is presented in the following lemma.
    \begin{lemma}[Lemma 13 in \cite{Qu2020FiniteTimeAO}]
        Let $X_t$ be a $P_t$-adapted stochastic process with $\mathbb{E}X_t|P_{t-\tau}=0$. Meanwhile, $|X_t|\le \bar{X}_t$ almost surely, then with probability at least $1-\delta$, we have $\left| \sum_{k=0}^t X_k \right|\le \sqrt{2\tau \sum_{k=0}^t \bar{X}_k^2\log(\frac{2\tau}{\delta})}$.
    \end{lemma}
    Recall that $\sum_{k=\tau}^t \hat{\alpha}_k\tilde{B}_{k,t}\epsilon(k)$ is a random vector with its $i$-th entry $\sum_{k=\tau}^t \hat{\alpha}_k\epsilon_i(k)\prod_{l=k+1}^t(1-\hat{\alpha}_l d_{l,i})$, $d_{l,i}\ge\sigma^\prime=M\sigma$. Fixing $i$, $\epsilon_i(k)$ is a $P_{k+1}$ adapted stochastic process satisfying $\mathbb{E}\epsilon_i(k)|P_{k-\tau}=0$ (see Equation \ref{appeq:epsilon}). However, $\prod_{l=k+1}^t(1-\hat{\alpha}_ld_{l,i})$ is not $P_{k-\tau}$ measurable. To erase the randomness in it, we introduce the following lemma.
    \begin{lemma}[Adapted from Lemma 14 in \cite{Qu2020FiniteTimeAO}]
        \label{applemma:azuma}
        For each $i$, we have almost surely,
        \begin{align*}
            \left| \sum_{k=\tau}^t \hat{\alpha}_k\epsilon_i(k)\prod_{l=k+1}^t (1-\hat{\alpha}_l d_{l,i}) \right| \le \sup_{\tau\le k_0\le t}\left( \left| \sum_{k=k_0+1}^t \epsilon_i(k)\beta_{k,t} \right| + 2\bar{\epsilon}\beta_{k_0,t} \right).
        \end{align*}
    \end{lemma}
    \begin{proof}
        Replacing $\alpha_k$ with $\hat{\alpha}_k$ and setting $v_i=1$ (since we use standard infinity norm) in Lemma 14 of \cite{Qu2020FiniteTimeAO} conclude the proof.
    \end{proof}
    After that, we can proceed with the proof with the aid of the following lemma.
    \begin{lemma}
        \label{applemma:errorepsilon}
        For each $t$, with probability at least $1-\delta$, we have
        \begin{align*}
            \left\|\sum_{k=\tau}^t \hat{\alpha}_k \tilde{B}_{k,t}\epsilon(k)\right\| \le 6\bar{\epsilon} \sqrt{\dfrac{(\tau+1)h}{\sigma(t+1+t_0)}\log\left( \dfrac{2(\tau+1)tn}{\delta} \right)}.
        \end{align*}
    \end{lemma}
    \begin{proof}
        Fix $i$ and $\tau\le k_0\le t$, we have $\epsilon_i(k)\beta_{k,t}$ is a $P_{k+1}$ adapted stochastic process satisfying $\mathbb{E}\epsilon_i(k)\beta_{k,t}|P_{k-\tau}=0$. We also have $|\epsilon_i(k)\beta_{k,t}|\le |\epsilon_i(k)|\beta_{k,t}\le \bar{\epsilon}\beta_{k,t}$ (by using Lemma \ref{applemma:boundonepsilonphi}). We then can use the Azuma-Hoeffding bound in Lemma \ref{applemma:azuma}. With probability at least $1-\delta$, we have
        \begin{align*}
            \left| \sum_{k=k_0+1}^t\epsilon_i(k)\beta_{k,t} \right| \le \bar{\epsilon}\sqrt{2(\tau+1)\sum_{k=k_0+1}^t\beta^2_{k,t}\log\left( \dfrac{2(\tau+1)}{\delta} \right)}.
        \end{align*}
        By a union bound on $\tau\le k_0\le t$, we have with probability at least $1-\delta$,
        \begin{align*}
            \sup_{\tau\le k_0\le t}\left| \sum_{k=k_0+1}^t \epsilon_i(k)\beta_{k,t} \right| \le \bar{\epsilon}\sqrt{2(\tau+1)\sum_{k=\tau+1}^t\beta^2_{k,t}\log\left( \dfrac{2(\tau+1)t}{\delta} \right)}.
        \end{align*}
        Notice that $\sigma h>2$ and hence $\dfrac{(k_0 + 1+t_0)^{\sigma h}}{k_0+t_0}$ monotonically increases with $k_0$. Therefore, we have $\dfrac{(k_0 + 1+t_0)^{\sigma h}}{k_0+t_0} \le \dfrac{(t + 1+t_0)^{\sigma h}}{t+t_0},\forall\, \tau\le k_0\le t$. Here, we assume that $h> \dfrac{2}{\sigma(1-\gamma)}$ (again, we set $\sigma=\dfrac{1}{2}\mu_{\rm min}$) which obviously satisfies the assumption that $h>\dfrac{2}{\sigma}$ we make in Lemma \ref{applemma:beta}.
        
        Then, by using Lemma \ref{applemma:azuma} and Lemma \ref{applemma:beta}, we have with probability at least $1-\delta$,
        \begin{align*}
            \left| \sum_{k=\tau}^t \hat{\alpha}_k\epsilon_i(k)\prod_{l=k+1}^t(1-\hat{\alpha}_ld_{l,i}) \right| &\le \sup_{\tau\le k_0 \le t}\left( \left| \sum_{k=k_0+1}^t \epsilon_i(k) \beta_{k,t} \right| + 2\bar{\epsilon}\beta_{k_0,t} \right) \\
            &\le \bar{\epsilon}\sqrt{2(\tau+1)\sum_{k=\tau+1}^t\beta^2_{k,t}\log\left( \dfrac{2(\tau+1)t}{\delta} \right)} + \sup_{\tau\le k_0\le t}2\bar{\epsilon}\beta_{k_0,t} \\
            &\le 2\bar{\epsilon} \sqrt{\dfrac{(\tau+1)h}{\sigma (t+1+t_0)}\log\left( \dfrac{2(\tau+1)t}{\delta} \right)} + 2\bar{\epsilon} \sup_{\tau\le k_0 \le t} \dfrac{h}{k_0+t_0} \left( \dfrac{k_0+1+t_0}{t+1+t_0} \right)^{\sigma h} \\
            &\le 2\bar{\epsilon}\sqrt{\dfrac{(\tau+1)h}{\sigma(t+1+t_0)}\log\left(\dfrac{2(\tau+1)t}{\delta}\right)} + 2\bar{\epsilon}\dfrac{h}{t+t_0} \\
            &\le 6\bar{\epsilon}\sqrt{\dfrac{(\tau+1)h}{\sigma(t+1+t_0)}\log\left(\dfrac{2(\tau+1)t}{\delta}\right)}.
        \end{align*}
        The last inequality is due to that $\dfrac{1}{t+t_0}$ is asymptotically smaller than $\sqrt{\dfrac{1}{t+1+t_0}}$. Finally, by using the union bound over $i\in\{1,2,\ldots,n\}$, we have 
        \begin{align*}
            \left| \sum_{k=\tau}^t \hat{\alpha}_k\epsilon_i(k)\prod_{l=k+1}^t(1-\hat{\alpha}_ld_{l,i}) \right| &\le 6\bar{\epsilon}\sqrt{\dfrac{(\tau+1)h}{\sigma(t+1+t_0)}\log\left(\dfrac{2(\tau+1)tn}{\delta}\right)}.
        \end{align*}
    \end{proof}
    By replacing $\delta$ with $\frac{\delta}{t}$, we can rewrite the conclusion in Lemma \ref{applemma:errorepsilon} as:
    \begin{align*}
        \left\|\sum_{k=\tau}^t \hat{\alpha}_k \tilde{B}_{k,t}\epsilon(k)\right\| \le 6\bar{\epsilon} \sqrt{\dfrac{(\tau+1)h}{\sigma(t+1+t_0)}\log\left( \dfrac{2(\tau+1)t^2n}{\delta} \right)} := C_\epsilon \sqrt{\dfrac{1}{t+1+t_0}},
    \end{align*}
    where $C_\epsilon = 6\bar{\epsilon}\sqrt{\frac{(\tau+1)h}{\sigma}\log\left( \frac{2(\tau+1)t^2n}{\delta} \right)}$, then by recalling Equation \ref{appeq:recursiveform}, we have for $\tau\le t\le T$, with probability at least $1-\delta$,
    \begin{align*}
        a_{t+1} &\le \tilde{\beta}_{\tau-1,t}a_\tau + \gamma \sup_{i\in\mathcal{N}}\sum_{k=\tau}^t b_{k,t,i}a_k + \left\| \sum_{k=\tau}^t\hat{\alpha}_k \tilde{B}_{k,t}\epsilon(k) \right\| + \left\| \sum_{k=\tau}^t\hat{\alpha}_k \tilde{B}_{k,t}\phi(k) \right\| \\
        &\le \tilde{\beta}_{\tau-1,t}a_\tau + \gamma \sup_{i\in\mathcal{N}}\sum_{k=\tau}^t b_{k,t,i}a_k + \dfrac{C_\epsilon}{\sqrt{t+1+t_0}} + \dfrac{C_\phi}{t+1+t_0}.
    \end{align*}
    We now want to show that
    \begin{align}
        \label{appeq:finalconclusion}
        a_T \le \frac{C_a}{\sqrt{T+t_0}} + \frac{C_a^\prime}{T+t_0},
    \end{align}
    where $C_a = \frac{12\bar{\epsilon}}{1-\gamma}\sqrt{\frac{(\tau+1)h}{\sigma}\log\left(\frac{2(\tau+1)T^2n}{\delta}\right)}$, $C_a^\prime = \frac{4}{1-\gamma}\max(C_\phi, \frac{2(\tau+t_0)r_{\rm max}}{1-\gamma})$. We use induction to show Equation \ref{appeq:finalconclusion}. It is easy to see that when $t = \tau$, Equation \ref{appeq:finalconclusion} holds as $\frac{C_a^\prime}{\tau+t_0} \ge \frac{8r_{\rm max}}{(1-\gamma)^2}\ge a_\tau$, where $a_\tau = \|x(\tau) - x^*\|\le \|x(\tau)\| + \|x^*\|\le \frac{2r_{\rm max}}{1-\gamma}$. We then assume that Equation \ref{appeq:finalconclusion} holds for up to $k\le t$, then we have
    \begin{align*}
        a_{t+1} &\le \tilde{\beta}_{\tau-1,t}a_\tau + \gamma \sup_{i\in\mathcal{N}}\sum_{k=\tau}^t b_{k,t,i}a_k + \dfrac{C_\epsilon}{\sqrt{t+1+t_0}} + \dfrac{C_\phi}{t+1+t_0} \\
        &\le \tilde{\beta}_{\tau-1,t}a_\tau + \gamma \sup_{i\in\mathcal{N}}\sum_{k=\tau}^t b_{k,t,i}\left( \dfrac{C_a}{\sqrt{k+t_0}} + \dfrac{C_a^\prime}{k+t_0} \right) + \dfrac{C_\epsilon}{\sqrt{t+1+t_0}} + \dfrac{C_\phi}{t+1+t_0} \\
        &= \tilde{\beta}_{\tau-1,t}a_\tau + \gamma \sup_{i\in\mathcal{N}}\sum_{k=\tau}^t b_{k,t,i}\dfrac{C_a}{\sqrt{k+t_0}} + \gamma \sup_{i\in\mathcal{N}}\sum_{k=\tau}^t b_{k,t,i}\dfrac{C_a^\prime}{k+t_0}+ \dfrac{C_\epsilon}{\sqrt{t+1+t_0}} + \dfrac{C_\phi}{t+1+t_0} \\
        &\le \left( \dfrac{\tau+t_0}{t+1+t_0} \right)^{\sigma h}a_\tau + \dfrac{C_\phi}{t+1+t_0} + \gamma \sup_{i\in\mathcal{N}}\sum_{k=\tau}^t b_{k,t,i}\dfrac{C_a^\prime}{k+t_0} + \gamma \sup_{i\in\mathcal{N}}\sum_{k=\tau}^t b_{k,t,i}\dfrac{C_a}{\sqrt{k+t_0}} + \dfrac{C_\epsilon}{\sqrt{t+1+t_0}}
    \end{align*}
    where we use the bound for $\tilde{\beta}_{k,t}$ in Lemma \ref{applemma:beta}. To finish the final step of the proof, we need the aid of the following lemma.
    \begin{lemma}[Adapted from Lemma 15 in \cite{Qu2020FiniteTimeAO}]
        \label{applemma:sqrtgamma}
        If $\sigma h(1-\sqrt{\gamma})\ge 1$, $t_0\ge 1$ and $\alpha_0 \le \frac{1}{2M}$. Then for any $i\in\mathcal{N}=\{1,2,\ldots,n\}$ and any $\omega\in(0,1]$, we have,
        \begin{align*}
            \sum_{k=\tau}^t b_{k,t,i}\dfrac{1}{(k+t_0)^\omega} \le \dfrac{1}{\sqrt{\gamma}(t+1+t_0)^\omega}.
        \end{align*}
    \end{lemma}
    \begin{proof}
        Denote $e_t = \sum_{k=\tau}^t b_{k,t,i}\dfrac{1}{(k+t_0)^\omega}$. We use induction to show that $e_t\le \dfrac{1}{\sqrt{\gamma}(t+1+t_0)^\omega}$. The conclusion is true for $t = \tau$ as $\hat{\alpha}_\tau \le M\alpha_\tau\le \dfrac{1}{2}$, then $e_\tau = b_{\tau,\tau,i}\dfrac{1}{(\tau+t_0)^\omega} = \hat{\alpha}_\tau d_{\tau,i}\dfrac{1}{(\tau+t_0)^\omega} \le \dfrac{1}{\sqrt{\gamma}(\tau+1+t_0)^\omega}$ due to $\left(1+\dfrac{1}{t_0}\right)^\omega \le 1 + \dfrac{1}{t_0} \le 2 \le \dfrac{2}{\sqrt{\gamma}}$, $t_0\ge1, \omega\le 1$. Then we assume the statement is true for $t-1$, then we have
        \begin{align*}
            e_t &= \sum_{k=\tau}^{t-1}b_{k,t,i}\dfrac{1}{(k+t_0)^\omega} + b_{t,t,i}\dfrac{1}{(t+t_0)^\omega} = (1-\hat{\alpha}_td_{t,i})\sum_{k=\tau}^{t-1}b_{k,t-1,i}\dfrac{1}{(k+t_0)^\omega} + \hat{\alpha}_td_{t,i}\dfrac{1}{(t+t_0)^\omega} \\
            &= (1-\hat{\alpha}_td_{t,i})e_{t-1} + \hat{\alpha}_td_{t,i}\dfrac{1}{(t+t_0)^\omega} \le (1-\hat{\alpha}_td_{t,i}) \dfrac{1}{\sqrt{\gamma}(t+t_0)^\omega} + \hat{\alpha}_td_{t,i}\dfrac{1}{(t+t_0)^\omega} \\
            &= \left[ 1-\hat{\alpha}_t d_{t,i}(1-\sqrt{\gamma}) \right] \dfrac{1}{\sqrt{\gamma}(t+t_0)^\omega} \le \left[ 1-\alpha_t M\sigma(1-\sqrt{\gamma}) \right] \dfrac{1}{\sqrt{\gamma}(t+t_0)^\omega} = \left[ 1-\dfrac{h}{t+t_0} \sigma(1-\sqrt{\gamma}) \right] \dfrac{1}{\sqrt{\gamma}(t+t_0)^\omega} \\
            &= \left[ 1-\dfrac{\sigma h}{t+t_0} (1-\sqrt{\gamma}) \right]\left(\dfrac{t+1+t_0}{t+t_0}\right)^\omega \dfrac{1}{\sqrt{\gamma}(t+1+t_0)^\omega} = \left[ 1-\dfrac{\sigma h}{t+t_0} (1-\sqrt{\gamma}) \right]\left(1 + \dfrac{1}{t+t_0}\right)^\omega \dfrac{1}{\sqrt{\gamma}(t+1+t_0)^\omega},
        \end{align*}
        where we have used the fact that $\hat{\alpha}_k \ge \alpha_k$ and $d_{t,i}\ge M \sigma$. Using the fact that for any $x>-1$, $(1+x)\le e^x$, we have,
        \begin{align*}
            \left[ 1-\dfrac{\sigma h}{t+t_0} (1-\sqrt{\gamma}) \right]\left(1 + \dfrac{1}{t+t_0}\right)^\omega \le e^{-\frac{\sigma h}{t+t_0}(1-\sqrt{\gamma}) + \omega\frac{1}{t+t_0}} \le 1,
        \end{align*}
        where we have used $\omega \le 1$ and $\sigma h(1-\sqrt{\gamma})\ge 1$, therefore $\omega - \sigma h(1-\sqrt{\gamma})\le 0$. Thus, we have
        \begin{align*}
            e_t \le \left[ 1-\dfrac{\sigma h}{t+t_0} (1-\sqrt{\gamma}) \right]\left(1 + \dfrac{1}{t+t_0}\right)^\omega \dfrac{1}{\sqrt{\gamma}(t+1+t_0)^\omega} \le \dfrac{1}{\sqrt{\gamma}(t+1+t_0)^\omega}.
        \end{align*}
        This finishes the induction, and concludes the proof of this lemma.
    \end{proof}
    By using Lemma \ref{applemma:sqrtgamma} and setting $\omega = 1, \frac{1}{2}$, respectively, we have
    \begin{align*}
        a_{t+1} &\le \left( \dfrac{\tau+t_0}{t+1+t_0} \right)^{\sigma h}a_\tau + \dfrac{C_\phi}{t+1+t_0} + \gamma \sup_{i\in\mathcal{N}}\sum_{k=\tau}^t b_{k,t,i}\dfrac{C_a^\prime}{k+t_0} + \gamma \sup_{i\in\mathcal{N}}\sum_{k=\tau}^t b_{k,t,i}\dfrac{C_a}{\sqrt{k+t_0}} + \dfrac{C_\epsilon}{\sqrt{t+1+t_0}} \\
        &\le \left( \dfrac{\tau+t_0}{t+1+t_0} \right)^{\sigma h}a_\tau + \dfrac{C_\phi}{t+1+t_0} + \sqrt{\gamma} \dfrac{C_a^\prime}{t+1+t_0} + \sqrt{\gamma} \dfrac{C_a}{\sqrt{t+1+t_0}} + \dfrac{C_\epsilon}{\sqrt{t+1+t_0}}.
    \end{align*}
    Denote $F_t = \sqrt{\gamma}\dfrac{C_a}{\sqrt{t+1+t_0}} + \dfrac{C_\epsilon}{\sqrt{t+1+t_0}}$ and $F_t^\prime = \sqrt{\gamma}\dfrac{C_a^\prime}{t+1+t_0} + \dfrac{C_\phi}{t+1+t_0}+\left( \dfrac{\tau+t_0}{t+1+t_0} \right)^{\sigma h}a_\tau$, then we have $a_{t+1}\le F_t+F_t^\prime$. It suffices to show $F_t\le \dfrac{C_a}{\sqrt{t+1+t_0}}, F_t^\prime\le\dfrac{C_a^\prime}{t+1+t_0}$.

    Notice that $\dfrac{C_\epsilon}{C_a} = \dfrac{6\bar{\epsilon}\sqrt{\frac{(\tau+1)h}{\sigma}\log\left( \frac{2(\tau+1)t^2n}{\delta} \right)}}{\frac{12\bar{\epsilon}}{1-\gamma}\sqrt{\frac{(\tau+1)h}{\sigma}\log\left(\frac{2(\tau+1)T^2n}{\delta}\right)}}\le \dfrac{1-\gamma}{2}\le 1-\sqrt{\gamma}$. The last inequality is a direct result of the fact that $(\sqrt{\gamma}-1)^2\ge 0$. Thus $F_t\le \dfrac{C_a}{\sqrt{t+1+t_0}}$.

    We also notice that $\dfrac{a_\tau(\tau+t_0)}{C_a^\prime} \le \dfrac{2r_{\rm max}}{1-\gamma} \dfrac{\tau+t_0}{C_a^\prime}\le\dfrac{1-\gamma}{4}\le \dfrac{1-\sqrt{\gamma}}{2}$. Furthermore, we have $\dfrac{C_\phi}{C_a^\prime} \le \dfrac{1-\gamma}{4} \le\dfrac{1-\sqrt{\gamma}}{2}$. Then, we have
    \begin{align*}
        F_t^\prime &= \sqrt{\gamma}\dfrac{C_a^\prime}{t+1+t_0} + \dfrac{C_\phi}{t+1+t_0}+\left( \dfrac{\tau+t_0}{t+1+t_0} \right)^{\sigma h}a_\tau \le \sqrt{\gamma}\dfrac{C_a^\prime}{t+1+t_0} + \dfrac{C_\phi}{t+1+t_0}+\dfrac{a_\tau(\tau+t_0)}{t+1+t_0} \\
        &\le \sqrt{\gamma}\dfrac{C_a^\prime}{t+1+t_0} + \dfrac{1-\sqrt{\gamma}}{2}\dfrac{C_a^\prime}{t+1+t_0} + \dfrac{1-\sqrt{\gamma}}{2}\dfrac{C_a^\prime}{t+1+t_0} = \dfrac{C_a^\prime}{t+1+t_0}
    \end{align*}
    This finishes the induction, and we have $a_T \le \frac{C_a}{\sqrt{T+t_0}} + \frac{C_a^\prime}{T+t_0}$, where $C_a = \frac{12\bar{\epsilon}}{1-\gamma}\sqrt{\frac{(\tau+1)h}{\sigma}\log\left(\frac{2(\tau+1)T^2n}{\delta}\right)}$, $C_a^\prime = \frac{4}{1-\gamma}\max(C_\phi, \frac{2(\tau+t_0)r_{\rm max}}{1-\gamma}), C_\phi = \frac{16\bar{\epsilon}h\tau}{\sigma}$. Based on Lemma \ref{applemma:boundonepsilonphi}, we have $\bar{\epsilon}:= \dfrac{4r_{\rm max}}{1-\gamma} + C$ where $C\le (1+\gamma)\|x^*\|\le 2\dfrac{r_{\rm max}}{1-\gamma}$. Therefore, we have $\bar{\epsilon}\le \dfrac{6r_{\rm max}}{1-\gamma}$. Taken together with $\tau = \lceil \log_2(\frac{2}{\mu_{\rm min}}) \rceil t_{\rm mix}$ and $\sigma = \dfrac{\mu_{\rm min}}{2}$, we have with probability at least $1-\delta$,
    \begin{align*}
        \|\hat{Q}_T - Q^*\| &\le \dfrac{72r_{\rm max}}{(1-\gamma)^2} \sqrt{\dfrac{2(\lceil \log_2(\frac{2}{\mu_{\rm min}}) \rceil t_{\rm mix} + 1)h}{\mu_{\rm min} (T+t_0)} \log\left( \dfrac{2(\lceil \log_2(\frac{2}{\mu_{\rm min}}) \rceil t_{\rm mix} + 1)T^2|\mathcal{S}||\mathcal{A}|}{\delta} \right)} \\
        & + \dfrac{4r_{\rm max}}{(1-\gamma)^2} \max \left( \dfrac{192 h \lceil \log_2(\frac{2}{\mu_{\rm min}}) \rceil t_{\rm mix}}{\mu_{\rm min}}, 2\left( \lceil \log_2(\frac{2}{\mu_{\rm min}}) \rceil t_{\rm mix} + t_0\right) \right)\dfrac{1}{T+t_0} \\
        &\simeq \tilde{\mathcal{O}}\left( \dfrac{r_{\rm max}\sqrt{t_{\rm mix}}}{(1-\gamma)^{2.5}\mu_{\rm min}}\dfrac{1}{\sqrt{T}} + \dfrac{ r_{\rm max}t_{\rm mix}}{(1-\gamma)^3\mu_{\rm min}^2}\dfrac{1}{T} \right).
    \end{align*}
    The above inequality holds when we take $h=\Theta(\frac{1}{\mu_{\rm min}(1-\gamma)}), t_0 = \tilde{\Theta}(\max(\frac{1}{\mu_{\rm min}(1-\gamma)}, t_{\rm mix}))$. The whole proof is thus completed.
\end{proof}

\subsection{Proof of Corollary \ref{coro:samplecomplexity}}
\begin{corollary}[Sample complexity]
For any $0<\delta<1$ and $0<\epsilon<1$, with Q-SMR algorithm we have:
\begin{equation}
    \forall (s,a)\in\mathcal{S}\times\mathcal{A}: \|\hat{Q}_T - Q^*\|_\infty \le \epsilon,
\end{equation}
holds with probability at least $1-\delta$, provided the iteration number $T$ obeys:
\begin{equation}
    T\stackrel{>}{\sim}\dfrac{r_{\rm max}^2 t_{\rm mix}}{(1-\gamma)^5\mu_{\rm min}^2}\dfrac{1}{\epsilon^2}.
\end{equation}
\end{corollary}
\begin{proof}
    The proof is quite straightforward. Since $\|\hat{Q}_T - Q^*\| \le \tilde{\mathcal{O}}\left( \dfrac{r_{\rm max}\sqrt{t_{\rm mix}}}{(1-\gamma)^{2.5}\mu_{\rm min}}\dfrac{1}{\sqrt{T}} + \dfrac{r_{\rm max}t_{\rm mix}}{(1-\gamma)^3\mu_{\rm min}^2}\dfrac{1}{T} \right)$. Reaching an accuracy of $\epsilon$ means that $\dfrac{r_{\rm max}\sqrt{t_{\rm mix}}}{(1-\gamma)^{2.5}\mu_{\rm min}}\dfrac{1}{\sqrt{T}} + \dfrac{ r_{\rm max}t_{\rm mix}}{(1-\gamma)^3\mu_{\rm min}^2}\dfrac{1}{T} \le \epsilon$. With the scale of $\dfrac{1}{\sqrt{T}}$ and $\dfrac{1}{T}$, $\dfrac{r_{\rm max}\sqrt{t_{\rm mix}}}{(1-\gamma)^{2.5}\mu_{\rm min}}\dfrac{1}{\sqrt{T}}\le \epsilon$ is sufficient, which leads to $T\ge \dfrac{r_{\rm max}^2 t_{\rm mix}}{(1-\gamma)^5\mu_{\rm min}^2}\dfrac{1}{\epsilon^2}$.
\end{proof}

\section{Missing Experimental Results and Details}
\label{sec:missingexperiments}
In this section, we provide some missing experimental results and details. We first demonstrate the experimental results of TD3-SMR and DDPG-SMR, and we also show how reducing SMR ratio and increasing batch size will affect them. We then list the full results of SAC-SMR on DMC suite \cite{Tassa2018DeepMindCS} and PyBullet-Gym \cite{benelot2018}, including state-based tasks and image-based tasks. Furthermore, we show that SMR can boost the sample efficiency of the base algorithm with longer online interactions (1M online interactions). We also conduct experiments on Arcade Learning Environment (Atari) where we combine SMR with DQN \cite{Mnih2015HumanlevelCT}. Finally, we show that SMR can improve sample efficiency regardless of the initial learning rate. 
\subsection{Performance of TD3-SMR and DDPG-SMR}
\label{appsec:td3smrddpgsmr}
We summarize the full performance comparison of TD3-SMR against the vanilla TD3 as well as DDPG-SMR versus DDPG (here we use our DDPG from \cite{Fujimoto2018AddressingFA}) on four continuous control tasks from OpenAI Gym \cite{Brockman2016OpenAIG} in Figure \ref{fig:td3-smr-ddpg-smr}. We use $M=10$ by default. We notice that the sample efficiency of both TD3 and DDPG benefit greatly from SMR on many of the evaluated tasks. While we do observe a sort of performance instability on Ant-v2 for TD3-SMR, and find that DDPG-SMR underperforms the vanilla DDPG. For TD3, this may be because the neural networks encounter the phenomenon of overfitting in this environment. While for DDPG, this may be due to the fact that \textit{SMR does not modify the way of value estimation}, indicating that the phenomenon of overestimation still exists in DDPG-SMR. The overestimation bias can be accumulated during the sample reuse loop on Ant-v2, resulting in poor performance. On other tasks, we find that SMR consistently aids the sample efficiency of the base algorithm for both TD3 and DDPG, often by a large margin. As mentioned in Section \ref{sec:overfitting}, the ways of remedying the overfitting phenomenon can be (1) using smaller $M$, e.g., $M=5$; (2) resetting the agent periodically; (3) leveraging a larger batch size; etc. We show below the effectiveness of part of them, including using a smaller SMR ratio and using a larger batch size.

\begin{figure}
    \centering
    \includegraphics[width=0.95\linewidth]{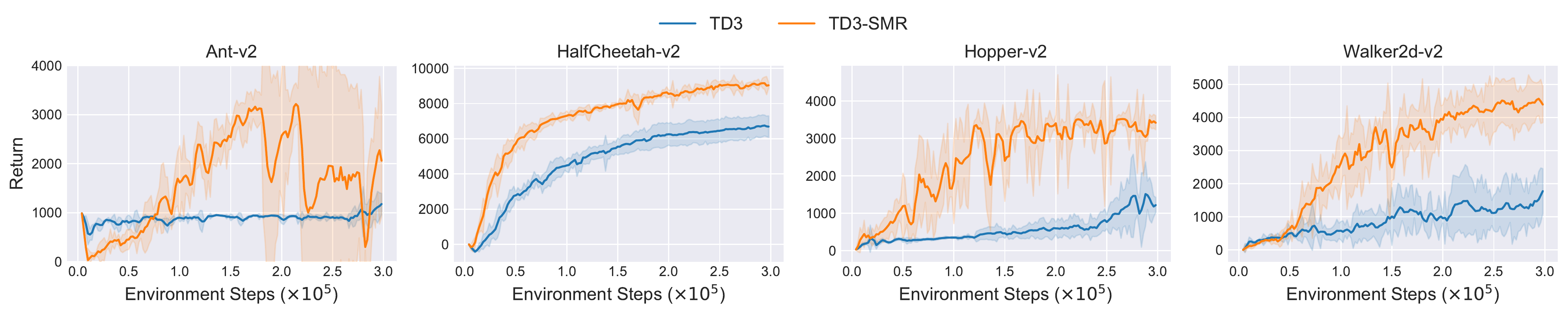}
    \includegraphics[width=0.95\linewidth]{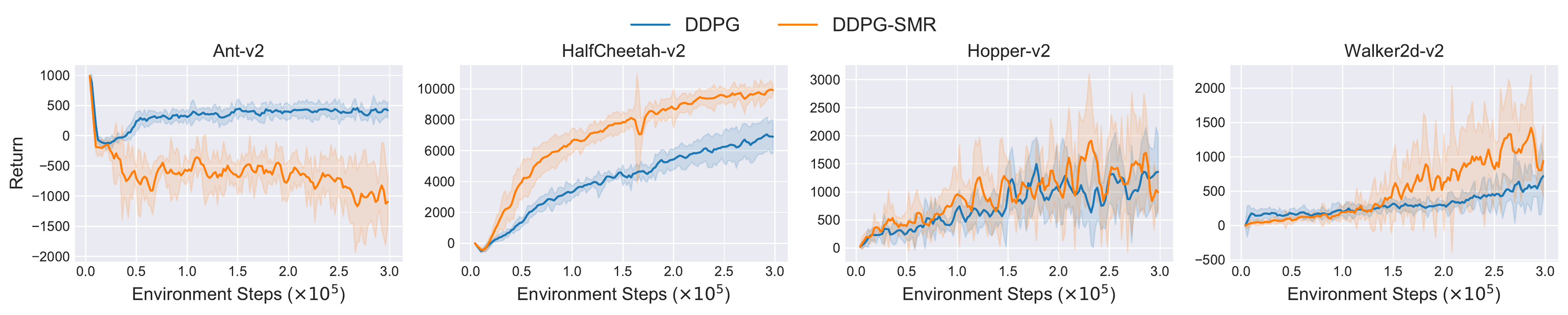}
    \caption{Experimental results of TD3-SMR against TD3 and DDPG-SMR against DDPG. The results are averaged over 6 random seeds, and the shaded region denotes the standard deviation.}
    \label{fig:td3-smr-ddpg-smr}
\end{figure}

\begin{figure}[h]
    \centering
    \subfigure[TD3-SMR batch size]{
    \label{fig:td3batchsize}
    \includegraphics[scale=0.6]{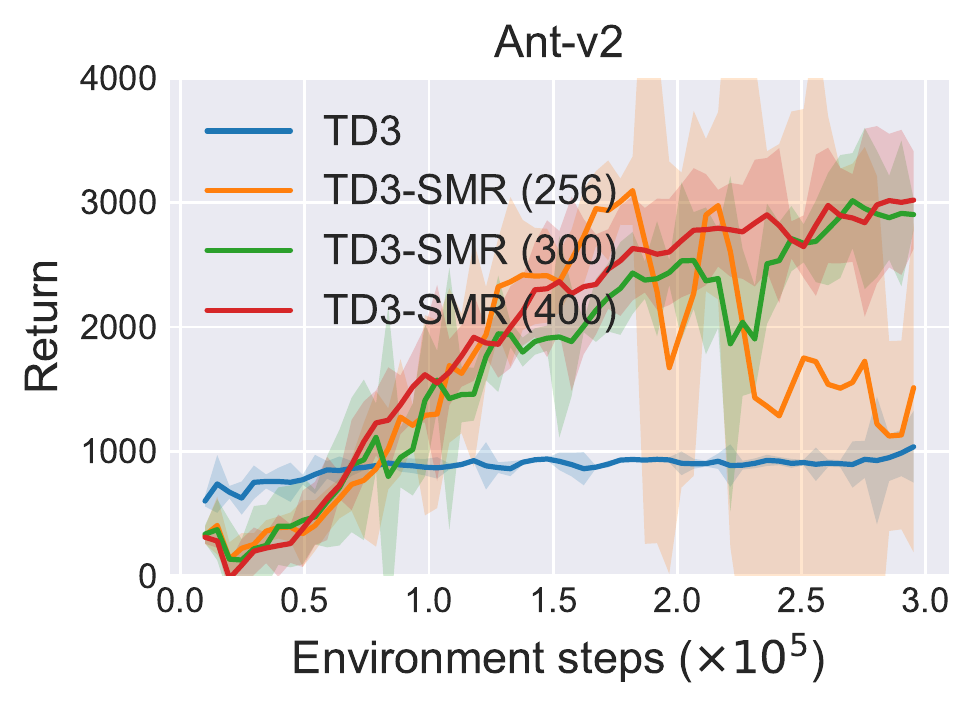}
    }\hspace{-2mm}
    \subfigure[TD3-SMR SMR ratio]{
    \label{fig:td3ratio}
    \includegraphics[scale=0.6]{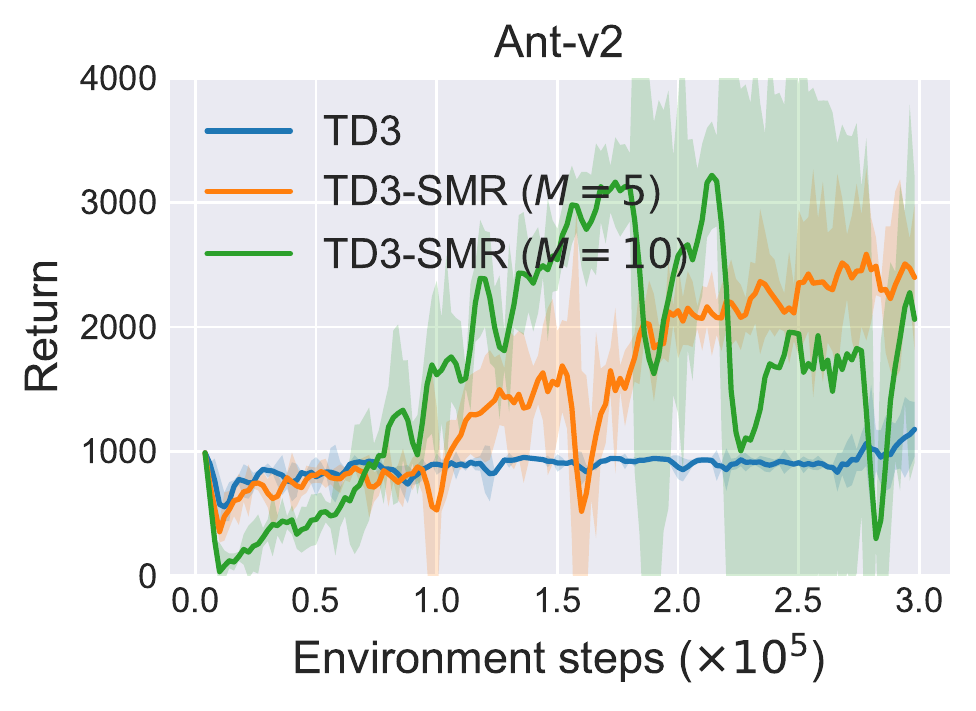}
    }\hspace{-2mm}
    \subfigure[DDPG-SMR batch size]{
    \label{fig:ddpgbatchsize}
    \includegraphics[scale=0.6]{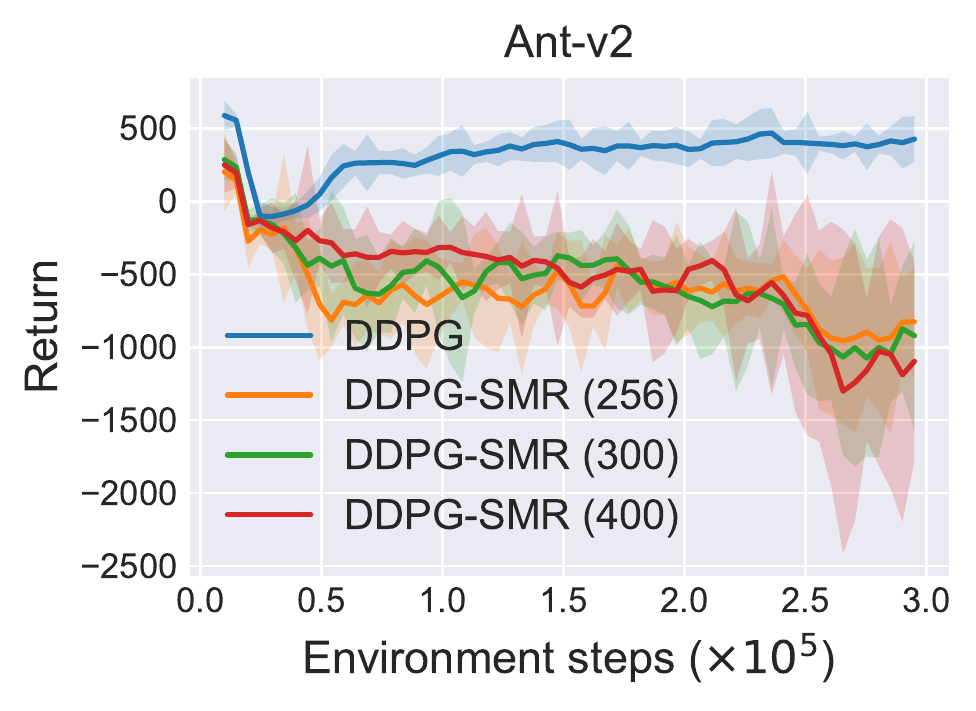}
    }\hspace{-2mm}
    \subfigure[DDPG-SMR SMR ratio]{
    \label{fig:ddpgratio}
    \includegraphics[scale=0.6]{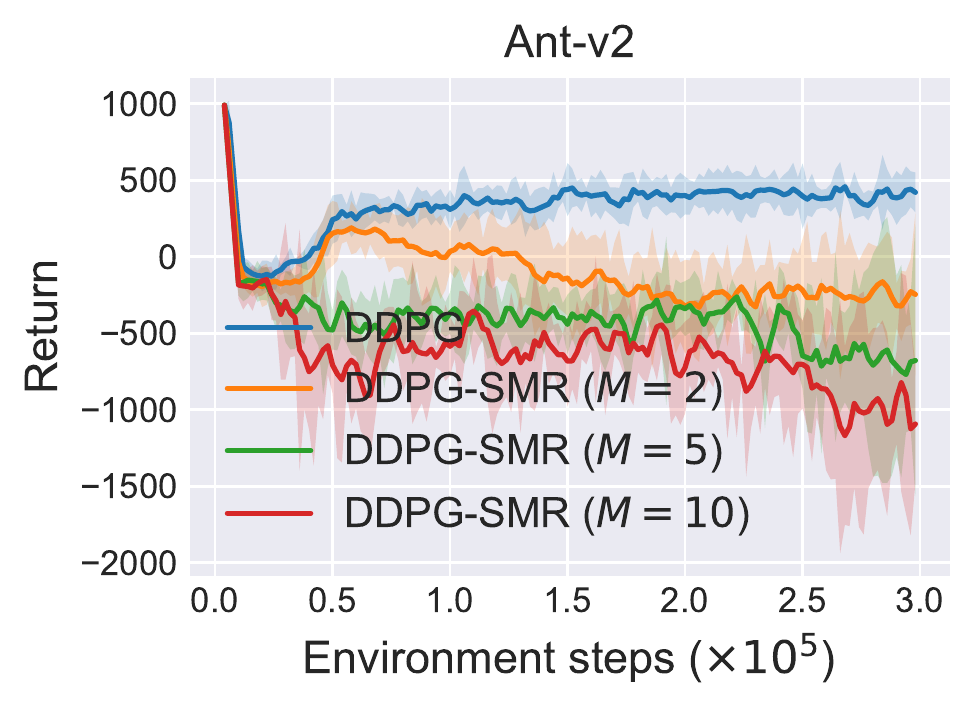}
    }\hspace{-2mm}
    \caption{Performance comparison of TD3-SMR and DDPG-SMR against their base algorithms on Ant-v2. (a) TD3-SMR with different batch sizes where we sweep over $\{256,300,400\}$; (b) TD3-SMR with different SMR ratios $M$ where we compare TD3-SMR ($M=10$) against TD3-SMR ($M=5$); (c) DDPG-SMR with different batch sizes in $\{256,300,400\}$; (d) DDPG-SMR with different SMR ratios $M, M\in\{2,5,10\}$. We report the mean performance along with the standard deviation over 6 different random seeds.}
    \label{fig:alleviatingoverfitting}
\end{figure}

We summarize the empirical results in Figure \ref{fig:alleviatingoverfitting}, where we run for 300K online interactions and evaluate the agent every 1000 timesteps over 10 trials. We find that using a smaller SMR ratio or using a larger batch size is beneficial to the stability and satisfying performance for TD3-SMR as shown in Figure \ref{fig:td3batchsize} and \ref{fig:td3ratio}. However, it can be seen that DDPG-SMR does not seem to benefit from either a smaller SMR ratio $M$ (even $M=2$) or a larger batch size. Using a smaller SMR ratio or larger batch size can help improve DDPG-SMR with $M=10$ to some extent. While they still underperform vanilla DDPG. This is due to \textit{SMR only enforces more updates on the fixed batch data instead of dealing with overestimation bias}. As shown in Equation \ref{eq:smrlearningrate}, SMR tends to \textit{smooth the gradient for updating} by leveraging the gradient of intermediate values. On tasks like \texttt{HalfCheetah-v2} and \texttt{Walker2d-v2}, SMR can benefit DDPG by better exploiting collected data, while on some tasks like \texttt{Ant-v2}, DDPG-SMR does not seem to be able to escape from the curse of overestimation bias. We, therefore, recommend the application of SMR upon off-policy continuous control algorithms that can address the overestimation bias themselves, e.g., TD3 \cite{Fujimoto2018AddressingFA} by using clipped double Q-learning; TQC \cite{Kuznetsov2020ControllingOB} by truncating small proportion of estimated $Q$ distribution, etc.

\subsection{Omitted results from DMC suite and PyBullet-Gym}
We demonstrate in this subsection the missing experimental results of SAC-SMR on DMC suite \cite{Tassa2018DeepMindCS} and PyBullet-Gym \cite{benelot2018}. The performance comparison of SAC-SMR and SAC is available in Figure \ref{fig:missingstatebasedtasks}. As expected, we observe that SAC-SMR outperforms the vanilla SAC on all of the evaluated state-based tasks. These further show that SMR can benefit the sample efficiency of the base algorithm on a wide range of different tasks.

\begin{figure}
    \centering
    \includegraphics[width=0.95\linewidth]{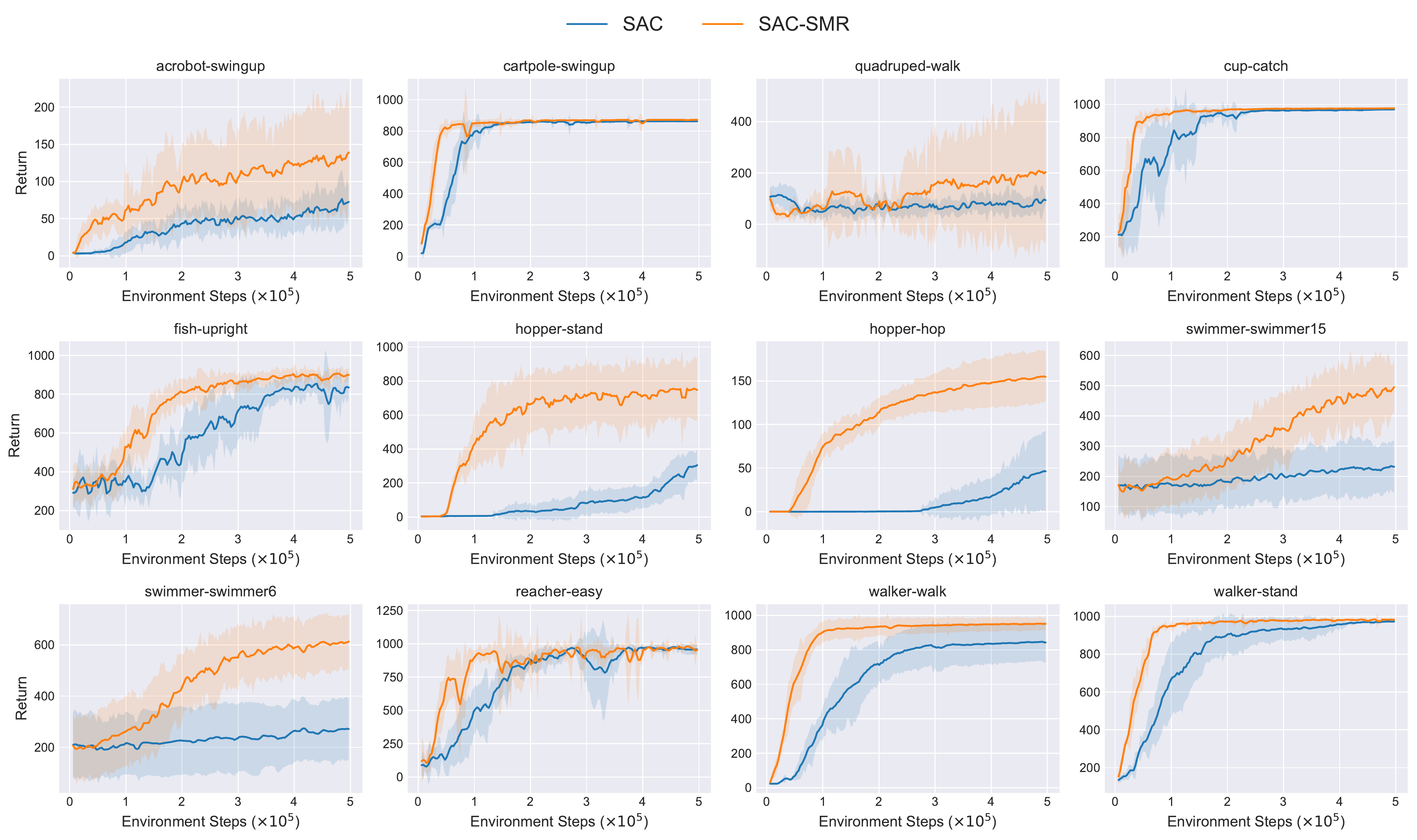}
    \includegraphics[width=0.95\linewidth]{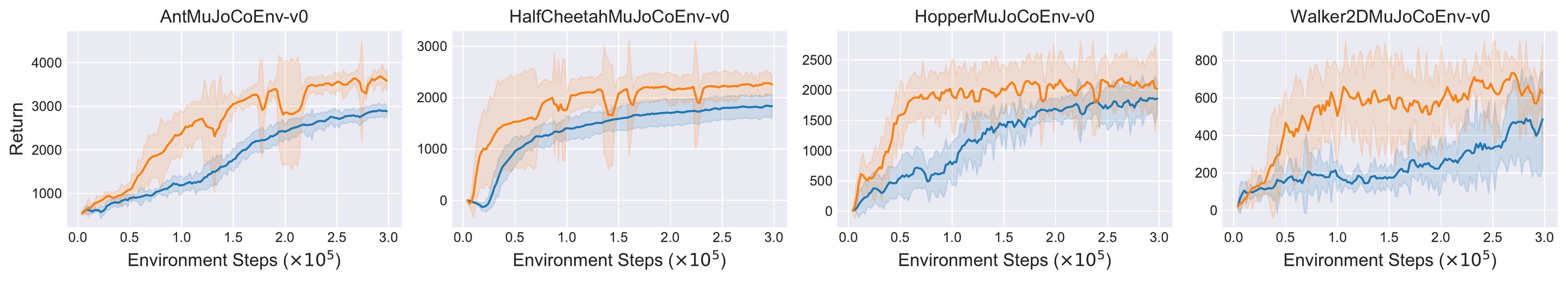}
    \caption{Experimental results of SAC-SMR against SAC on state-based tasks from DMC suite and PyBullet-Gym. The results are averaged over 6 random seeds with 500K online interactions, and the shaded region denotes the standard deviation.}
    \label{fig:missingstatebasedtasks}
\end{figure}

We further demonstrate in Figure \ref{fig:dmc-imagesacsmr} the experimental results on 4 additional image-based tasks from DMC suite. Note that we run experiments on DMC suite 100K benchmarks for SAC-SMR. For image-based tasks, we use a comparatively small SMR ratio $M=5$ as it will be very time-consuming to adopt a $M=10$ (image-based tasks already take much longer time to run than state-based tasks). For example, it takes about 4 hours to run with SAC on \texttt{reacher-easy} while it takes about 15 hours to run with SAC-SMR ($M=5$) on this task. If we adopt $M=10$, it will take more than 24 hours. We observe that on some of the image-based tasks, our SMR can boost the sample efficiency of SAC, e.g., SAC-SMR learns faster than vanilla SAC on \texttt{cheetah-run} and beats SAC on \texttt{cup-catch}. While we also see that SAC-SMR kind of underperforming SAC on \texttt{cartpole-swingup} and \texttt{reacher-easy}. In fact, we do not see a large margin on image-based tasks as on the state-based tasks. We attribute the reason to \textit{bad representation}. We usually leverage an encoder to deal with image input, where we do representation learning to reduce the size of the input. However, the parameters of the encoder are also continually updated during the training process. The error in representations accumulates and may impede the agent from learning better policy. For some of the tasks, SMR can benefit the agent, while on some tasks, things are different. SMR can benefit the state-based tasks as the states are \textit{precise representation} of the information that the agent needs to perform control. This phenomenon also exists on Atari tasks, one can refer to Appendix \ref{appsec:atari} for more details. Furthermore, as mentioned in \cite{yarats2022mastering}, the automatic entropy adjustment strategy in SAC is inadequate and in some cases may result in a premature \textit{entropy collapse}. This will in turn prevent the agent from finding more optimal behaviors. SMR can somewhat worsen this phenomenon due to multiple updates on the sampled batch. These we believe can explain the failure of SAC-SMR on some of the image-based tasks.

\begin{figure}
    \centering
    \includegraphics[width=0.95\linewidth]{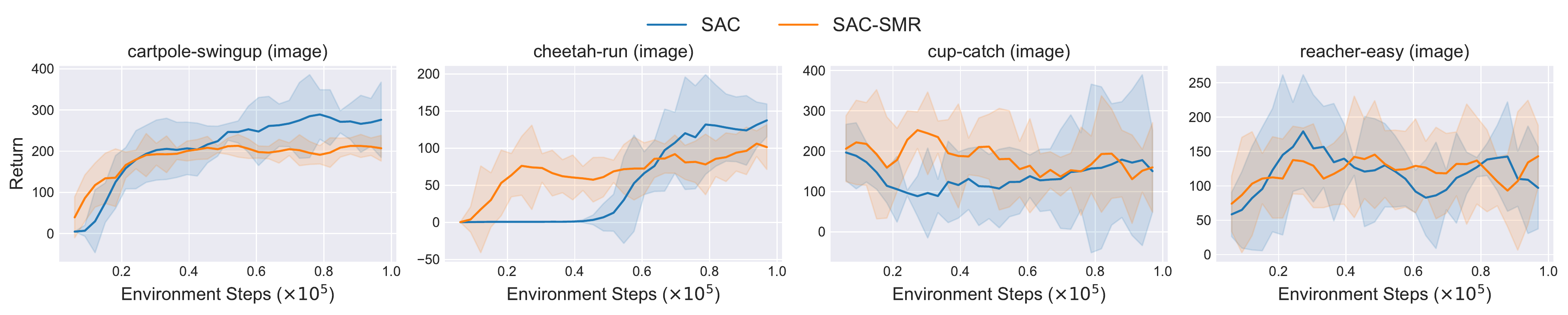}
    \caption{Experimental results on 4 image-based tasks from DMC suite 100K benchmarks. The results are averaged over 6 different random seeds, and the shaded region denotes the standard deviation.}
    \label{fig:dmc-imagesacsmr}
\end{figure}

One interesting question is: whether more advanced methods for image-based tasks can benefit from SMR? To answer this question, we select the most recent DrQ-v2 \cite{yarats2022mastering} and combine it with our SMR. DrQ-v2 is built upon DrQ \cite{yarats2021image}, an actor-critic approach that uses data augmentation to learn directly from pixels. The improvements of DrQ-v2 over DrQ include: (1) switch SAC to DDPG (to avoid entropy collapse); (2) incorporate $n$-step returns to estimate temporal difference (TD); (3) introduce a decaying schedule for exploration noise error; (4) improve running speed; (5) find better hyperparameters. We adopt the SMR ratio $M=5$ as we do for image-based tasks in SAC-SMR. We demonstrate in Figure \ref{fig:drqv2} that DrQ-v2-SMR can outperform DrQ-v2 on most of the evaluated tasks (e.g., \texttt{cup-catch}, \texttt{cartpole-swingup}) and is competitive to DrQ-v2 on other tasks (e.g., \texttt{cheetah-run}). We also compare the final performance of DrQ-v2 and DrQ-v2-SMR at 500K frames in Table \ref{tab:drqv2performance}, where we unsurprisingly find the advantage of DrQ-v2-SMR over DrQ-v2. The success of SMR upon DrQ-v2 may due to (1) no entropy collapse and better exploration mechanism; (2) data augmentation to help alleviate the negative influence of initial bad representation.

Note that DrQ-v2-SMR spends 3 times of training time than DrQ-v2. Thanks to the fast running speed of DrQ-v2, this cost is comparatively tolerable. For example, DrQ-v2 requires 7 hours on \texttt{finger-spin} while DrQ-v2-SMR takes 20 hours.

\begin{figure}
    \centering
    \includegraphics[width=0.8\linewidth]{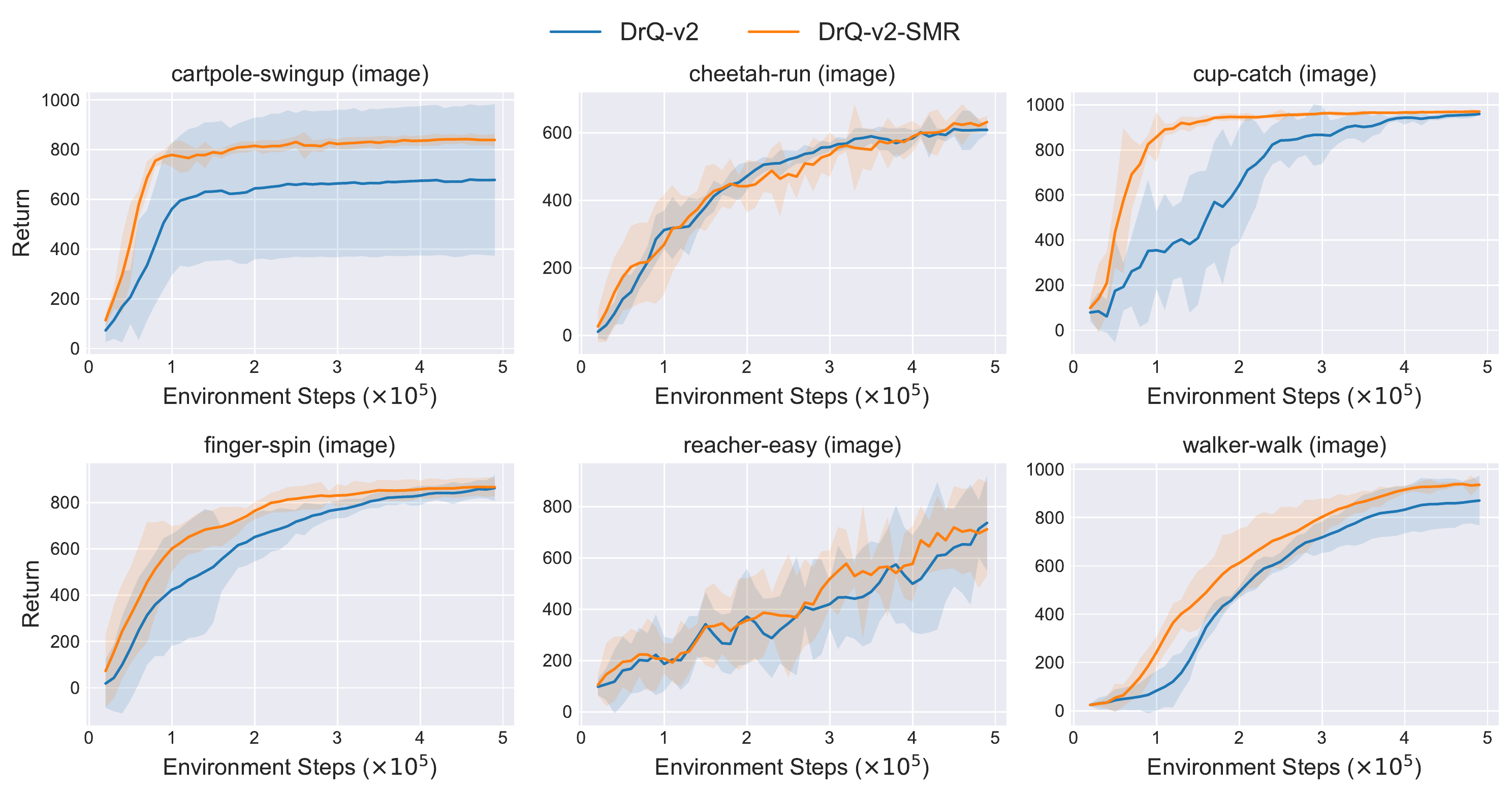}
    \caption{Experimental results of DrQ-v2-SMR against vanilla DrQ-v2 on six image-based tasks from DMC suite. Each algorithm is run for 500K frames and evaluated over 10 trials every 1000 frames. The results are averaged over 6 different random seeds. We report the mean performance and the standard deviation.}
    \label{fig:drqv2}
\end{figure}

\begin{table}
  \caption{Performance comparison of DrQ-v2 and DrQ-v2-SMR on six image-based tasks from DMC suite. The numbers indicate the performance achieved when the specific number of frames is seen. We \textbf{bold} the best mean results.}
  \vspace{0.2cm}
  \renewcommand\arraystretch{1.05}
  \label{tab:drqv2performance}
  \centering
  \small
  \begin{tabular}{l|ll}
    \toprule
    Tasks  & DrQ-v2 & DrQ-v2-SMR \\
    \midrule
    cartpole-swingup@500K & 682.7$\pm$304.8 & \textbf{842.5}$\pm$25.2 \\
    cheetah-run@500K & 605.5$\pm$12.3 & \textbf{626.5}$\pm$17.8 \\  
    cup-catch@500K & 965.6$\pm$5.9 & \textbf{970.2}$\pm$5.1 \\
    finger-spin@500K & 867.7$\pm$55.0 & \textbf{872.0}$\pm$39.3 \\
    reacher-easy@500K & \textbf{736.8}$\pm$185.5 & 736.2$\pm$182.1 \\
    walker-walk@500K & 869.5$\pm$102.6 & \textbf{949.3}$\pm$8.9 \\
    \bottomrule
  \end{tabular}
\end{table}

\subsection{Can SMR still work with longer online interactions?}
\label{sec:smrlonger}
In the main text and the appendix above, we run most of the experiments with only 300K online interactions or 500K online interactions (100K for SAC-SMR on image-based tasks from the DMC suite). Though 300K or 500K (or even fewer) online interactions are widely adopted for examining sample efficiency in model-based methods \cite{Janner2019WhenTT, Pan2020TrustTM, Lai2020BidirectionalMP, wu2022plan} and REDQ \cite{Chen2021RandomizedED}, one may wonder whether our method can consistently improve sample efficiency with longer online interactions. To address this concern, we run SAC-SMR ($M=10$) on 16 tasks from the DMC suite for a typical 1M online interactions. Each algorithm is evaluated every 1000 timesteps over 10 trials. We summarize the empirical results in Figure \ref{fig:sacsmr1m} where SAC-SMR significantly outperforms vanilla SAC on all of the evaluated tasks by a remarkable margin. SAC-SMR can converge faster and learn faster with longer online interactions.

\begin{figure}
    \centering
    \includegraphics[width=0.95\linewidth]{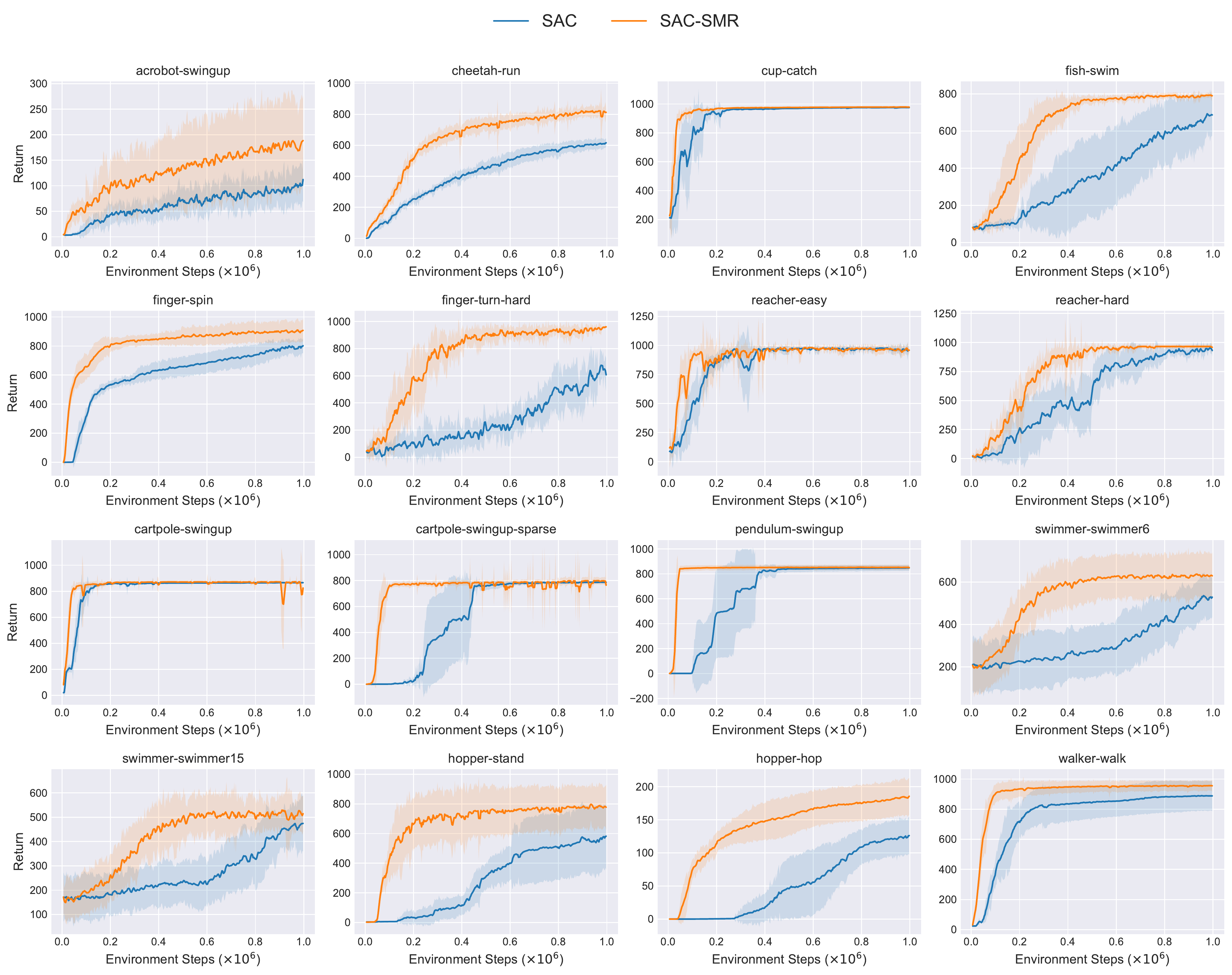}
    \caption{Experimental results of SAC-SMR against SAC on 16 tasks from DMC suite. All methods are run for 1M online interactions, The results are averaged over 6 different random seeds and the shaded region represents the standard deviation.}
    \label{fig:sacsmr1m}
\end{figure}

Furthermore, we run TQC-SMR and TQC for 1M online interactions on 4 OpenAI Gym environments to show the generality of the above conclusion. We summarize the empirical results in Figure \ref{fig:tqcsmr1m}. It can be seen that SMR consistently improves the sample efficiency of TQC with longer interactions, often surpassing the base algorithm by a large margin. We believe the evidence above are enough to show that SMR does aid sample efficiency with longer interactions.

\begin{figure}
    \centering
    \includegraphics[width=0.95\linewidth]{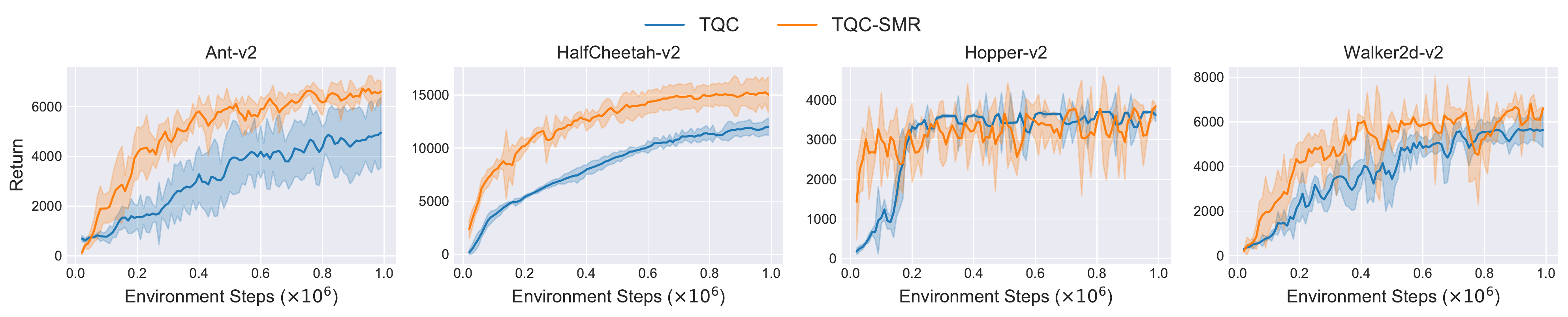}
    \caption{Experimental results of TQC-SMR against TQC on 4 tasks from OpenAI Gym. All methods are run for 1M online interactions, The results are averaged over 6 different random seeds and the shaded region represents the standard deviation.}
    \label{fig:tqcsmr1m}
\end{figure}

The concern on whether SMR aids sample efficiency with longer interactions is strongly correlated with the concern on the asymptotic performance of SMR. One can find in Figure \ref{fig:sacsmr1m} and \ref{fig:tqcsmr1m} that the asymptotic performance of SMR upon different baseline algorithms are actually quite good. For example, on many tasks like \textit{finger-turn-hard}, \textit{reacher-hard}, SAC-SMR converges very fast and achieves the highest possible return on them. Meanwhile, as we emphasize in the main text, we do not mean that the users have to always use a large SMR ratio if one worries about overfitting. SMR can serve as a quite good warm-up strategy, i.e., utilizing SMR (with SMR ratio $M=10$) for some initial interactions (e.g., 300K) and then resume vanilla training process (with SMR ratio $M=1$). In this way, one can enjoy both good sample efficiency from SMR and good asymptotic performance from the vanilla algorithm.

\subsection{Experimental results on Atari}
\label{appsec:atari}

\begin{wrapfigure}{r}{0.4\textwidth}
\centering
\includegraphics[scale=0.6]{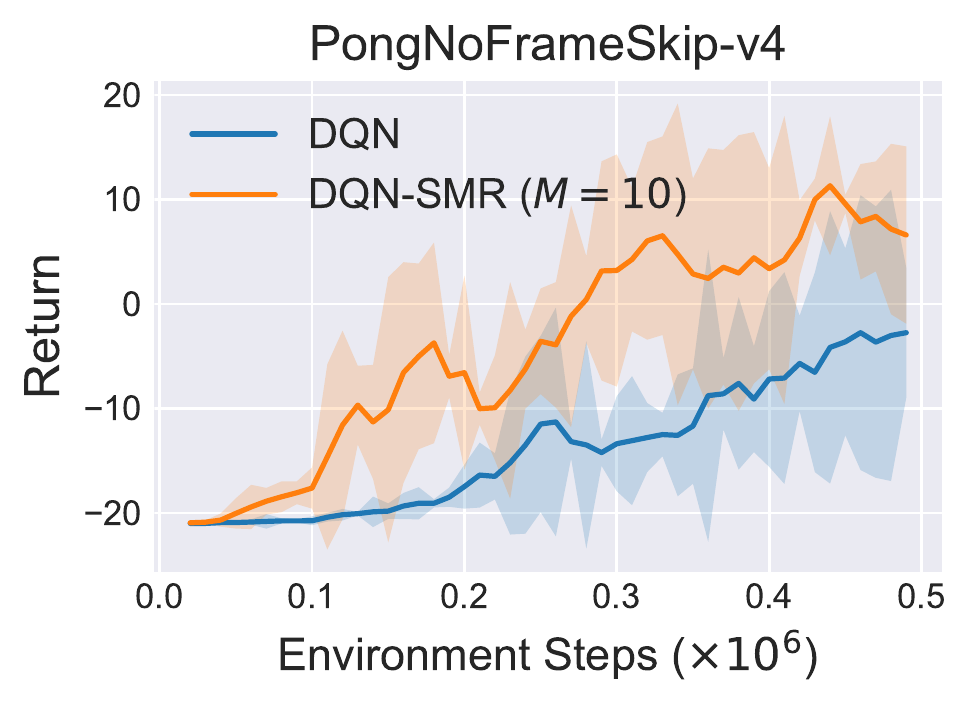}
\caption{Experimental results of DQN-SMR ($M=10$) against DQN on PongNoFrameSkip-v4 task. Each method is run for 500K frames. The results are averaged over 5 random seeds, and the shaded region captures the standard deviation.}
\label{fig:pongexample}
\end{wrapfigure}

We show in the main text that Q-SMR significantly outperforms Q-learning on two discrete control tasks. One may naturally ask: can SMR aid the sample efficiency of DQN on Arcade Learning Environment \cite{Bellemare2012TheAL}? To answer this question, we conduct experiments on one typical environment \texttt{PongNoFrameSkip-v4}. We adopt the original way of processing the Atari image input, i.e., map the image ($3\times 84 \times 84$) into an embedding of size $32\times 7 \times 7$ with convolutional networks. Then, this representation is passed into an MLP with a hidden layer size of 256 to get $Q$-value estimate. We keep the default hyperparameters of DQN unchanged and only incorporate a sample reuse loop in it to yield DQN-SMR. It can be found in Figure \ref{fig:pongexample} that DQN-SMR with SMR ratio $M=10$ remarkably outperforms DQN on Pong. However, It takes about \textbf{84 hours} for DQN-SMR to run 500K frames on Pong, which is 8.7 times slower than DQN (9.6 hours). The computation cost is due to the fact that the size of image input is very large, and it takes time for the network to process it. Updating on the fixed batch (which SMR does) will inevitably worsen the situation and take longer time to train the agent. Considering that there are many advanced methods for discrete control with image input like EfficientZero \cite{Ye2021MasteringAG} (which solves Atari within 2 hours of real-time game play), MuZero \cite{Schrittwieser2019MasteringAG}, Dreamer v2 \cite{hafner2021mastering}, SimPLe \cite{Kaiser2019ModelBasedRL}, it is \textbf{\color{red} STRONGLY NOT RECOMMENDED} to adopt SMR on image-based tasks like Atari. 

However, to show that our method can also work in discrete control settings, we conduct experiments on four environments from Atari. To save training time, we only adopt a small SMR ratio $M=2$. We run \texttt{PongNoFrameSkip-v4} for 1M frames, and other tasks for 4M frames. Each algorithm is evaluated every 5000 timesteps over 10 trials. The results are summarized in Figure \ref{fig:dqn-4}.

It is interesting to see that DQN-SMR outperforms DQN on \texttt{PongNoFrameSkip-v4} and is slightly better than DQN on \texttt{BreakoutNoFrameSkip-v4}. However, DQN-SMR underperforms DQN on \texttt{BeamRiderNoFrameSkip-v4} and only exceeds DQN at the last few online interactions. DQN-SMR learns faster than DQN at first several timesteps on \texttt{SpaceInvadersNoFrameSkip-v4} and underperforms DQN afterwards. Such a phenomenon is due to the fact that the encoder in the DQN network (convolutional layers) is also continually updated during training. At the first several steps, the encoder may output bad representations for the task, indicating that the resulting representations are actually biased and inaccurate. With the sample reuse loop on these bad representations, it will become harder for the network to learn the correct knowledge and policy for this control task. For some of the tasks, the agent may successfully get rid of this dilemma, while on some other tasks, the agent may get stuck and cannot escape from it. Also, DQN is known to incur overestimation bias \cite{Hasselt2015DeepRL, Sabry2019OnTR}, which is similar to DDPG. We observe DDPG-SMR underperforms DDPG on Ant-v2 in Appendix \ref{appsec:td3smrddpgsmr}, and the situation is similar here. Meanwhile, though we adopt a very small SMR ratio $M=2$, it still takes about 2 times longer training time than vanilla DQN, e.g., it takes 18 hours for DQN to run 1M steps on \texttt{PongNoFrameSkip-v4}, while it takes 37 hours of training time for DQN-SMR with $M=2$; it takes 3 days for DQN to run 4M frames on \texttt{BreakoutNoFrameSkip-v4}, while it takes about 6 days of training time for DQN-SMR with SMR ratio $M=2$. We thus do not recommend using SMR loop on image-based tasks. Since we set our focus on the continuous control domain, we do not actively conduct extensive experiments on DQN and its variants (e.g. C51 \cite{Bellemare2017ADP}, Rainbow \cite{Hessel2017RainbowCI}) in discrete control tasks.

\begin{figure}
    \centering
    \includegraphics[width=0.95\linewidth]{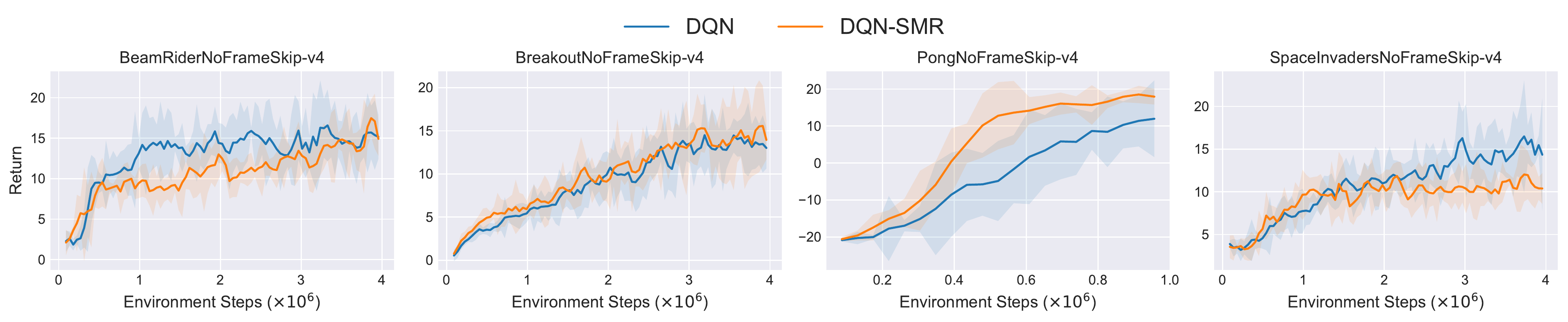}
    \caption{Empirical results of DQN-SMR against DQN on four tasks from Atari. The results are averaged over 5 different random seeds with the shaded region denoting standard deviation.}
    \label{fig:dqn-4}
\end{figure}

\subsection{Can SMR benefit base algorithm with different learning rate?}
In the main text, we combine SMR with the base algorithm without tuning the hyperparameters. Considering the difference between magnifying learning rate and SMR loop, one may wonder whether SMR can boost the sample efficiency of the base algorithm with different initial learning rates. We answer this by comparing SAC-SMR ($M=10$) against SAC and conducting experiments on two typical environments from OpenAI Gym \cite{Brockman2016OpenAIG}, HalfCheetah-v2 and Walker2d-v2, under different initial learning rates for actor and critic networks. We sweep the learning rate over $\{1\times 10^{-3}, 1\times 10^{-4}, 3\times 10^{-4}, 3\times 10^{-5}\}$ (SAC uses a learning rate of $3\times 10^{-4}$ by default, one can check the detailed hyperparameter setup for SAC in Appendix \ref{appsec:continuouscontrolalgo}). We summarize the empirical results in Figure \ref{fig:sac-different-lr}. It is easy to find that SMR notably improves the sample efficiency of SAC upon different initial learning rates, even with a large learning rate $1\times 10^{-3}$. We believe this evidence can alleviate the concern, and validate the effectiveness and generality of SMR.

\begin{figure}[h]
    \centering
    \subfigure[Learning rate 0.001]{
    \label{fig:lr1e-3}
    \includegraphics[width=0.46\linewidth]{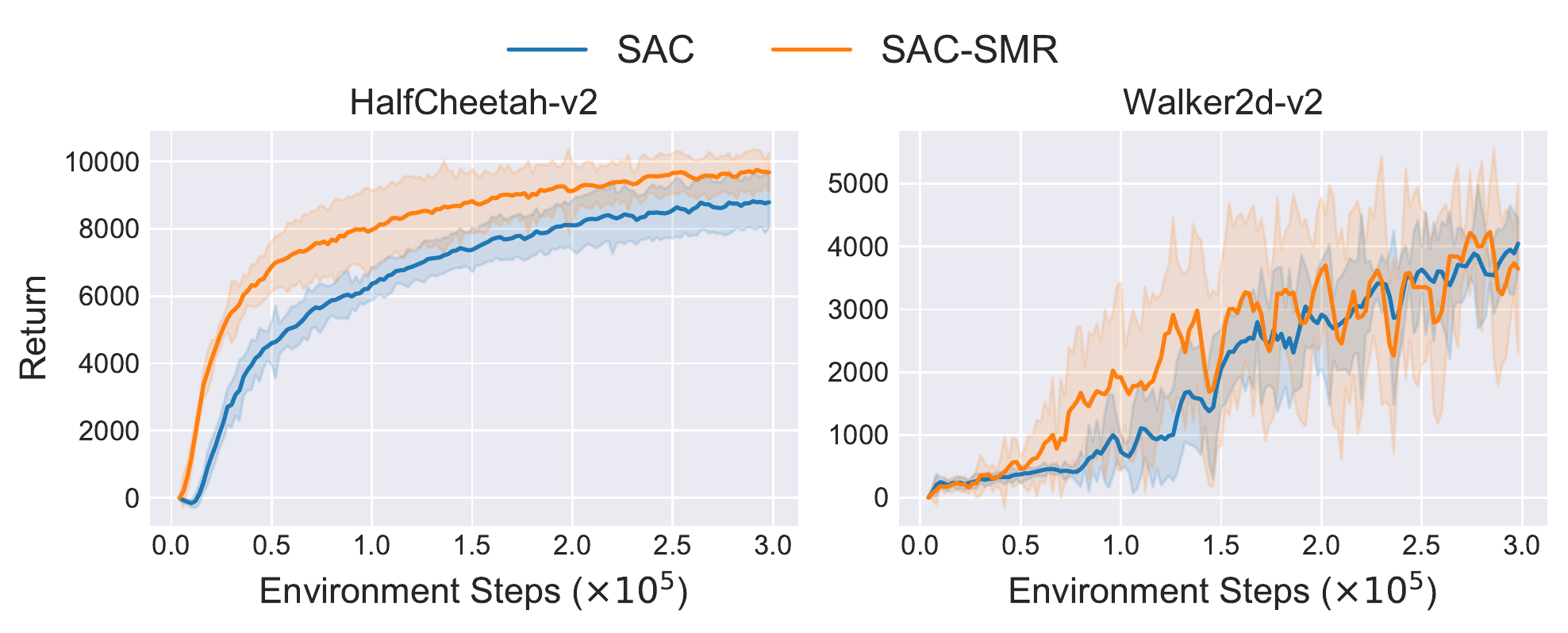}
    }\hspace{-2mm}
    \subfigure[Learning rate 0.0001]{
    \label{fig:lr1e-4}
    \includegraphics[width=0.46\linewidth]{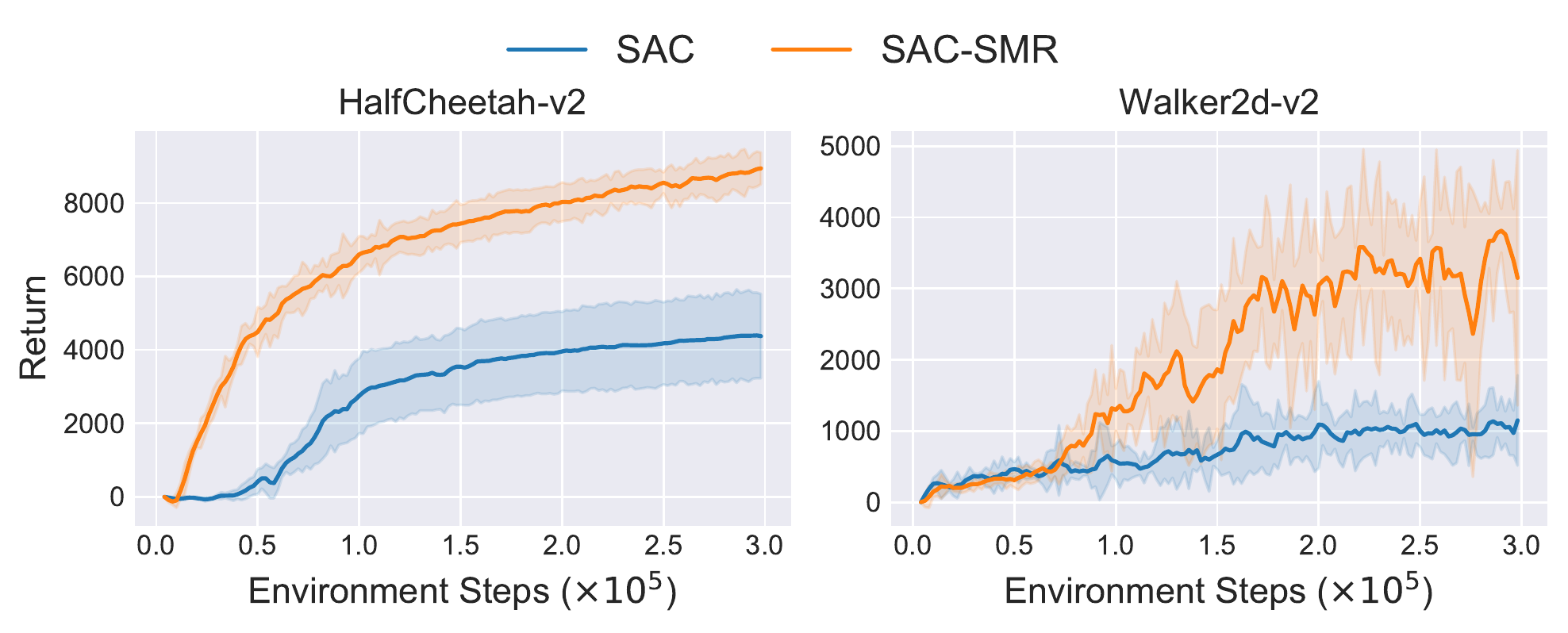}
    }\hspace{-2mm}
    \subfigure[Learning rate 0.0003]{
    \label{fig:lr3e-4}
    \includegraphics[width=0.46\linewidth]{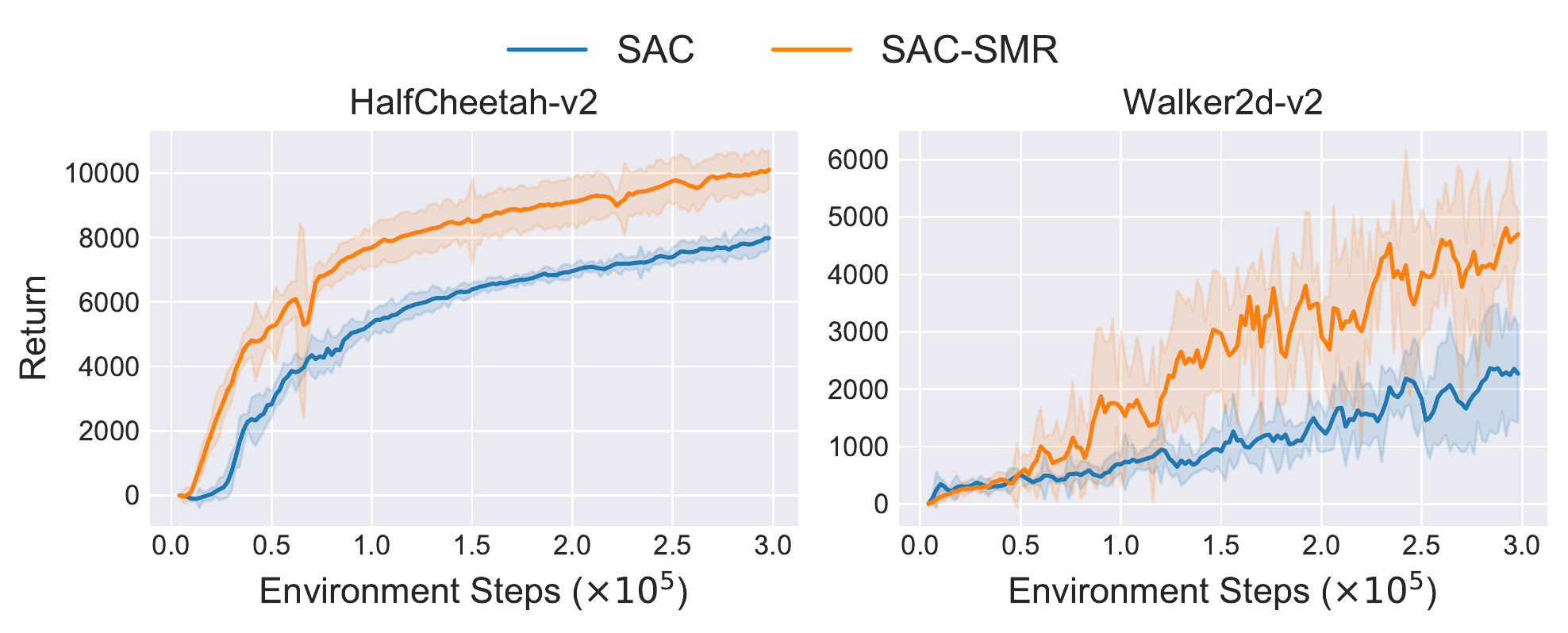}
    }\hspace{-2mm}
    \subfigure[Learning rate 0.00003]{
    \label{fig:lr3e-5}
    \includegraphics[width=0.46\linewidth]{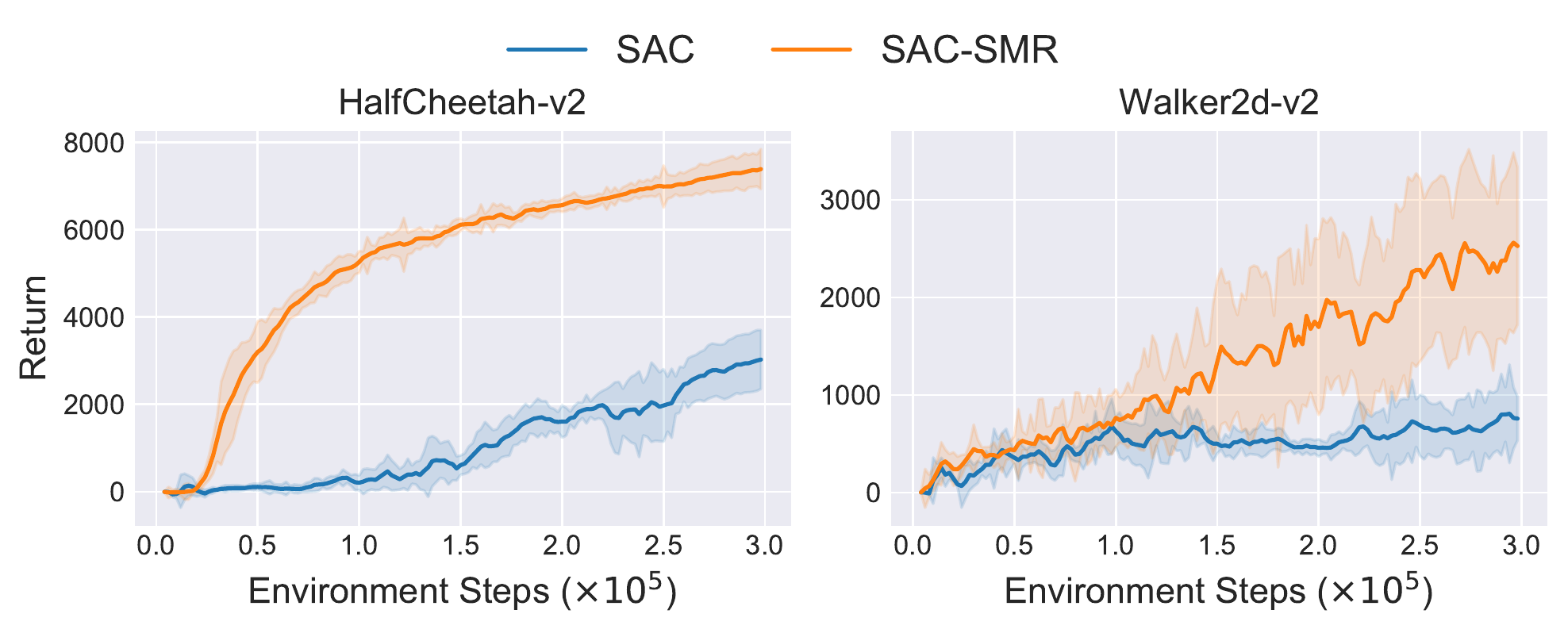}
    }\hspace{-2mm}
    \caption{Performance comparison of SAC-SMR with SMR ratio $M=10$ and SAC on HalfCheetah-v2 and Walker2d-v2 under different initial (fixed) learning rates for actor and critic networks. We sweep the initial learning rate across $\{0.001, 0.0001, 0.0003, 0.00003\}$. The results are averaged over 6 different random seeds and the shaded region denotes the standard deviation.}
    \label{fig:sac-different-lr}
\end{figure}


\section{Pseudo Codes and Hyperparameters of Off-Policy Algorithms with SMR}
\label{sec:pseudocodes}
In this section, we list the missing details on pseudo codes and hyperparameter setup for off-policy algorithms we adopt in this paper. We first introduce the hyperparameters for Q-learning and Q-SMR (the pseudo codes are omitted, please check Algorithm \ref{alg:q-smr}). As we only present the general framework of SMR upon actor-critic architecture in Algorithm \ref{alg:ac-smr}, we further offer the detailed pseudo codes and hyperparameter setup for various continuous control algorithms.
\subsection{Q-learning and Q-SMR}
\label{appsec:qlearningandqsmr}
We conduct experiments using Q-SMR and Q-learning on two discrete control tasks, \texttt{CliffWalking-v0} from OpenAI Gym \cite{Brockman2016OpenAIG} and \texttt{maze-random-20x20-plus-v0} from Gym-Maze (please refer to Gym documentation (\href{https://gymnasium.farama.org/}{https://gymnasium.farama.org/}) and \href{https://github.com/MattChanTK/gym-maze}{https://github.com/MattChanTK/gym-maze}). \texttt{CliffWalking-v0} is a gridworld learning task adapted from Example 6.6 from \cite{Sutton2005ReinforcementLA}. It contains 4$\times$12 grids. The agent starts at the bottom-left and aims at reaching the bottom-right. There exists a cliff in the bottom-center, and the agent will return to the start position if it steps on the cliff. The agent can take four actions (move up, move down, move left, and move right). Each episode of game play contains 100 steps. The episode ends if the agent steps on the cliff. We run Q-SMR and Q-learning for 500 episodes and average their performance over 20 independent runs. 

\texttt{maze-random-20x20-plus-v0} is a 2D maze environment where the agent is targeted to find its way from the top left corner to the goal at the bottom right corner. The objective is to find the shortest path from the start to the goal. The agent can also take four actions (move up, move down, move left, and move right) and the observation space is given by the coordinates of the agent. The agent receives a reward of 1 if it reaches the goal position. For every step in the maze, the agent receives an additional reward of $-\frac{0.1}{\# \rm cells}$, where $\# \rm cells$ denotes the number of cells. For \texttt{maze-random-20x20-plus-v0}, there are $20\times 20$ cells. Specially, the agent can \textit{teleport} from a portal to another portal of the same color. We run Q-SMR and Q-learning for 100 episodes where each episode contains 40000 steps. The maze will be reset if the episode terminates.

For both two environments, we use a learning rate of $\alpha = 0.05$, discount factor $\gamma=0.99$, and exploration rate $\epsilon = 0.1$ ($\epsilon$-greedy) for the training process. Unlike DQN, we use a fixed exploration rate and learning rate instead of decaying them. During the evaluation, we use an exploration rate of $\epsilon=0$. We adopt random seeds of 0-19 for simplicity.

\subsection{Continuous control algorithms}
\label{appsec:continuouscontrolalgo}
In our experiments, we use \texttt{MuJoCo 2.0} with \texttt{Gym} version \texttt{0.18.3} and \texttt{PyTorch} \cite{Paszke2019PyTorchAI} version \texttt{1.8}. We conduct experiments on MuJoCo ``-v2" environments and PyBullet ``-MuJoCoEnv-v0" environments.

We present in Algorithm \ref{alg:algtd3smr} the pseudo code for TD3-SMR. We omit the pseudo code for SAC-SMR since it is much similar to that of TD3-SMR. Compared to the original TD3, TD3-SMR only injects a sample reuse loop (see line 7-15 of Algorithm \ref{alg:algtd3smr}), which is the only modification. We list in Table \ref{tab:parametertd3sac} the hyperparameters for SAC, TD3, and SAC-SMR, TD3-SMR where SAC-SMR and TD3-SMR share identical parameters with their base algorithms. We keep the hyperparameters of all these algorithms fixed on all of the evaluated tasks. Our parameter setup generally resembles \cite{haarnoja2018softactorcritic, Haarnoja2018SoftAO}. It is worth noting that this hyperparameter setup is slightly different from the original TD3, where network size $(400,300)$, learning rate $1\times 10^{-3}$ and batch size $100$ are adopted (see \cite{Fujimoto2018AddressingFA}). As the authors mentioned (please see \href{https://github.com/sfujim/TD3}{https://github.com/sfujim/TD3}), the parameter setup for TD3 is now different from the original paper. We therefore choose to follow the current hyperparameter setup in the authors' official implementation.

\begin{algorithm}[tb]
\caption{TD3-SMR}
\label{alg:algtd3smr}
\begin{algorithmic}[1] 
\STATE Initialize critic networks $Q_{\theta_1}, Q_{\theta_2}$ and actor network $\pi_{\phi}$ with random parameters
\STATE Initialize target networks $\theta_1^\prime \leftarrow \theta_1, \theta_2^\prime \leftarrow \theta_2, \phi^\prime \leftarrow \phi$ and replay buffer $\mathcal{B} = \{\}$
\FOR{$t$ = 1 to $T$}
\STATE Select action $a$ with exploration noise $a\sim \pi_\phi(s) + \epsilon$, where $\epsilon\sim \mathcal{N}(0,\sigma)$ and observe reward $r$, new state $s^\prime$
\STATE Store transitions in the replay buffer, i.e., $\mathcal{B}\leftarrow\mathcal{B}\bigcup \{(s,a,r,s^\prime)\}$
\STATE Sample $N$ transitions $\{(s_j,a_j,r_j,s_j^\prime)\}_{j=1}^N\sim\mathcal{B}$
\color{red} \FOR{$m$ = 1 to $M$}
\STATE $\tilde{a}\sim\pi_{\phi^\prime}(s^\prime) + \epsilon$, $\epsilon\sim$ clip($\mathcal{N}(0,\bar{\sigma}),-c,c$)
\STATE $y \leftarrow r + \gamma\min_{i=1,2}Q_{\theta^\prime_i}(s^\prime,\tilde{a})$
\STATE Update critic $\theta_i$ by minimizing $\frac{1}{N}\sum_s (Q_{\theta_i}(s,a)-y)^2$
\IF{$t\mod d$}
\STATE Update actor $\phi$ by deterministic policy gradient $\nabla_\phi J(\phi) = \frac{1}{N}\sum_s \nabla_a Q_{\theta_1}(s,a)|_{a=\pi_{\phi}(s)}\nabla_{\phi}\pi_{\phi}(s)$
\STATE Update target networks: $\theta_i^\prime \leftarrow \tau\theta_i + (1-\tau)\theta_i^\prime, \phi^\prime\leftarrow\tau\phi+(1-\tau)\phi^\prime$
\ENDIF
\ENDFOR
\ENDFOR
\end{algorithmic}
\end{algorithm}

We list the pseudo code for DARC-SMR \cite{Efficient2022Lyu} in Algorithm \ref{alg:algdarc} and its hyperparameter setup in Table \ref{tab:parametertd3sac}. We follow the original hyperparameter setup of the DARC paper and adopt the network size $(400,300)$ for both the actor network and critic network. For the weighting coefficient $\nu$ in DARC (for balancing the underestimation bias and overestimation bias), we also follow the best recommended hyperparameter by the authors, where we adopt $\nu=0.15$ for Hopper-v2, $\nu=0.25$ for Ant-v2, and $\nu=0.1$ for HalfCheetah-v2 and Walker2d-v2. For the regularization parameter $\lambda$ in DARC, we use $\lambda=0.005$ by default. Other parameters are identical to the original paper and we keep them unchanged throughout our experiments. We use the official implementation of DARC (\href{https://github.com/dmksjfl/DARC}{https://github.com/dmksjfl/DARC}) when conducting experiments.

\begin{table*}
\centering
\caption{Hyperparameters setup for TD3 \cite{Fujimoto2018AddressingFA}, SAC \cite{haarnoja2018softactorcritic}, DARC \cite{Efficient2022Lyu}, TQC \cite{Kuznetsov2020ControllingOB}, and REDQ \cite{Chen2021RandomizedED} on continuous control benchmarks.}
\label{tab:parametertd3sac}
\begin{tabular}{lrr}
\toprule
\textbf{Hyperparameter}  & \textbf{Value} \quad \\
\midrule
Shared & \\
\qquad Actor network  & \qquad  $(400,300)$ for DARC and $(256,256)$ for others \\
\qquad Batch size     &\qquad   $256$ \\
\qquad Learning rate  & \qquad $1\times 10^{-3}$ for DARC and $3\times 10^{-4}$ for others \\
\qquad Optimizer & \qquad Adam \cite{KingmaB14adam} \\
\qquad Discount factor & \qquad $0.99$ \\
\qquad Replay buffer size & \qquad $10^6$  \\
\qquad Warmup steps & \qquad $256$ for TQC and $5\times 10^3$ for others \\
\qquad Nonlinearity & \qquad ReLU \\
\qquad Target update rate & \qquad $5\times 10^{-3}$ \\
\midrule
TD3 & \\
\qquad Target update interval & \qquad $2$ \\
\qquad Critic network & \qquad $(256,256)$ \\
\qquad Exploration noise &\qquad  $\mathcal{N}(0,0.1)$ \\
\qquad Target noise & \qquad $0.2$ \\
\qquad Noise clip & \qquad $0.5$ \\
\midrule
DARC & \\
\qquad Regularization parameter $\lambda$ & \qquad $0.005$ \\
\qquad Critic network & \qquad $(400,300)$ \\
\midrule
SAC & \\
\qquad Critic network & \qquad $(256,256)$ \\
\qquad Target update interval & $1$ \\
\qquad Reward scale & \qquad $1$ \\
\qquad Entropy target & \qquad $-{\rm dim}(\mathcal{A})$  \\
\qquad Entropy auto-tuning & \qquad True \\
\qquad Maximum log std & \qquad $2$ \\
\qquad Minimum log std & \qquad $-20$ \\
\midrule
TQC & \\
\qquad Critic network & \qquad $(512,512,512)$ \\
\qquad Number of critic networks & \qquad $5$ \\
\qquad Number of atoms & \qquad $25$ \\
\qquad Huber loss parameter & \qquad $1$ \\
\midrule
REDQ & \\
\qquad Critic network & \qquad $(256,256)$ \\
\qquad Update-to-data (UTD) ratio & \qquad $20$ \\
\qquad Number of critic networks & \qquad $10$ \\
\qquad In-target minimization parameter & \qquad $2$ \\
\bottomrule
\end{tabular}
\end{table*}

\begin{algorithm}[tb]
\caption{DARC-SMR}
\label{alg:algdarc}
\begin{algorithmic}[1] 
\STATE Initialize critic networks $Q_{\theta_1}, Q_{\theta_2}$ and actor networks $\pi_{\phi_1}, \pi_{\phi_2}$ with random parameters
\STATE Initialize target networks $\theta_1^\prime \leftarrow \theta_1, \theta_2^\prime \leftarrow \theta_2, \phi_1^\prime \leftarrow \phi_1, \phi_2^\prime \leftarrow \phi_2$ and replay buffer $\mathcal{B} = \{\}$
\FOR{$t$ = 1 to $T$}
\STATE Select action $a$ with $\max_i \max_j Q_{\theta_i}(s, \pi_{\phi_j}(s))$ added $\epsilon\sim \mathcal{N}(0,\sigma)$
\STATE Execute action $a$ and observe reward $r$, new state $s^\prime$ and done flag $d$
\STATE Store transitions in the replay buffer, i.e., $\mathcal{B}\leftarrow\mathcal{B}\bigcup \{(s,a,r,s^\prime,d)\}$
\FOR{$i = 1,2$}
\STATE Sample $N$ transitions $\{(s_j,a_j,r_j,s_j^\prime,d_j)\}_{j=1}^N\sim\mathcal{B}$
\color{red} \FOR{$m$ = 1 to $M$}
\STATE Get actions: $a_1\leftarrow \pi_{\phi_1^\prime}(s^\prime) + \epsilon$, $a_2 \leftarrow \pi_{\phi_2^\prime}(s^\prime) + \epsilon$, $\epsilon\sim$ clip($\mathcal{N}(0,\bar{\sigma}),-c,c$)
\STATE Calculate $\hat{V}(s^\prime) = (1 - \nu) \max_{k=1,2} \min_{j=1,2} Q_{\theta_j^\prime}(s^\prime,a_k) + \nu \min_{k=1,2} \min_{j=1,2} Q_{\theta_j^\prime}(s^\prime,a_k),$
\STATE $y \leftarrow r + \gamma(1-d) \hat{V}(s^\prime)$
\STATE Update critic $\theta_i$ by minimizing $\frac{1}{N}\sum_s \left \{ (Q_{\theta_i}(s,a)-y)^2 + \lambda [Q_{\theta_1}(s,a) - Q_{\theta_2}(s,a)]^2 \right \}$
\STATE Update actor $\phi_i$ by maximizing $\frac{1}{N}\sum_s \nabla_a Q_{\theta_i}(s,a)|_{a=\pi_{\phi_i}(s)}\nabla_{\phi_i}\pi_{\phi_i}(s)$
\STATE Update target networks: $\theta_i^\prime \leftarrow \tau\theta_i + (1-\tau)\theta_i^\prime, \phi_i^\prime\leftarrow\tau\phi_i+(1-\tau)\phi_i^\prime$
\ENDFOR
\ENDFOR
\ENDFOR
\end{algorithmic}
\end{algorithm}

We further combine SMR with TQC \cite{Kuznetsov2020ControllingOB} and list the pseudo code for TQC-SMR in Algorithm \ref{alg:algtqc}, with its hyperparameter setup listed in Table \ref{tab:parametertd3sac}. Similarly, we follow the default hyperparameter recommended by the authors, e.g., the actor network has a network size of $(256,256)$, while the critic network has a size of $(512,512,512)$. The agent starts training when $256$ samples are collected. For the most critical hyperparameter, the number of dropped atoms $d$, we follow the original paper and adopt $d=5$ for Hopper-v2, $d=0$ for HalfCheetah-v2, and $d=2$ for Ant-v2 and Walker2d-v2. For TD3-SMR, SAC-SMR, and TQC-SMR, we adopt an SMR ratio $M=10$ by default for all of the evaluated state-based tasks. We use the official implementation of TQC (\href{https://github.com/SamsungLabs/tqc_pytorch}{https://github.com/SamsungLabs/tqc\_pytorch}) for all of the experimental evaluation.

\begin{algorithm}[tb]
\caption{TQC-SMR}
\label{alg:algtqc}
\begin{algorithmic}[1] 
\STATE Initialize critic networks $Z_{\theta_n}, n\in\{1,2,\ldots,N\}$ and actor network $\pi_{\phi}$ with random parameters
\STATE Initialize target networks $\theta_n^\prime \leftarrow \theta_n, n\in\{1,2,\ldots,N\}$ and replay buffer $\mathcal{D} = \{\}$
\STATE Set target entropy $\mathcal{H}_T = -{\rm dim}(\mathcal{A})$, $\alpha=1$, number of quantiles $L$, left atom proportion $k$
\FOR{$t$ = 1 to $T$}
\STATE Execute action $a\sim\pi_\phi$ and observe reward $r$, new state $s^\prime$
\STATE Store transitions in the replay buffer, i.e., $\mathcal{D}\leftarrow\mathcal{D}\bigcup \{(s,a,r,s^\prime)\}$
\STATE Sample a mini-batch transitions $B = \{(s,a,r,s^\prime)\}\sim\mathcal{D}$
\color{red} \FOR{$m$ = 1 to $M$}
\STATE Update temperature parameter via $\nabla_\alpha J(\alpha) = \nabla_\alpha \mathbb{E}_{B,\pi_\phi}[\log\alpha \cdot (-\log\pi_\phi(a|s) - \mathcal{H}_T)]$
\STATE Update actor parameter $\phi$ via $\nabla_\phi \mathbb{E}_{B,\pi_\phi} \left[ \alpha\log\pi_{\phi}(a|s) - \dfrac{1}{NL}\displaystyle\sum_{l,n=1}^{N,L} \psi_{\theta_n}^l(s,a) \right]$ \\ // $\psi_{\theta_n}^l,l\in[1,L]$ is the atom at location $l$
\STATE $y_i = r + \gamma[z_{(i)}(s^\prime,a^\prime) - \alpha\log\pi_\phi(a^\prime|s^\prime)]$ \quad // $z_{(i)}$ is the sorted atoms in ascending order, $i\in[NL]$
\STATE Update critic parameter $\theta_n$ by $\nabla_{\theta_n} \mathbb{E}_{B,\pi_\phi}\left[ \dfrac{1}{kNL}\displaystyle \sum_{l=1}^L \sum_{i=1}^{kN} \rho_{\tau_l}(y_i - \psi_{\theta_n}^l) \right]$ \\ 
// $\rho_{\tau_l}$ is the Huber quantile loss with parameter 1
\STATE Update target networks: $\theta_n^\prime \leftarrow \beta\theta_n + (1-\beta)\theta_n^\prime, n\in\{1,2,\ldots,N\}$
\ENDFOR
\ENDFOR
\end{algorithmic}
\end{algorithm}

For REDQ \cite{Chen2021RandomizedED}, we also keep the original hyperparameters unchanged when combining it with SMR, i.e., it uses a learning rate of $3\times 10^{-4}$ and a network size of $(256,256)$ for both the actor network and critic network, an ensemble size of $10$ for critics. REDQ also adopts a high update-to-data (UTD) ratio of $G=20$ and samples $2$ different indices from $10$ critics when calculating the target $Q$ value. We summarize the pseudo code for REDQ-SMR in Algorithm \ref{alg:algredq} and the hyperparameter setup in Table \ref{tab:parametertd3sac}. Inspired by the fact that model-based methods often attain higher sample efficiency by using a high UTD ratio (the number of updates taken by the agent compared to the number of actual interactions with the environment), REDQ explores the feasibility of high UTD ratio without a dynamics model on continuous control tasks. As discussed in the main text (Section \ref{sec:relatedwork}), SMR is different from adopting a high UTD ratio. REDQ and model-based methods update the agent multiple times with bootstrapping, i.e., each time the agent sees different samples and updates on these different data multiple times. SMR, however, updates multiple times on the \textit{fixed} batch data for multiple times. Since REDQ already leverages a high UTD ratio, we use an SMR ratio $M=5$ for REDQ-SMR. It is worth noting that our reported performance of REDQ is slightly different from the original paper. We have tried our best to reproduce the performance of REDQ on MuJoCo tasks. However, as the authors commented in \href{https://github.com/watchernyu/REDQ}{https://github.com/watchernyu/REDQ}, the performance of REDQ seems to be quite different with different PyTorch \cite{Paszke2019PyTorchAI} version and the reasons are not entirely clear. We thus choose to run REDQ with its official implementation (\href{https://github.com/watchernyu/REDQ}{https://github.com/watchernyu/REDQ}) and PyTorch 1.8 and report the resulting learning curves.

\begin{algorithm}[tb]
\caption{REDQ-SMR}
\label{alg:algredq}
\begin{algorithmic}[1] 
\STATE Initialize critic networks $Q_{\theta_i}, i=1,2,\ldots,N$ and actor network $\pi_{\phi}$ with random parameters
\STATE Initialize target networks $\theta_i^\prime \leftarrow \theta_i, i=1,2,\ldots,N$ and replay buffer $\mathcal{D} = \{\}$
\FOR{$t$ = 1 to $T$}
\STATE Take one action $a_t\sim\pi_{\phi}(\cdot|s_t)$ and observe reward $r_t$, new state $s^\prime_{t+1}$
\STATE Store transitions in the replay buffer, i.e., $\mathcal{D}\leftarrow\mathcal{D}\bigcup \{(s_t,a_t,r_t,s^\prime_{t+1})\}$
\FOR{$g$ = 1 to $G$}
\STATE Sample a mini-batch $B = \{(s,a,r,s^\prime)\}\sim\mathcal{D}$
\color{red} \FOR{$m$ = 1 to $M$}
\STATE Sample a set $\mathcal{K}$ of $K$ indices from $\{1,2,\ldots,N\}$
\STATE Compute the $Q$ target $y = r + \gamma \left( \min_{i\in\mathcal{K}}Q_{\theta_i^\prime}(s^\prime,\tilde{a}^\prime) - \alpha\log\pi_\phi(\tilde{a}^\prime|s^\prime) \right), \tilde{a}^\prime\sim\pi_\phi(\cdot|s^\prime)$
\FOR{$i=1,2,\ldots,N$}
\STATE Update $\theta_i$ with gradient descent using $\nabla_{\theta_i}\frac{1}{|B|}\sum_{(s,a,r,s^\prime)\sim B} (Q_{\theta_i}(s,a)-y)^2$
\STATE Update target networks: $\theta_i^\prime \leftarrow \tau\theta_i + (1-\tau)\theta_i^\prime$
\ENDFOR
\IF{$g=G$}
\STATE Update actor $\phi$ with gradient ascent using $\nabla_\phi \frac{1}{|B|} \sum_{s\in B} \left( \frac{1}{N}\sum_{j=1}^N Q_{\theta_j}(s,\tilde{a}) - \alpha\log\pi_\phi(\tilde{a}|s) \right),\tilde{a}\sim\pi_\phi(\cdot|s)$
\ENDIF
\ENDFOR
\ENDFOR
\ENDFOR
\end{algorithmic}
\end{algorithm}

For image-based tasks, we adopt the environment wrapper from TD-MPC \cite{Hansen2022TemporalDL} for SAC and SAC-SMR. The image is processed with a 4-layer CNN with kernel size $(7,5,3,3)$, stride $(2,2,2,2)$ and $32$ filters per layer. Then the representation is input into a 2-layer MLP with $512$ hidden units. We map the raw image input into an embedding of size $200$ and repeat the actions every two frames for six evaluated tasks. For DrQ-v2, we use its official implementation (\href{https://github.com/facebookresearch/drqv2}{https://github.com/facebookresearch/drqv2}) and keep its default hyperparameters setup fixed. For DQN experiments on Atari, we adopt the widely used RL playground implementation for DQN (\href{https://github.com/TianhongDai/reinforcement-learning-algorithms}{https://github.com/TianhongDai/reinforcement-learning-algorithms}).

\section{Compute Infrastructure}

In Table \ref{tab:computing}, we list the compute infrastructure that we use to run all of the algorithms.

\begin{table}[htb]
\caption{Compute infrastructure.}
\label{tab:computing}
\centering
\begin{tabular}{c|c|c}
\toprule
\textbf{CPU}  & \textbf{GPU} & \textbf{Memory} \\
\midrule
AMD EPYC 7452  & RTX3090$\times$8 & 288GB \\
\bottomrule
\end{tabular}
\end{table}

\section{Licences}

We implement SAC on our own. Other codes are built upon source DDPG and TD3 codebases under MIT licence (\href{https://github.com/sfujim/TD3}{https://github.com/sfujim/TD3}), DARC codebase under MIT licence (\href{https://github.com/dmksjfl/DARC}{https://github.com/dmksjfl/DARC}), TQC codebase under MIT licence (\href{https://github.com/SamsungLabs/tqc\_pytorch}{https://github.com/SamsungLabs/tqc\_pytorch}), REDQ codebase under MIT licence (\href{https://github.com/watchernyu/REDQ}{https://github.com/watchernyu/REDQ}), DrQ-v2 codebase under MIT licence (\href{https://github.com/facebookresearch/drqv2}{https://github.com/facebookresearch/drqv2}).

\section{Broader Impacts}
This work mainly focuses on a simple and novel way of improving sample efficiency of the off-the-shelf off-policy RL algorithms. We do not foreseen any potential negative social impact of this work.

\end{document}